\documentclass[twoside,11pt]{article}

\usepackage{blindtext}

\usepackage[preprint]{jmlr2e}

\makeatletter
\renewcommand{\jmlrheading}[7]{%
  \thispagestyle{plain}%
}
\makeatother

\usepackage[utf8]{inputenc} 
\usepackage[T1]{fontenc}    
\usepackage{hyperref}       
\usepackage{url}            
\usepackage{booktabs}       
\usepackage{amsfonts}       
\usepackage{nicefrac}       
\usepackage{microtype}      
\usepackage{xcolor}         
\usepackage{booktabs}
\usepackage{mathtools}
\usepackage{mathrsfs}
\usepackage{amsmath}
\usepackage{amssymb}
\usepackage{multirow}
\usepackage{adjustbox}
\usepackage{bm}
\usepackage{thmtools}
\usepackage{thm-restate}
\usepackage{subcaption} 
\usepackage{hyperref} 

\DeclareMathOperator*{\argmin}{arg\,min}

\usepackage[capitalize,noabbrev]{cleveref}
\crefname{assumption}{Assumption}{Assumptions}

\newtheorem{example}{Example} 
\newtheorem{theorem}{Theorem}
\newtheorem{lemma}[theorem]{Lemma} 
\newtheorem{proposition}[theorem]{Proposition} 
\newtheorem{remark}[theorem]{Remark}

\newtheorem{definition}[theorem]{Definition}

\newtheorem{assumption}[theorem]{Assumption}

\usepackage{lastpage}
\jmlrheading{23}{2022}{1-\pageref{LastPage}}{1/21; Revised 5/22}{9/22}{21-0000}{Author One and Author Two}

\ShortHeadings{Score Matching for Estimating Finite Point Processes}{Score Matching for Estimating Finite Point Processes}
\firstpageno{1}

\newif\ifarxiv      

\arxivtrue          

\begin{document}

\title{Score Matching for Estimating Finite Point Processes}

\author{\name Haoqun Cao \email hcao65@wisc.edu \\
       \addr Department of Statistics\\
       University of Wisconsin-Madison, USA
       \AND
       \name Yixuan Zhang \email zh1xuan@hotmail.com \\
       \addr School of Statistics and Data Science\\
       Southeast University, China
       \AND
       \name Feng Zhou\thanks{Corresponding author.} \email feng.zhou@ruc.edu.cn\\
       \addr Center for Applied Statistics and School of Statistics\\
       Renmin University of China, 100872, China
       }

\editor{My editor}

\maketitle

\begin{abstract}

Score matching estimators have garnered significant attention in recent years because they eliminate the need to compute normalizing constants, thereby mitigating the computational challenges associated with maximum likelihood estimation (MLE).While several studies have proposed score matching estimators for point processes, this work highlights the limitations of these existing methods, which stem primarily from the lack of a mathematically rigorous analysis of how score matching behaves on finite point processes---special random configurations on bounded spaces where many of the usual assumptions and properties of score matching no longer hold.

To this end, we develop a formal framework for score matching on finite point processes via Janossy measures and, within this framework, introduce an (autoregressive) weighted score-matching estimator, whose statistical properties we analyze in classical parametric settings. 
For general nonparametric (e.g., deep) point process models, we show that score matching alone does not uniquely identify the ground-truth distribution due to subtle normalization issues, and we propose a simple survival-classification augmentation that yields a complete, integration-free training objective for any intensity-based point process model for spatio-temporal case. Experiments on synthetic and real-world temporal and spatio-temporal datasets,  demonstrate that our method accurately recovers intensities and achieves performance comparable to MLE with better efficiency.

\end{abstract}


\section{Introduction}
\label{intro}

Point processes are a class of statistical models used to characterize event occurrences. Typical models include Poisson processes~\citep{kingman1992poisson} and Hawkes processes~\citep{hawkes1971spectra}. Their applications span various fields such as seismology~\citep{ogata1998space,ogata1999seismicity}, finance~\citep{bacry2014hawkes,hawkes2018hawkes}, criminology~\citep{mohler2013modeling}, and neuroscience~\citep{linderman2016bayesian,zhou2022efficient}. 
In the field of point processes, maximum likelihood estimation (MLE) has been the conventional estimator. However, MLE has an inherent limitation: it requires the computation of the normalizing constant, which corresponds to the intensity integral term in the likelihood. Except for simple cases, calculating this intensity integral analytically is generally infeasible. This necessitates the use of numerical integration methods, such as Monte Carlo or quadrature to approximate the computation. This introduces approximation errors, and more importantly, for high-dimensional problems, numerical integration encounters the curse of dimensionality, rendering training infeasible. 

To address this issue, prior research has introduced the concept of score matching (SM)~\citep{Hyvarinen05} to the field of point processes. For instance, \citet{SahaniBM16} derived the application of SM to the estimation of traditional statistical Poisson processes. Furthermore, \citet{zhang2023integration} extended the use of SM to the estimation of deep covariate spatio-temporal point processes. \citet{li2023smurf} further generalized the application of SM to Hawkes processes. These works have greatly advanced the utilization of SM for point processes. 
However, this line of work suffers from several fundamental limitations. Theoretically, existing methods generally lack provable guarantees that their training objectives are correct or statistically well-posed. Empirically, we find that their methods only work for specific classes of point processes and can even fail to recover the intensity of simple parametric models. In this work, we formally demonstrate the incompleteness of the estimators proposed in the aforementioned studies.

On the theoretical side, this incompleteness primarily arises because \textbf{a finite point process is a random configuration on a bounded space, so several properties and assumptions that underlie score matching for general multivariate random variables no longer hold}. Two issues are particularly important. First, the original implicit score-matching objective is not, in general, a statistically valid training objective for point processes. Second, even explicitly matching the score still does not provably recover the distribution, because the normalization condition for finite point processes is substantially more delicate. These observations call for a mathematically rigorous framework for score matching on finite point processes.

On the empirical side, although score- or diffusion-based deep point process models \citep{yuan2023spatio} have achieved strong performance on many spatio-temporal benchmarks, they are not intensity-based models. As a consequence, they can only sample from the process but cannot recover the ground-truth intensity or compute the exact log-likelihood at evaluation time. In this work, we propose a \textbf{completely integration-free training objective that applies to any intensity-based nonparametric point process model}, and we show---both theoretically and through extensive experiments---that it can provably recover the ground-truth intensity.

Specifically, we make the following contributions:\textbf{(1)} We establish a formal framework for score matching on finite point processes via Janossy measures. Within this framework, we theoretically demonstrate that the implicit, autoregressive score-matching estimators in \citet{SahaniBM16, zhang2023integration, li2023smurf} are \emph{incomplete}, in the sense that the regularity conditions required for their correctness cannot be satisfied for general point processes. \textbf{(2)} We propose an (autoregressive) weighted score-matching (WSM or AWSM) estimator that is applicable to general point processes and, under the classical parametric statistics framework, establish its consistency and convergence rate. \textbf{(3)} For general nonparametric point process models, we show that even with weighting, the (A)WSM objective does not, in general, uniquely identify the ground-truth intensity or the underlying distribution, due to subtle normalization issues. We then propose a simple remedy via survival classification that yields a complete, integration-free training objective that applies to any nonparametric intensity model for spatio-temporal cases. \textbf{(4)} We conduct extensive experiments on both synthetic and real-world data, using both classical statistical and deep point process models, on both temporal and spatio-temporal datasets, and empirically validate all our theoretical claims. Our proposed training objective achieves accuracy comparable to MLE while remaining efficient and scalable.



\section{Preliminaries}
\label{preliminary}

In this section, we provide background knowledge on finite point processes, score matching, and autoregressive score matching. 

\subsection{Finite Point Processes}
This paper focuses on point processes that contain a finite number of points almost surely. We begin with two definitions that formalize the probability space considered throughout the paper.

\begin{definition}
Let the points lie in a bounded subset $V \subset \mathbb{R}^d$. A distribution $\{p_N\}_{N \geq 0}$ is given, specifying the probability of having exactly $N$ points in $V$. 
Let $V^{(0)}=\emptyset$, for each integer $N \geq 1$, a probability distribution $\Pi_N(\cdot)$ is defined on the Borel sets of $V^{(N)} := V \times \cdots \times V$ (with $N$ copies of $V$), which determines the joint distribution of the positions of the points in the process, given that the total number of points is $N$.    
The \textbf{Janossy measure} is defined as: 
\begin{equation*}
   J_0(\emptyset) = p_0, J_N(A^{(N)}) = p_N\sum_{\pi}\Pi_N(A^{(N)}_{\pi}), N\geq 1 
\end{equation*}
where $A^{(N)} \in \mathcal{B}_{V^{(N)}}$ is a Borel set in Borel sigma algebra on $V^{(N)}$, $\pi$ is a permutation of $\{1, \ldots, N\}$, and $A^{(N)}_{\pi}$ is the permutation of $A^{(N)}$ by $\pi$. When $J_N(\cdot)$ has a density with respect to the Lebesgue measure, the density is called the Janossy density, denoted by $j_N(\bm x_1, \ldots, \bm x_N)$ with $(\bm x_1, \ldots, \bm x_N)\in V^{(N)}$. A finite point process that has a Janossy density is called \textbf{regular}. 
\end{definition}

\begin{definition}\label{def:data_generating_measure}
    Let $V^{\cup} = \bigcup_{N\geq 0}V^{(N)}$. Let $\mathscr{F} = \sigma\Big(\cup_{N\geq0}\mathcal B_{V^{(N)}}\Big)$, where $\sigma(\cdot)$ denotes the generated sigma algebra, then define $\mathbb  P$ to be: 
    \begin{equation*}
        \mathbb P(B) = \sum_{N}\frac{1}{N!}J_N(B\cap V^{(N)}), B\in \mathscr F.
    \end{equation*}
    Then $(V^{\cup}, \mathscr F, \mathbb P)$ defines a finite point process on $V$. We denote the process with $\mathcal X \in V^{\cup}$, i.e., $\mathcal X(\omega)=\omega$ for $\omega \in V^{\cup}$. 
\end{definition}

Let $\bm{X}$ represents an element in $V^{\cup}$, and $\bm{X}_N$ represents an element in $V^{(N)}$. The $n$-th point in $\bm{X}$ is denoted by $\bm{x}_n$. Throughout this paper, $N$ implies the total number of points is $N$ and $n$ is used for the $n$-th point among all realized points. We denote $\bm{X}_{-n}$ by the sequence $\bm{X}$ excluding $\bm{x}_n$. 
We define a measurable function $j: V^{\cup} \to \mathbb{R}$ as: 
\begin{equation*}
    j(\bm{X}) := j_N(\bm{X}), \quad \text{when } \bm{X} \in V^{(N)},
\end{equation*}
where $j_N$ is the Janossy density. We denote the number of points in $\bm{X}$ by $N(\bm{X})$. Notice that $N(\mathcal{X})$ is a random variable. 

From here onward, we use $\mathbb P$ for the ground-truth data generating probability measure. It is worth noting that we have not explicitly verified that the Janossy measure and $\mathbb{P}$ are indeed measures. One can refer to Lemma 2.2 of \citet{moyal1962general} for proof. Furthermore, since $\mathbb{P}$ is a probability measure and is induced by $J_n$, we know that the normalizing condition for $J_n$ is given by: 
\begin{equation}\label{eq:J_normalising}
    \sum_{N \geq 0} \frac{1}{N!} J_N(V^{(N)}) = \sum_{N \geq 0} p_N = 1.
\end{equation}

There is also a more prevalent definition of point processes through random counting measures, which is equivalent to our definition, as discussed in \citet{daley2003introduction}. However, the definition via the Janossy measure suffices our discussion of finite point processes on bounded sets, and it is convenient to use since it is related to the likelihood function of a finite point process. 
\begin{definition}
The likelihood of a realization $\{\bm{x}_1, \ldots, \bm{x}_N\}$ of a regular point process $\mathcal{X}$, where $N = N(\mathcal{X})$, is the Janossy density: 
\begin{equation*}
    L(\bm{x}_1, \ldots, \bm{x}_N) = j_N(\bm{x}_1, \ldots, \bm{x}_N).
\end{equation*}
\end{definition}


Although most point processes are not directly defined in terms of the distribution $p_N$ and $\Pi_N$, there are quite a few cases where the Janossy density has a closed-form. For more general cases, where 
the Janossy density lack explicit formulas, the process can sometimes be described via a collection of conditional probability densities. In this work, we focus on one such case known as \textbf{spatio-temporal point processes.}

Now assume $V = (0, T)\times S$, where $S\in \mathbb R^d$ is bounded. We order the points by time. For a sequence of points $\bm X$, let $(t_1, \ldots, t_{N(\bm X)})$ be the time coordinate of the ordered $N(\bm X)$ points, and let $(\bm s_1,\ldots, \bm s_{N(\bm X)})$ denote their corresponding locations. We define $p_n(t_n,\bm s_n|t_1,\bm s_1,\ldots, t_{n-1},\bm s_{n-1})$ as the conditional density of the (ordered) $n$-th point given past realized $n-1$ points. The following result shows the equivalence of specifying such a process by Janossy measure and by a set of conditional probabilities.
\begin{proposition}\label{prop:existence_of_conditional_densityd}\textnormal{(Proposition 7.2I in \citet{daley2003introduction})}
    For a regular point process on $(0,T)\times S$, there exists a uniquely determined family of conditional probability density functions $p_n(t,\bm s|t_1,\bm s_1,\ldots, t_{n-1},\bm s_{n-1})$ with the exception of $p_1(t,\bm s)$ and associated survivor functions $G_n$: 
    \begin{equation*}
        G_n(t|t_1,\bm s_1,\ldots, t_{n-1},\bm s_{n-1}):=1-\int_{t_{n-1}}^t\int_S p_n(u,\bm s|t_1,\bm s_1,\ldots, t_{n-1},\bm s_{n-1})d\bm sdu, 
    \end{equation*}
    where $0<t_1<\ldots<t_{n-1}<t, (\bm s_1,\ldots, \bm s_{n-1})\in S^{(n-1)}$, $p_n(\cdot, \cdot|t_1,\bm s_1,\ldots, t_{n-1},\bm s_{n-1})$ has support $(t_{n-1},T)\times S$, and for all $n\geq 1$,
    \begin{equation}\label{eq:hawkes_cond_density}
    \begin{aligned}
        J_0\big((0,T)\big)&=G_1(T), \\
        j_n(t_1,\bm s_1,\ldots, t_n,\bm s_n) &= p_1(t_1,\bm s_1)p_2(t_2,\bm s_2|t_1,\bm s_1)\ldots p_n(t_n,\bm s_n|t_1,\bm s_1,\ldots, t_{n-1},\bm s_{n-1})\\
        &\cdot G_{n+1}(T|t_1,\bm s_1,\ldots, t_{n},\bm s_{n}), 
    \end{aligned}
    \end{equation}
where $0<t_1<\ldots <t_n< T$. Conversely, given any such family of conditional densities $\{p_n\}$ for all $t>0, \bm s\in S$, \cref{eq:hawkes_cond_density} uniquely specify the distribution of a regular point process on $(0,T)\times S$.
\end{proposition}

We denote $D_n^T$ to be the transformed simplex $D_n^T :=\{(t_1\,\ldots ,t_n): 0 < t_1 < \ldots < t_n<T\}$. Let $\mathcal H_{n}:=(t_1,\ldots, t_n,\bm s_1,\ldots, \bm s_n)\in D_n^T\times S^{(n)}$ and $\mathcal H_{t_n}:=(t_1,\ldots, t_n,\bm s_1,\ldots, \bm s_{n-1})\in D_n^T\times S^{(n-1)}$. We further decompose the conditional density of $(t_n,\bm s_n)$ given $\mathcal H_{n-1}$ as: 
\begin{equation*}
    p_n(t_n,\bm s_n|\mathcal H_{n-1})=p_n(t_n|\mathcal H_{n-1})\times f_{S,n}(\bm s_n|\mathcal H_{t_n}). 
\end{equation*}
We also define a conditional intensity as $\lambda_{T,n}(t,\bm s|\mathcal H_{n-1}):= \frac{p_n(t,\bm s|\mathcal H_{n-1})}{G_n(t|\mathcal H_{n-1})}$. Throughout the paper, we denote $t_0 = 0$ and $t_{N+1}=T$ when $N(\mathcal X) = N$.  Addtionally, we set $p_1(t,\bm s|\mathcal H_0) = p_1(t,\bm s)$. 

A Janossy density can be explicitly expressed in terms of the conditional intensity. Therefore, using $\lambda_n(t, \bm s|\mathcal H_{n-1})$, we obtain an alternative expression of the likelihood as: 
\begin{equation}\label{eq:time_space_likelihood}
    L(\bm x_1,\ldots, \bm x_N) = \left[\prod_{n=1}^{N}\lambda_n(t_n,{\bm s}_n|\mathcal H_{n-1})\right]\exp\left(-\sum_{n=1}^{N+1}\int_{t_{n-1}}^{t_{n}}\int _S \lambda_n(t, \bm s|\mathcal H_{n-1}) d\bm s dt \right).
\end{equation}

\paragraph{Notation}
For a multivariate function $f(\bm{x}_1, \bm{x}_2)$, we sometimes need to consider it as a function of only one variable while keeping the other fixed. In such case, we write $f(\cdot, \bm x_2)$ to denote the function of $\bm x_1$ with $\bm x_2$ held fixed. For a function parameterized by $\theta$ in $\Theta$, we denote it as $f_{\theta}(\cdot)$. Unless otherwise stated, any statement regarding $f_{\theta}(\cdot)$ is assumed to hold for all $\theta \in \Theta$.

For the purposes of statistical inference, we usually assume that the data generating process is governed by a true parameter \( \theta^* \in \Theta\subset \mathbb R^r\) and a parameterized model, either \( j_\theta^*(\bm{X}) \) or \( p_{n,\theta^*}(t_n, \bm s_n | \mathcal{H}_{n-1}) \). To simplify notation, we omit the explicit dependence on the true parameter and write \( j(\bm{X}) \) and \( p_n(t_n, \bm s_n | \mathcal{H}_{n-1}) \) to denote the actual data generating process. Every expectation in this paper is always taken w.r.t. the true data generating probability.

Throughout the paper, we write $a.s.$ as almost everywhere w.r.t. $\mathbb P$, which is the data generating measure. We write $a.e.$ to represent almost everywhere w.r.t. the Lebesgue measure. $\|\cdot \|$ is used to represent Euclidean distance and $\mathbb L^p(\mathbb P)$ to represent $L^p$ space under base measure $\mathbb P$.

\subsection{Score Matching and Autoregressive Score Matching}

MLE is a classical estimator that minimizes Kullback–Leibler (KL) divergence between a model distribution and the data distribution. However, a major drawback is the intractability of the normalizing constant, whose approximation via numerical integration can be computationally expensive.
In contrast, SM~\citep{Hyvarinen05} offers an alternative by minimizing Fisher divergence between model and data distributions: 
\begin{equation*}
    \mathcal{L}^{\text{SM}}(\theta)=\frac{1}{2}\mathbb{E}\Vert\nabla_{\bm{x}}\log p(\bm{x})-\nabla_{\bm{x}}\log p_{\theta}(\bm{x})\Vert^2, 
    \label{esm}
\end{equation*}
where $p(\bm{x})$ represents the data distribution, $p_{\theta}(\bm{x})$ is the parameterized model distribution, the gradient of the log-density is called the score. Minimizing the Fisher divergence above provides the parameter estimate. 
The advantage of SM lies in its ability to bypass the computation of the normalizing constant since the score no longer contains this constant: $p_{\theta}(\bm{x})=\frac{1}{Z(\theta)}\tilde{p}_{\theta}(\bm{x})$ where $Z(\theta)=\int \tilde{p}_{\theta}(\bm{x})d\bm{x}$, $\nabla_{\bm{x}}\log p_{\theta}(\bm{x})=\nabla_{\bm{x}}\log\tilde{p}_{\theta}(\bm{x})$. 
Under certain conditions, we can use integration by parts to replace the explicit SM objective, which involves an unknown distribution $p(\bm{x})$, with an equivalent implicit one: 
\begin{equation}
\label{eq:ISM objective}
\mathcal{J}^{\text{SM}}(\theta)=\mathbb{E}\left[\frac{1}{2}\Vert\nabla_{\bm{x}}\log p_{\theta}(\bm{x})\Vert^2+\text{Tr}\left(\nabla^2_{\bm{x}}\log p_{\theta}(\bm{x})\right)\right]. 
\end{equation}

An autoregressive model defines a probability density $p(\bm{x})$ as a product of conditionals using the chain rule: $p(\bm{x})=\prod_{n=1}^N p(x_n|\bm{x}_{<n})$, 
where $x_n$ is the $n$-th entry and $\bm{x}_{<n}$ denotes the entries with indices smaller than $n$. 
The original SM is not suitable for autoregressive models because the autoregressive structure introduces challenges in gradient computation in \cref{eq:ISM objective}. 
To address this issue, \cite{meng2020autoregressive} proposed autoregressive score matching (ASM). Unlike SM, which minimizes the Fisher divergence between the joint distributions of the model $p_{\theta}(\bm{x})$ and the data $p(\bm{x})$, ASM minimizes the Fisher divergence between the conditionals of the model $p_{\theta}(x_n|\bm{x}_{<n})$ and the data $p(x_n|\bm{x}_{<n})$: 
\begin{equation*}
    \mathcal{L}^{\text{ASM}}(\theta)=\frac{1}{2}\sum_{n=1}^N\mathbb{E}\left(\frac{\partial\log p(x_n|\bm{x}_{<n})}{\partial x_n}-\frac{\partial\log p_{\theta}(x_n|\bm{x}_{<n})}{\partial x_n}\right)^2. 
\end{equation*}
Similarly, the above explicit ASM objective involves an unknown distribution $p(x_n|\bm{x}_{<n})$. Under specific regularity conditions, we can apply integration by parts to derive an implicit ASM objective: 
\begin{equation}
    \mathcal{J}^{\text{ASM}}(\theta)=\sum_{n=1}^N \mathbb{E}\left[\frac{1}{2}\left(\frac{\partial\log p_{\theta}(x_n|\bm{x}_{<n})}{\partial x_n}\right)^2+\frac{\partial^2 \log p_{\theta}(x_n|\bm{x}_{<n})}{\partial x_n^2}\right]. 
\label{asm}
\end{equation}

\section{Score Matching for Point Process}
\label{method}
In this section, we discuss score matching for regular finite point processes when Jannossy density has closed forms. 
The first motivation example is a simple inhomogeneous Poisson process on $(0,T)$, and we show why the original score matching loss fails to give a consistent estimator for the process. Then, we propose a provably consistent weighted score matching estimator for more general domain such as $V \subset \mathbb R^d$.

\subsection{Failure of Score Matching for Poisson Process}\label{sec:SM_fail}


Now we denote the temporal Poisson process as $\mathcal T\in(0,T)^{\cup}$. For a realization $\bm t = (t_1, \ldots, t_N)$ and the associated Janossy density $j_N(\bm t)$, denote the score function as $\psi_n(\bm t) := \partial_{t_n}\log j_N(\bm t)$. Through computation, we see $\psi_n(\bm t)=\frac{\partial}{\partial t_n}\log \lambda(t_n)$.
The associated model score would be $\bm \psi_{\theta}(\bm t)=(\psi_{1,\theta}(\bm t),\ldots, \psi_{N,\theta}(\bm t))$.

Previous works \citep{SahaniBM16,zhang2023integration} have both attempted to apply SM to the Poisson process: 
\begin{equation}
    \mathcal{L}^{\text{SM}}(\theta)=\frac{1}{2}\mathbb E\left[\sum_{n=1}^{N(\mathcal T)}\left(\psi_n(\mathcal T)-\psi_{n,\theta}(\mathcal T)\right)^2\right]. 
    \label{poisson_l_sm}
\end{equation}
Desirably, we hope this explicit score matching loss is equivalent to the implicit score matching loss defined as: 
\begin{equation}
    \mathcal{J}^{\text{SM}}(\theta)=\mathbb{E}\left[\sum_{n=1}^{N(\mathcal T)}\frac{1}{2}\psi^2_{n,\theta}(\mathcal T)+\partial_{t_n} \psi_{n,\theta}(\mathcal T)\right]. 
\label{eq:poisson_j_sm}
\end{equation}

Different from the situation in score matching for random variables in $\mathbb R^d$, the equivalence does not hold between explicit SM and implicit SM for most Poisson processes. Thus optimizing \cref{eq:poisson_j_sm} leads to a wrong estimate. Its failure is due to the fact that the specific regularity conditions cannot be satisfied for Poisson processes. 

Such conditions require the probability density of the random variable to be zero when it approaches infinity in any of its dimensions. However, for temporal Poisson process, such requirement is not satisfied, because the Janossy density is not zero at the boundary, and this comes from the nature that the temporal Poisson process $\mathcal T = (t_1,\ldots, t_{N(\mathcal T)})$ is not of fixed dimension and takes values in  $(0, T)^{\cup}$. Therefore, for general Poisson processes, we cannot derive the implicit SM in \cref{eq:poisson_j_sm} based on the explicit SM in \cref{poisson_l_sm}. 
We demonstrate this with an example below. 
\begin{example}
    Consider an inhomogeneous Poisson process with intensity $\lambda(t)=\rho t^{\rho-1}$ with $\rho$ being the ground-truth parameter. 
    The observation is $\{t_1^{(i)},\ldots, t_{N_i}^{(i)}\}_{i=1}^m$ where $m$ is the number of sequences. 
    We assume the model intensity is $\lambda_{\theta}(t)=\theta t^{\theta - 1}$, so we can compute the score as: $\psi_{n,\theta}(\bm t) = \frac{\theta - 1}{t_n}, \ \ \ \partial_{t_n} \psi_{n,\theta}(\bm t) = \frac{1- \theta}{t_n^2}$. 
    Substituting the above into \cref{eq:poisson_j_sm}, we obtain: 
    \begin{equation*}
        \hat{\mathcal J}^{\textnormal{SM}}(\theta) = \frac{1}{2}(\theta-1)(\theta-3)\sum_{i=1}^m\sum_{n=1}^{N_i}\frac{1}{(t_n^{(i)})^2}. 
    \end{equation*}
    Minimizing such an objective results in $\hat \theta = 2$. This estimate is even independent of the observations, which is clearly an erroneous estimate. 
\end{example}

\subsection{Weighted Score Matching}\label{sec:wsm}
To address the situation where SM fails, inspired by \cite{hyvarinen2007some,yu2019generalized, liu2022estimating}, we introduce the weighted score matching (WSM) for regular finite point processes. The core idea of WSM is to use a weight function that takes zero at the boundary of the integration region to make all terms associated with data density disappear in the final population loss.
To proceed, we sometimes require a more general notion of integration by parts in a multivariate space and a measure-theoretic manner. Thus, we introduce the following definition, which allows integration by parts to be performed in an almost everywhere sense.
\begin{definition}\textnormal{(Lipschitz Domain)}
    Let $V$ be an open and bounded domain in $\mathbb R^d$. We say $V$ is a Lipschitz domain if for any $\bm x\in \partial V$, there exists an open ball $B(\bm x, r)$  and a Lipschitz function $f(x_1,\ldots, x_{d-1})$ such that $V\cap B(\bm x, r)$  is expressed by $\{\bm z\in B(\bm x,r) | z_d>f(z_1,\ldots, z_{d-1})\}$, upon a transformation of the coordinate system if necessary. 
\end{definition}
More information about Lipschitz domain can be found in Section 7 of \citet{atkinson2005theoretical}. In real applications, most domains are Lipschitz and all polytopes are Lipschitz domains.

\begin{definition}\textnormal{(Sobolev Space)}
    Let $L^2(V)$ be the $L^2$ space on $V\subset \mathbb R^d$ endowed with the Lebesgue measure. Then, $H^1(V)$ is the Sobolev space defined by: 
    \begin{equation*}
        H^1(V):=\{f\in L^2(V) \mid \|f\|^2_{L^2(V)}+\sum_i\|D_i f\|^2_{L^2(V)}<\infty\}, 
    \end{equation*}
    where $D_i$ is the weak derivative corresponding to $\frac{\partial}{\partial x_i}$ and $\|f\|_{L^2(V)}=\sqrt {\int_V |f(\bm x)|^2d\bm x}$. 
\end{definition}
The notion of weak derivative is necessary here because later we will have  weight function that does not have derivative everywhere. And for differentiable function, the weak derivative is the same as derivative.

Now we generalize the temporal Poisson process to an open and bounded Lipschitz domain $V\subset \mathbb R^d$ and consider a regular finite point process on it. Suppose the process $\mathcal X$ has the ground-truth Janossy density denoted as $j(\bm X)$. To define the score function, we assume $j_N(\cdot), j_{N,\bm\theta}(\cdot)$ are continuously differentiable on $V^{(N)}$ and define the true score function as: 
\begin{equation*}
    \bm \psi_n: V^{\cup} \rightarrow \mathbb R^d, \quad
    \bm \psi_{n}(\bm X) := \begin{cases}
        \nabla_{\bm x_n}\log j_{N}(\bm X), &\text{ if } \bm X\in V^{(N)}, N\geq n,\\
        0, &\text{  if } N(\bm X) < n. 
    \end{cases}
\end{equation*}
We make the similar definition for model score $\bm \psi_{n,\theta}(\bm X)$. To ensure the integration by parts, we make the following assumptions on the model score. 

\begin{assumption}
\textnormal{(Regularity of Score Function)}
\label{assum:wsm_score_regularity}

\begin{itemize}
    \item (A1) $\sum_{n=1}^{N(\bm X)}\|\bm \psi_{n,\theta}(\bm X)\|^4$, $\sum_{n=1}^{N(\bm X)}(\bm \psi_{n}(\bm X)\cdot \bm \psi_{n,\theta}(\bm X))^2 \in L^1(\mathbb P)$; 
    \item (A2) For every $N\ge 1$ and every $1\le n\le N$, for a.e. 
$\bm X_{N,-n}\in V^{(N-1)}$, the maps
$\bm\psi_{n,\theta}(\cdot,\bm X_{N,-n})$ 
and $j_N(\cdot,\bm X_{N,-n})$ belong to $H^1(V)$.
\end{itemize}
\end{assumption}
The above assumptions are necessary for the existence of the population objective function, and also for the divergence theorem. They are regularity conditions that most elementary functions and finite process will satisfy.
Now we define a non-negative weight function $h(\bm x): V \rightarrow \mathbb{R}_{\geq0}$. For simplicity, we will only consider weight functions that can be continuously extended to the closure of $V$.  We require them to satisfy the following conditions. 
\begin{assumption}
\label{assum:wsm_regularity_of_weight}
\textnormal{(Requirements of Weight Function)}
\begin{itemize}
    \item (A1) $h \in H^1(V)$; 
    \item (A2) $h>0$ a.e. in $V$, and its continuous extension satisfies
$h(\bm x)=0$ for all $\bm x\in\partial V$;
    \item (A3) $\sum_{n=1}^{N(\bm X)}h^2(\bm x_n) \in L^1(\mathbb P)$. 
\end{itemize}
\end{assumption}
Conditions (A1) and (A2) ensure that the divergence theorem can be applied to eliminate the intractable surface integration. (A1) combined with (A3) and a Holder's inequality guarantees the existence of population loss.
Typically, (A1) serves as a regularity condition that most functions satisfy. As for (A2), it is generally satisfied by functions that take the value 0 on $\partial V$. For (A3), a sufficient condition is that $\mathbb E[N(\mathcal X)]<\infty$ and $h$ is bounded. 

With a valid weight function $h$, the explicit WSM objective can be defined as: 
\begin{equation}
\begin{aligned}
        &\mathcal L^{\textnormal{WSM}}_{h}(\theta) = \frac{1}{2}\mathbb E\left[\sum_{n=1}^{N(\mathcal X)}\|\bm \psi_{n}(\mathcal X)-\bm \psi_{n,\theta}( \mathcal X)\|^2 h(\bm x_n)\right]. 
\end{aligned}
\end{equation}
The introduction of the weight function allows control over the values of the integrand at the boundaries of the integration domain, thereby eliminating the intractable integration on the surface of the domain. The next theorem shows that minimizing this explicit loss does recover the ground-truth density in a parametric setting. 

\begin{theorem}\label{thm:unique_wsm}
    Suppose $\psi_{n, \theta_1}(\bm  X)=\psi_{n,\theta_2}(\bm X) \textnormal{ a.s. for }\forall n\in \mathbb N_+$ gives $\theta_1=\theta_2$, then the unique minimizer of $\mathcal L^{\textnormal{WSM}}_h(\theta)$ is $\theta^*$.
\end{theorem}
Unlike score matching for random variables in $\mathbb{R}^d$, where identifiability of the parametric family ensures that the minimizer of the score matching loss gives the true parameter, this is not sufficient for estimating point processes, as we will illustrate in \cref{section:unidentifiability}. Therefore, we require that the parameter is recovered whenever the score is matched. This condition holds for point process models that are (implicitly) specified through the score. For those  processes that do not meet the assumption, we provide a remedy in \cref{sec:WSM_remedy}.

The explicit WSM objective is not practical as it depends on the unknown data density $j(\bm X)$, so we further derive the implicit WSM objective which is tractable. 
\begin{theorem}\label{thm:equivalence between EWSM and IWSM}
Suppose \cref{assum:wsm_score_regularity,assum:wsm_regularity_of_weight} are satisfied and $V$ is a Lipschitz domain, then we have: 
    \begin{equation*}
        \mathcal L^{\textnormal{WSM}}_{h}(\theta)=\mathcal J^{\textnormal{WSM}}_{h}(\theta) + \textnormal{const},
    \end{equation*}
    \begin{equation}
    \begin{aligned}
                &\mathcal J^{\textnormal{WSM}}_h(\theta)=\mathbb E[q_{\theta,h}^{\textnormal{WSM}}(\mathcal X)]\\&=\mathbb E\left\{\sum_{n=1}^{N(\mathcal X)} \left[\left(\frac{1}{2}\|\bm \psi_{n,\theta}(\mathcal X)\|^2+\textnormal{Tr}\big(\nabla_{\bm x_n} \bm \psi_{n, \theta}(\mathcal X)\big)\right)h(\bm x_n)+\bm \psi_{n,\theta}(\mathcal X)\cdot \nabla h(\bm x_n) \right]\right\}. 
    \end{aligned}
    \label{poisson_j_wsm}
    \end{equation}
\end{theorem}
For regular finite point processes, \cref{poisson_j_wsm} is always a valid objective function with a suitable weight function. 
One thing we want to stress here is that this equivalence refers to the equality of two expectations. 
In reality, we use sample average to approximate $\mathcal J^{\textnormal{WSM}}_h(\theta)$. 
To ensure the suitability of this approximation, we consider a multiple-trajectory scenario in which we collect many i.i.d. realizations of a point process and rely on the law of large numbers. 

Now we give some parametric point processes and show how to compute their score functions or the associated loss functions.
\begin{example}
    \textnormal{(Inhomogeneous Poisson Process)} 
    Let $\lambda(\bm x): \mathbb R^d \rightarrow \mathbb R^+$ be a intensity function. The Poisson process on $V \subset \mathbb R^d$ is completely specified by the intensity function with Janossy density: 
    \begin{equation*}
        j_N(\bm X_N) = \prod_{n=1}^N \lambda(\bm x_n)\exp\left(-\int _{V}\lambda(\bm x)d\bm x \right). 
    \end{equation*}
    We parameterize the model intensity as $\lambda_{\theta}(\bm x)$. The model score function w.r.t. the $n$-th point is $\bm \psi_{n,\theta}(\bm X)=\nabla_{\bm x_n}\log \lambda_{\theta}(\bm x_n)$.  
\end{example}


\begin{example}
    \textnormal{(Exponential Family Point Process)}\label{ex:exponential}
    A class of point process has Janossy density with the following form,
    \begin{equation*}
        \log j_{\theta}(\bm X) = \theta^\top T(\bm X)-c(\theta)+b(\bm X),
    \end{equation*}
    where $T(\bm X)\in \mathbb R^p$. $T(\bm X)$ and $b(\bm X)$ are symmetric w.r.t. every point in $\bm X$. This is usually referred to as the exponential family.  In the case where $T(\bm X)$ and $b(\bm X)$ are differentiable (restricted on $V^{(N)}$), similar to \citet{yu2019generalized}, the empirical objective has a closed-form quadratic form.

\begin{align*}
    \hat{\mathcal J}_{h}^{\textnormal{WSM}}(\theta)&=\frac{1}{2}\theta^\top \hat {\bm \Gamma}\theta - \hat {\bm g}^\top\theta+\textnormal{const}\\
    \hat {\bm \Gamma}&:=\frac{1}{m}\sum_{i=1}^m\sum_{n=1}^{N_i}h(\bm x_n^{(i)})\bm S_n^{(i)}{\bm S_n^{(i)\top}}\\
    \bm {\hat g}&:= \frac{1}{m}\sum_{i=1}^m\sum_{n=1}^{N_i}\Big[h(\bm x_n^{(i)})\bm S_n^{(i)}\bm b_n^{(i)}+\textnormal{Trs}(\bm H_n^{(i)})h(\bm x_n^{(i)})+\bm S_n^{(i)}\nabla_{\bm x_n} h(\bm x_n^{(i)})\Big],
\end{align*}
where $m$ is the number of trajectories and $N_i$ is the length of i-th trajectory, and
\begin{align*}
    &\bm S_n^{(i)}:=\nabla_{\bm x_n^{(i)}}T(\bm X^{(i)})\in \mathbb R^{p\times d} , \bm b_n^{(i)}:=\nabla_{\bm x_n^{(i)}}b(\bm X^{(i)})\in \mathbb R^d,\\
    &\bm H_n^{(i)}:=\nabla^2_{\bm x_n^{(i)}}T(\bm X^{(i)})\in \mathbb R^{p\times d\times d}, \textnormal{Trs}(\bm H_n^{(i)}):=\sum_{j=1}^{d}\bm H_n^{(i)}[:,j,j]\in \mathbb R^p.
\end{align*}
\end{example}

\section{Autoregressive Score Matching for Point Process}
\label{method2}
Now we consider spatio-temporal point process whose conditional density class or conditional intensity are of closed form and the Janossy density is not.  In this case, if we directly use (weighted) score matching, we would still end up with a loss function with integral. For example, directly taking score to \cref{eq:time_space_likelihood} gives 
\begin{align*}
    \nabla_{\bm x_n}\log L(\bm x_1,\ldots, \bm x_N) &=  \sum_{i=n}^N\Big[\nabla_{\bm x_n}\log\lambda_i(t_i, \bm s_i|\mathcal H_{i-1})\Big] - \nabla_{\bm x_n}\int_{t_{n-1}}^{t_n}\int_S \lambda_n(t,\bm s|\mathcal H_{n-1})dtd\bm s \\&- \sum_{i=n+1}^{N+1}\int_{t_{i-1}}^{t_i}\int_S \nabla_{\bm x_n}\lambda_i(t,\bm s|\mathcal H_{i-1})dtd\bm s.
\end{align*}
We can see that they still contain integral, and thus the methods from previous discussion is not suitable here. This can be resolved by performing ASM conditioning on the history. 

In this section, we first introduce the existing ASM loss for temporal Hawkes processes and demonstrate its inconsistency in parameter estimation. 
Subsequently, we propose a provably consistent autoregressive WSM (AWSM) estimator for regular finite spatio-temporal point process specified by conditional intensity.

\subsection{Failure of Autoregressive Score Matching for Hawkes Process}\label{sec:ASM_fail}
Consider a temporal Hawkes process with ordered timestamps $(t_1,\ldots, t_{N(\mathcal X)})$ on $(0,T)$ with the underlying conditional density of $t_n$ denoted as $p_n(t_n|\mathcal H_{n-1})$. The parameterized conditional probability density model of $t_n$ is $p_{n, \theta}(t_n|\mathcal H_{{n-1}})$. 
We denote the conditional score as: 
\begin{equation}\label{eq:conditional_temporal_intensity}
    \psi_n(t_n|\mathcal H_{{n-1}}):=\partial_{t_n}\log p_n(t_n|\mathcal H_{{n-1}})=\partial_{t_n}\log \lambda_n(t_n|\mathcal H_{{n-1}}) - \lambda_n(t_n|\mathcal H_{{n-1}}). 
\end{equation}
An explicit ASM objective is defined as: 
\begin{equation}\label{eq:smurf loss}
\begin{aligned}
        &\mathcal L^{\text{ASM}}(\theta) =\frac{1}{2}\mathbb E\left[\sum_{n=1}^{N(\mathcal X)}(\psi_n(t_n|\mathcal H_{{n-1}}) - \psi_{n,\theta}(t_n|\mathcal H_{{n-1}}))^2\right].
\end{aligned}
\end{equation}
Similarly, to make ASM practical, \cite{li2023smurf} asserted that such a loss is equivalent to a implicit version of ASM: 
\begin{equation}
\begin{aligned}
       &\mathcal J^{\text{ASM}}(\theta) = \mathbb E\left[\sum_{n=1}^{N(\mathcal X)}\frac{1}{2}\psi^2_{n,\theta}(t_n|\mathcal H_{{n-1}})+\partial_{t_n}\psi_{n,\theta}(t_n|\mathcal{H}_{{n-1}})\right]. 
\end{aligned}
\label{J_asm}
\end{equation}
However, the same issue as in the temporal Poisson process arises here. The regularity conditions required to eliminate the unknown data distribution do not hold. 
Therefore, we cannot derive the implicit ASM in \cref{J_asm} based on the explicit ASM in \cref{eq:smurf loss}.

\subsection{Autoregressive Weighted Score Matching}\label{sec:awsm}

Now we move to the more general spatio-temporal domain and consider the corresponding estimator. Now $V=(0,T)\times S$. We consider the temporal conditional density of $p_n(t_n|\mathcal H_{n-1})$ and spatial conditional density $f_{S,n}(\bm s_n|\mathcal H_{t_n})$. We denote the corresponding score as: 
\begin{align*}
    \psi_{T,n}: D_{n}^T\times S^{(n-1)}\rightarrow \mathbb R,  \quad
    \psi_{T, n}(t_n|\mathcal H_{n-1}) &:= \partial_{t_n}\log p_n(t_n|\mathcal H_{n-1}), 
    \\
    \bm \psi_{S, n}:D_n^T\times S^{(n)} \rightarrow \mathbb R^d, \quad \bm \psi_{S, n}(\bm s_n|\mathcal H_{t_n}) &:= \nabla_{\bm s_n}\log f_n(\bm s_n|\mathcal H_{t_n})
    .
\end{align*}
And $\psi_{T,n,\theta}$ and $\bm \psi_{S,n,\theta}$ are defined in the same way. Another thing to notice is that both $\psi_{T,n}$ and $\bm \psi_{S,n}$ can also be interpreted as functions on $V^{\cup}$ since they take ordered points as input. We will write $\psi_{T,n}(\bm X)$ and $\bm \psi_{S,n}(\bm X)$ when we need to discuss it as a function on $V^{\cup}$.

Similar to \cref{assum:wsm_score_regularity}, we make necessary assumptions on those score functions. 
\begin{assumption}
\label{assum:AWSM_score_regularity}
\textnormal{(Regularity of Conditional Score Function)}
\begin{itemize}
    \item (A1) $\sum_{n=1}^{N(\bm X)}\psi^4_{T,n,\theta}(t_n|\mathcal H_{n-1})$, $\sum_{n=1}^{N(\bm X)}\psi^2_{T,n}(t_n|\mathcal H_{n-1})\psi^2_{T,n,\theta}(t_n|\mathcal H_{n-1})$, $\sum_{n=1}^{N(\bm X)}\|\bm \psi_{S,n,\theta}(\bm s_n|\mathcal H_{t_n})\|^4$ and $\sum_{n=1}^{N(\bm X )}[\bm \psi_{S,n,\theta}(\bm s_n|\mathcal H_{t_n})\cdot \bm \psi_{S,n}(\bm s_n|\mathcal H_{t_n})]^2 \in L^1(\mathbb P)$; 
    \item (A2) $\forall n$, for a.e. $\mathcal H_{n-1}\in D_{n-1}^T\times S^{(n-1)} , \psi_{T,n,\theta}(\cdot|\mathcal H_{n-1}) \textnormal{ and } p_n(\cdot|\mathcal H_{n-1})\in H^1((t_{n-1},T))$, 
    $\forall n$, for a.e. $\mathcal H_{t_n}\in D^T_{n}\times S^{(n-1)}$, $\bm \psi_{S,n,\theta}(\cdot|\mathcal H_{t_n})$ and $f_n(\cdot|\mathcal H_{t_n}) \in H^1(S)$. 
\end{itemize}
\end{assumption}

Similarly, instead of directly matching those scores, we also have to add weight functions to make the loss function tractable. We define the temporal weight function as $h_T: D_2^T \to \mathbb{R}_{\geq0}$, and the spatial weight function as $h_S: V \to \mathbb{R}_{\geq0}$. Again, we only consider functions that can be continuously extended to the closure and they must satisfy the following conditions. 
\begin{assumption}\label{assum:awsm_regularity_of_weight}\textnormal{(Requirements of Weight Functions)}
\begin{itemize}
    \item (A1) For a.e. $t_{n-1}\in (0,T)$, $h_T(t_{n-1},\cdot ) \in H^1((t_{n-1},T))$, $h_S\in H^1(S)$; 
    \item (A2) $h_T$, $h_S>0$ a.e. $h_T(t_{n-1},t_n)=0$ when $t_n=t_{n-1}$ or $T$, $h_S(\bm s)=0$ when $\bm s\in \partial S$; 
    \item (A3) $\sum_{n=1}^{N(\bm X)}h_T^2(t_{n-1}, t_n)$, $\sum_{n=1}^{N(\bm X)}h^2_S(\bm s_n)\in L^1(\mathbb P)$. 
\end{itemize}
\end{assumption}

With a valid weight function, the explicit AWSM objective is defined as: 
\begin{equation}
\label{eq:Explicit Autoregressive Weighted Score Matching}
\begin{aligned}
\mathcal L^{\text{AWSM}}_{h_T,h_S}(\theta) = \frac{1}{2}\mathbb E\Bigg\{\sum_{n=1}^{N(\mathcal X)}\Big[&\Big (\psi_{T,n}(t_n|\mathcal H_{{n-1}})-\psi_{T,n,\theta}(t_n|\mathcal H_{n-1})\Big)^2h_{T}(t_{n-1}, t_{n})+\\&\|\bm \psi_{S,n}(\bm s_n|\mathcal H_{t_n})-\bm \psi_{S,n, \theta}(\bm s_n|\mathcal H_{t_n})\|^2h_{S}(\bm s_n)\Big]\Bigg\}. 
\end{aligned}
\end{equation}
The next theorem shows that under necessary assumptions, explicitly matching the conditional score recovers the true parameter. 
\begin{theorem}\label{thm:unique_awsm}
     Suppose $\psi_{T,n,\theta_1}(\bm X) =  \psi_{T,n,\theta_2}(\bm X)$ and ${\bm \psi}_{S,n,\theta_1}(\bm X) ={\bm \psi}_{S,n,\theta_2}(\bm X)$ a.s. for  $\forall n$ gives $\theta_1=\theta_2$. Then the unique minimizer of $\mathcal L_{h_T,h_S}^{\textnormal{AWSM}}(\theta)$ is  $\theta^*$. 
\end{theorem}
Again, here we require a strong assumption on the parametric structure of the model, such that same conditional score gives same parameter value. This is generally true in parametric (statistical) point process models, but does not hold for nonparametric (e.g. deep point process) models. We discuss the issue and propose a remedy in \cref{sec:integration-free-training}.

The explicit AWSM objective is not practical as it depends on the unknown data distribution $p_n(t_n|\mathcal H_{{n-1}})$ and $f_{S,n}(\bm s_n|\mathcal H_{t_n})$, so we further derive the implicit AWSM objective which is tractable. 
\begin{theorem}
\label{thm:tractable EWSM}
Assume \cref{assum:AWSM_score_regularity,assum:awsm_regularity_of_weight} are satisfied and $S$ is a Lipschitz domain, then we have: 
\begin{equation*}
    \begin{aligned}
        \mathcal L^{\textnormal{AWSM}}_{h_T,h_S}(\theta) &= \mathcal J^{\textnormal{AWSM}}_{h_T,h_S}(\theta) + \textnormal{const}, 
    \end{aligned}
\end{equation*}
where
\begin{equation}
\label{eq:implicit autoregressive weighted score matching}
\begin{aligned}
\mathcal J^{\textnormal{AWSM}}_{h_T,h_S}(\theta) &= \mathcal J^{\textnormal{AWSM,T}}_{h_T}(\theta) + \mathcal J^{\textnormal{AWSM,S}}_{h_S}(\theta),\\
\mathcal J^{\textnormal{AWSM,T}}_{h_T}(\theta) &=\mathbb E[q^{\textnormal{AWSM,T}}_{h_T,\theta}(\mathcal X)]= \mathbb E\Bigg\{\sum_{n=1}^{N(\mathcal X)}\Big[\frac{1}{2}\psi_{T,n,\theta}^2(t_n|\mathcal H_{n-1})h_T(t_{n-1}, t_n)\\&+\partial_{t_n} \psi_{T, n, \theta}(t_n|\mathcal H_{{n-1}})h_T(t_{n-1}, t_n) +
\psi_{T,n,\theta}(t_n|\mathcal H_{n-1})\partial_{t_n} h_T(t_{n-1},t_n)\Big]\Bigg\},\\
\mathcal J^{\textnormal{AWSM,S}}_{h_S}(\theta)&=\mathbb E[q^{\textnormal{AWSM,S}}_{h_S,\theta}(\mathcal X)]= \mathbb E\Bigg\{\sum_{n=1}^{N(\mathcal X)}\Big[\frac{1}{2}\|\bm \psi_{S,n,\theta}(\bm s_n|\mathcal H_{t_n})\|^2h_{S}(\bm s_n)\\&+\textnormal{Tr}(\nabla_{\bm s_n} \bm \psi_{S,n,\theta}(\bm s_n|\mathcal H_{t_{n}})) h_{S}(\bm s_n)+\bm \psi_{S,n, \theta}(\bm s_n|\mathcal H_{t_n})\cdot \nabla h_S(\bm s_n) \big]\Bigg\}. 
\end{aligned}
\end{equation}
\end{theorem}

For general spatio-temporal point processes, \cref{eq:implicit autoregressive weighted score matching} is always valid with a suitable weight function. Thus, we do not need to worry about the issues of failure that may arise when using \cref{J_asm}. 

\begin{remark}\label{remark:link_intensity_to_score}
We briefly discuss how the conditional score relates to the conditional intensity. Many parametric spatio-temporal point process models admit a closed-form conditional intensity 
$\lambda_{n,\theta}(t,\bm s \mid \mathcal H_{n-1})$. For example, a self-exciting model \citep{reinhart2018review} can be written as
\begin{equation*}
  \lambda_{n,\theta}(t,\bm s \mid \mathcal H_{n-1})
  = \mu_{\theta}(\bm s) + \sum_{i<n} g_{\theta}(t-t_i,\bm s-\bm s_i),
  \quad t\in(t_{n-1},T),\ \bm s\in S,
\end{equation*}
where $\mu_\theta$ and $g_\theta$ are nonnegative.
In such cases we can define the conditional temporal intensity and spatial density as
$\lambda_{T,n,\theta}(t \mid \mathcal H_{n-1})
  := \int_S \lambda_{n,\theta}(t,\bm s \mid \mathcal H_{n-1})\,d\bm s$
and
$f_{S,n,\theta}(\bm s \mid \mathcal H_{t_n})
  := \lambda_{n,\theta}(t_n,\bm s \mid \mathcal H_{n-1}) /
     \lambda_{T,n,\theta}(t_n \mid \mathcal H_{n-1})$.
Then the spatial score always has the closed-form expression
$\bm{\psi}_{S,n,\theta}(\bm{s}_n \mid \mathcal{H}_{t_n})
 :=\nabla_{\bm s_n}\log f_{S,n,\theta}(\bm s_n|\mathcal H_{n-1})= \nabla_{\bm{s}_n}\log \lambda_{n,\theta}(t_n, \bm{s}_n \mid \mathcal{H}_{n-1})$. In contrast, the temporal score typically has a closed-form expression via \cref{eq:conditional_temporal_intensity} only when $\lambda_{T,n,\theta}$ itself is available in closed form.

However in the nonparametric case, we instead directly model the conditional intensity as
\begin{equation*}
  \lambda_{n,\theta}(t,\bm s \mid \mathcal H_{n-1})
  = \lambda_{T,n,\theta}(t \mid \mathcal H_{n-1})\,
    \frac{f_{S,n,\theta}(\bm s \mid \mathcal H_{t_n})}
         {\int_S f_{S,n,\theta}(\bm u \mid \mathcal H_{t_n})\,d\bm u},
\end{equation*}
where $\lambda_{T,n,\theta}$ and $f_{S,n,\theta}$ are some functions we pick with closed form. In this parametrization both temporal and spatial scores are available in closed form, so the AWSM loss is fully tractable, and integrals only appear at the inference stage whn we evaluate the intensity.
\end{remark}

 \section{Theoretical Analysis}
\label{sec:Theory}
In this section, we study the statistical properties of (A)WSM estimators for finite point processes. In \cref{consistency}, we examine the consistency and asymptotic normality of our estimator. In \cref{sec:non_asymp}, we establish the finite-sample convergence rate. In \cref{section:optimal_weight_hawkes}, we discuss the choice of the weight function in relation to the finite-sample results we establish. In \cref{sec:change_variable}, we explore variable transformation as an alternative approach and its relationship with weighted score matching. 

When we discuss the statistical theory of (A)WSM estimators, we always consider the multiple trajectory setting. To be specific, we collect $m$ trajectories where the $i$-th trajectory is denoted as $\{\bm x_1^{(i)}, \ldots , \bm x_{N_i}^{(i)}\}$ where $N_i$ is the length of that sequence. We assume the true density is in the family of the model density, denoted as $j(\bm X)=j_{\theta^*}(\bm X)$ or $p_{n,\theta^*}(t_n|\mathcal H_{n-1})$ and $f_{n,\theta^*}(\bm s_n|\mathcal H_{t_n})$. 
We assume $\theta^* $ is contained in a open set inside a compact set $\Theta \subset \mathbb R^p$. The estimate $\hat \theta_m$ is obtained by $\hat \theta_m = \argmin_{\theta \in \Theta}\hat {\mathcal J}(\theta)$ where $\hat {\mathcal J}$ represents the empirical loss. 

\subsection{Asymptotic Property}
\label{consistency}
In this section, we prove that under the parametric setting, optimizing an (A)WSM objective yields a consistent and asymptotically normal estimator under suitable regularity conditions. Our proof builds on the classical theory of M estimator, leveraging existing results under appropriate smoothness assumptions. The technical details can be found in \cref{appendix:proof_of_asymptotic}.

For the WSM estimator, we have the following theorem. 
\begin{theorem}
\label{thm:WSM_consistency}
    Suppose the empirical minimizer of $\hat {\mathcal J}^{\textnormal{WSM}}_{h}(\theta)$ is $\hat {\theta}_m$, under mild regularity \cref{assum:wsm_asymp} and assume \cref{thm:unique_wsm}, \cref{thm:equivalence between EWSM and IWSM} hold, we have $\hat {\theta}_m \xrightarrow{\mathbb P} \theta^*$. Furthermore, we have,
         \begin{equation*}
          \sqrt m(\hat {\theta}_m - \theta^*) \xrightarrow[]{w} \mathcal N(\bf 0, \bm \Sigma^{\textnormal{WSM}}(\theta^*)),
     \end{equation*}
    where $\bm \Sigma^{\textnormal{WSM}}(\theta^*)=\nabla_{\theta} ^2 \mathcal J_h^{\textnormal{WSM}}(\theta^*)^{-1}\mathbb E[\nabla_{\theta}q^{\textnormal{WSM}}_{h,\theta^*}(\mathcal X)\nabla_{\theta}q^{\textnormal{WSM}}_{h,\theta^*}(\mathcal X)^\top]\nabla_{\theta} ^2 \mathcal J_h^{\textnormal{WSM}}(\theta^*)^{-1}$.
\end{theorem}

For the AWSM estimator, we have the following theorem. 
\begin{theorem}\label{thm:AWSM_consistency}
    Suppose the empirical minimizer of $\hat {\mathcal J}^{\textnormal{AWSM}}_{h_T,h_S}(\theta)$ is $\hat {\theta}_m$, under mild regularity \cref{assu:asymp_awsm} and assume \cref{thm:unique_awsm}, \cref{thm:tractable EWSM} hold, we have $\hat {\theta}_m\xrightarrow{\mathbb P} \theta^*$. Furthermore, we have,
    \begin{equation*}
        \sqrt m(\hat{\theta}_m-\theta^*) \xrightarrow{w}\mathcal N(\bm 0, \bm \Sigma^{\textnormal{AWSM}}(\theta^*)),
    \end{equation*}
    where $\bm \Sigma^{\textnormal{AWSM}}(\theta^*)=\nabla_{\theta}^2\mathcal J^{\textnormal{AWSM}}_{h_T,h_S}(\theta^*)^{-1}\mathbb E\big[\nabla_{\theta}q^{\textnormal{AWSM}}_{h_T,h_S,\theta^*}(\mathcal X) \nabla_{\theta}q^{\textnormal{AWSM}}_{h_T,h_S,\theta^*}(\mathcal X)^\top\big]\nabla_{\theta}^2\mathcal J^{\textnormal{AWSM}}_{h_T,h_S}(\theta^*)^{-1}$.
\end{theorem}

\subsection{Non-asymptotic Error Bound} 
\label{sec:non_asymp}
To establish the non-asymptotic result, we begin with two definitions that capture the common structure of our loss functions. We provide a heuristic explanation of these definitions and results at the end of this section.

\begin{definition}\label{def:strong_identif}
    For a function $q:\Theta\times V^{\cup}\rightarrow \mathbb R$, we say $q$ is $(\underline{C},\alpha)$ strongly identifiable if its expectation $\mathcal J(\theta)=\mathbb E[q_{\theta}(\mathcal X)]$ satisfies that there exists $\alpha>1$ such that, 
    \begin{equation*}
                \mathop{inf}_{\theta: ||\theta-\theta^*||\geq \delta} \mathcal J(\theta)-\mathcal J(\theta^*)\geq \underline{C}\delta^{\alpha},
    \end{equation*}
    holds for any small $\delta$. Here $\underline{C}$ is a constant depending on $q$. 
\end{definition}

Since $\underline{C}$ depends on the loss function, the choice of weight function in WSM and AWSM will affect $\underline{C}$. Later we will denote this constant by $C_h$ and $C_{h_T,h_S}$, respectively.

\begin{definition}
    A function $\bm g:\Theta\times V^{\cup}\rightarrow \mathbb R^d$ has Lipschitz modulus $\dot g$ if there exists measurable $\dot g:V^{\cup}\rightarrow \mathbb R^+$ such that, 
    \begin{equation*}
        \|\bm g_{\theta_1}(\bm X)-\bm g_{\theta_2}(\bm X)\|\leq \dot g(\bm X)\|\theta_1-\theta_2\|.  
    \end{equation*}
\end{definition}

Building on the above two definitions, we can establish a convergence rate for loss functions that share a common structure. We will further illustrate how these assumptions are satisfied in the case of the exponential family in \cref{exp:non_asymp_example}. For a specific instance, namely the WSM estimator, we present the corresponding convergence rate below.

\begin{theorem}\label{thm:WSM_bound}
    Assume $\mathcal J^{\textnormal{WSM}}_{h}(\theta)$ to be $(C_h, \alpha)$ strongly identifiable. If $\forall n$, $\frac{1}{2}\|\bm \psi_{n,\theta}(\bm X)\|^2+\textnormal{Tr}(\nabla_{\bm x_n}\bm \psi_{n,\theta}(\bm X))$ and $\bm \psi_{n,\theta}(\bm X)$ have Lipschitz modulus $\dot A_{n}(\bm X)$ and $\dot B_n(\bm X)$. Then given $\hat {\theta}_m$ converges to $\theta^*$ in probability,  for $\delta < CK_{\alpha}\sqrt {p}\frac{\Gamma(h, \dot A, \dot B)}{C_{h}}$, we have 
\begin{equation}
    \textnormal{Pr}\left[||\hat {\theta}_m - \theta^*|| \leq \Big(CK_{\alpha}\frac{\Gamma(h, \dot A, \dot B)}{\delta C_{{h}}}\sqrt{\frac{p}{m}}\Big)^{1/(\alpha-1)}\right]\geq 1-\delta,
    \label{wsm_upperbound}
\end{equation}
where $\Gamma(h,\dot A,\dot B)=\Big\|\sum_{n=1}^{N(\bm X)}\big[\dot A_n(\bm X) h(\bm x_n) +  \dot B_n(\bm X) \|\nabla_{\bm x_n} h(\bm x_n)\|\big]\Big\|_{L^2(\mathbb P)}$, $C$ is a universal constant, $K_{\alpha}=\frac{2^{2\alpha}}{2^{\alpha-1}-1}$, $p$ is the dimension $\Theta$.
\end{theorem}

For the next theorem, we establish the result for AWSM estimator.
\begin{theorem}\label{thm:AWSM_bound}
    Assume $\mathcal J^{\textnormal{AWSM}}_{h_T,h_S}$ is $(C_{h_S,h_T},\alpha)$ strongly identifiable. If $\forall n,$ $\frac{1}{2}\psi_{T,n,\theta}^2(t_n|\mathcal H_{n-1})+\partial_{t_n} \psi_{T,n,\theta}(t_n|\mathcal H_{n-1})$ and $\psi_{T,n,\theta}(t_n|\mathcal H_{n-1})$ have Lipschitz modulus $\dot A_{n}(t_n,\mathcal H_{n-1})$ and $\dot B_n(t_n,\mathcal H_{n-1})$. And if $\forall n$, $\frac{1}{2}\|\bm \psi_{S,n,\theta}(\bm s_n|\mathcal H_{t_n})\|^2+\textnormal{Tr}(\nabla_{\bm s_n} \bm \psi_{S,n,\theta}(\bm s_n|\mathcal H_{t_{n}}))$ and $\bm \psi_{S,n, \theta}(\bm s_n|\mathcal H_{t_n})$ have Lipschitz modulus $\dot C_n(t_n,\mathcal H_{n-1})$ and $ \dot D_n(t_n,\mathcal H_{n-1})$. Given that $\hat {\theta}_m$ converges to $\theta^*$ in probability,  for $\delta < CK_{\alpha}\sqrt {p}\frac{\Gamma(h_S, h_T, \dot A, \dot B, \dot C, \dot D)}{C_{h_s, h_T}}$, we have 
\begin{equation}
    \textnormal{Pr}\left[||\hat {\theta}_m - \theta^*|| \leq \left(CK_{\alpha}\frac{\Gamma(h_S,h_T, \dot A, \dot B, \dot C,\dot D)}{\delta C_{h_S,h_T}}\sqrt{\frac{p}{m}}\right)^{1/(\alpha-1)}\right]\geq 1-\delta,
    \label{awsm_upperbound}
\end{equation}
where $C$ is a universal constant, $K_{\alpha}=\frac{2^{2\alpha}}{2^{\alpha-1}-1}$,  and
\begin{align*}
        \Gamma(h_S,h_T,\dot A,\dot B, \dot C, \dot D)=\Big\|\sum_{n=1}^{N(\mathcal X)}\big[&\dot A_n(t_n,\mathcal H_{n-1}) h(t_{n-1}, t_n) +  \dot B_n(t_n,\mathcal H_{n-1}) \partial_{t_n} h(t_{n-1}, t_n)\\
        &+\dot C_n(\bm X_n) h_S(\bm s_n) + \dot D_n(\bm X_n)\|\nabla_{\bm s_n}h_S(\bm s_n)\|\big]\Big\|_{L^2(\mathbb P)}.
\end{align*}
\end{theorem}

The above two theorems are standard finite-sample bounds for
parametric $M$-estimators proved using textbook empirical-process
arguments (e.g., \citet{van2000asymptotic}; see also \citet{liu2022estimating}). We instantiate it in our specific probability space $(V^{\cup},\mathbb P)$.

In terms of interpretation, in most parametric models the population
criterion is locally quadratic around $\theta^\ast$, so the strong
identifiability exponent in \cref{def:strong_identif} is typically
$\alpha=2$. In this case the bounds in
Theorems~\ref{thm:WSM_bound}–\ref{thm:AWSM_bound} reduce to the usual
order
\[
  \|\hat\theta_m - \theta^\ast\|
  = \Theta\!\Big(\frac{\Gamma}{C_h}\sqrt{\frac{p}{m}}\Big),
\]
 where $p$ is
the parameter dimension. We give a concrete example in \cref{exp:non_asymp_example}.

Intuitively, the constants $C_h$ and $C_{h_T,h_S}$ measure the local
curvature of the population objective
$\mathcal J_h$ around $\theta^\ast$: a larger $C_h$ or
$C_{h_T,h_S}$ means the loss separates the true parameter from nearby
alternatives more sharply, leading to easier identification and tighter
bounds. The quantity $\Gamma$ in the numerator comes from controlling
the empirical process term, and is determined by the Lipschitz
modulus of the loss. If this modulus is
large, the loss reacts more strongly to random perturbations in the
data, and the resulting
finite-sample bound becomes correspondingly worse.

\subsection{Discussion on Optimal Weight Function}
\label{section:optimal_weight_hawkes}

In \cref{method,method2}, we only provide the conditions that the weight function needs to satisfy. In fact, there are many weight functions that satisfy these conditions. The optimal weight function should minimize the error bound in \cref{wsm_upperbound,awsm_upperbound}, which is equivalent to minimizing the coefficient $\frac{\Gamma(h, \dot A, \dot B)}{C_h}$ and $\frac{\Gamma(h_S,h_T,\dot A,\dot B,\dot C,\dot D)}{C_{h_S,h_T}}$. 
The numerator cannot be analytically computed as it involves an unknown distribution $\mathbb P$, but we can maximize the denominator $C_{h}$ and $C_{h_S,h_T}$ in a predefined function class. 

We define three function classes as follows, 
\begin{equation*}
\begin{aligned}
        \mathcal H(V,L) &:= \{
        h:V\rightarrow \mathbb R^+\text{ is }L \text{ Lipschitz, h satisfies \cref{assum:wsm_regularity_of_weight}} \},\\
        \mathcal H(T,L)&:=\{h: D_2^T\rightarrow \mathbb R^+ \text{ is }L \text{ Lipschitz w.r.t. its second argument, } h \text{ satisfies \cref{assum:awsm_regularity_of_weight}}\},\\
        \mathcal H(S,L) &:= \{
        h:S\rightarrow \mathbb R^+\text{ is }L \text{ Lipschitz, h satisfies \cref{assum:awsm_regularity_of_weight}} \},
\end{aligned}
\end{equation*}
then we have the following result regarding the optimality of the weight function in those function classes.

\begin{proposition} \label{prop:distance_minimize_local}
Suppose \cref{assum:wsm_score_regularity,assum:AWSM_score_regularity} hold, for finite point processes with a finite moment, i.e., $\mathbb E[N(\mathcal X)]<\infty$,  we have,

    \begin{equation*}
    \begin{aligned}
        h^{\textnormal{WSM}}_0:=\textnormal{dist}(x, \partial V)&=\mathop{\textnormal{argmax}}_{h\in \mathcal H(V,1)}\mathcal L^{\textnormal{WSM}}_{h}(\theta), \forall \theta\in \Theta\\
        &=\mathop{\textnormal{argmax}}_{h\in \mathcal H(V,1)}\inf_{\theta:\|\theta-\bm\theta^*\|\geq \delta}[\mathcal J^{\textnormal{WSM}}_{h}(\theta)-\mathcal J^{\textnormal{WSM}}_{h}(\theta^*)],\\
        h^{\textnormal{AWSM}}_{0,S}:=\textnormal{dist}(\bm s,\partial S)&=\mathop{\textnormal{argmax}}_{h\in \mathcal H(S,1)}\mathcal L^{\textnormal{AWSM},S}_{h}(\theta), \theta\in \Theta\\
        &=\mathop{\textnormal{argmax}}_{h\in \mathcal H(S,1)}\inf_{\theta:\|\theta-\bm\theta^*\|\geq \delta}[\mathcal J^{\textnormal{AWSM},S}_{h}(\theta)-\mathcal J^{\textnormal{AWSM},S}_{h}(\theta^*)],
        \\
        h^{\textnormal{AWSM}}_{0,T}:=\textnormal{dist}(t_n, \partial(t_{n-1},T))&=\mathop{\textnormal{argmax}}_{h\in \mathcal H(T,1)}\mathcal L^{\textnormal{AWSM,T}}_h(\theta),\theta\in \Theta\\
        &=\mathop{\textnormal{argmax}}_{h\in \mathcal H(T,1)}\inf_{\theta:\|\theta-\bm\theta^*\|\geq \delta}[\mathcal J^{\textnormal{AWSM},T}_{h}(\theta)-\mathcal J^{\textnormal{AWSM},T}_{h}(\theta^*)]. 
        \\
    \end{aligned}
    \end{equation*} 
\end{proposition}

Combined with the convergence result in \cref{wsm_upperbound,awsm_upperbound}, it can be observed that $h_0^{\textnormal{WSM}}$ maximizes $C_{\bm{h}}$ in $\mathcal H(V,1)$ and $(h_{0,S}^{\textnormal{AWSM}}, h_{0,T}^{\textnormal{AWSM}})$ maximizes $C_{h_S,h_T}$ in $\mathcal H(T,1)$ and $\mathcal H(S,1)$. Though it does not necessarily optimize the upper bound, it is an adequate choice without using any information on $\mathbb P$. In the subsequent experiments, we consistently employ this optimal weight function. 

\subsection{Relationship with Change of Variable}
\label{sec:change_variable}
The major issue for the failure of the original score matching is that the domain of interest is bounded instead of $\mathbb R^d$. A reasonable solution is to use the variable transformation to transform the original domain to $\mathbb R^d$. Here, we state that \textbf{variable transformation is equivalent to weighted score matching.} This is a novel insight that gives a geometric explanation for weighted score matching. 

We state our result for a random variable on $\mathbb R$ for simplicity. 
Consider a random variable $X$ in a bounded open interval $I\subset \mathbb R$ and $g: I \rightarrow \mathbb R$ is an injection. Denote the derivative of $g^{-1}$ as $\phi:={g^{-1}}^{\prime}$. Let $Y = g(X)$ be the transformed random variable with support $g(I)$. We now consider the score function of $Y$. The p.d.f of $X$ and $Y$ are denoted as $p_{X}(\cdot)$ and $p_{Y}(\cdot)$ respectively. 
Such a relationship should hold by the change of variable formula: 
\begin{align*}
    p_Y(y) &= p_{X}(g^{-1}(y))\phi(y)=p_{X}(x(y))\phi( y). 
\end{align*}

The equivalence of variable transformation and weighted score matching is summarized in the following theorem. 
\begin{theorem}\label{thm:change_of_variable}
    Suppose $g(I) = \mathbb R$ and $\lim _{| y|\rightarrow \infty}\phi(y) = 0$, then we have, 
    \begin{equation*}
        \mathcal J^{\textnormal{SM}}_{Y}(\theta) = \mathcal J^{\textnormal{WSM}}_{X, h}(\theta), 
    \end{equation*}
    where the weight function is defined as $h(x) = [\phi(g(x))]^2$ and satisfies $h(x)=0, x\in\partial I$.
\end{theorem}

Although we have shown that suitable data transformations and adding weights can be theoretically equivalent, working directly with weights is often more transparent, especially when searching for an optimal choice. Moreover, some popular transformations are not equivalent to any admissible weight satisfying \cref{assum:wsm_regularity_of_weight} or \cref{assum:awsm_regularity_of_weight}.

\begin{remark}
Some works apply a log transformation to the inter-event times before performing score matching, by defining
\begin{equation*}
  x_n := \log\big(t_n - t_{n-1}\big)
\end{equation*}
and working with the conditional density $p(x_n \mid \mathcal H_{n-1})$. However, this is not identical to using a valid weight as in our framework. Since $t_n \in (t_{n-1}, T)$, we have
\begin{equation*}
  x_n \in \big(-\infty, \log(T - t_{n-1})\big),
\end{equation*}
so $p(x_n \mid \mathcal H_{n-1})$ is supported on a half-open interval. Because the log map is strictly increasing, $p(x_n \mid \mathcal H_{n-1})$ remains strictly positive at the right boundary of its support. Thus the usual requirement that the density vanishes at the boundary still fails, and this log transformation does not resolve the underlying issue.
\end{remark}



\section{A Complete Framework for Integration-Free Training of Deep Spatio-Temporal Point Processes}
\label{sec:integration-free-training}

In this section, we first point out a drawback that arises when applying the previously proposed method to nonparametric point process models. We then introduce a simple remedy, yielding a complete recipe for integration-free training of any intensity-based nonparametric spatio–temporal point process model. Finally, we briefly mention the situation of multivariate point process, which is common in modern point process benchmarks. 

\subsection{Score Matching Is Non-identifiable for Finite Point Processes}
\label{section:unidentifiability}

Matching the gradient of the log-density determines the shape of the density curve, but it does not, in general, fix its overall level (``altitude''). In standard settings, the normalization condition pins down this level. However, different from a density of a random variable, the Janossy density does not always normalize to 1 (or any fixed constant), i.e., 
\begin{equation*}
    \int_{V^{(N)}}j_N(\bm X_N) d\bm X_N =J_N(V^{(N)})\neq 1.
\end{equation*}
In fact, it normalize to the probability that exactly $N$ points are observed in $V$.

Now we study in specific this issue in spatio-temporal process, and relate it to the conditional density function. For a spatio–temporal point process, given that the process has realized $n\!-\!1$ points $\mathcal H_{n-1}$ (chronologically), the $n$-th point has conditional density $p_n(t_n,\bm s_n\mid \mathcal H_{n-1})$. Since the normalization for the Janossy measure is
\begin{equation*}
    \sum_{N=1}^\infty \int_{V^{(N)}}\frac{j_N(\bm X_N)}{N!}\,d\bm X_N \;=\; 1-J_0.
\end{equation*}
Translating this into conditions on the conditional densities gives
\begin{equation}\label{eq:normalizing_conditional_density}
    \begin{aligned}
          &J_0=1-\int_S\int_{[0,T]}p_1(t_1,\bm s_1)\,dt_1\,d\bm s_1,\\
    &J_{n-1}=\Bigg[\int_{D_{n-1}^T\times S^{(n-1)}}\prod_{i=1}^{n-1} p_i(t_i,\bm s_i\mid \mathcal H_{i-1}) \, d\mathcal H_{n-1}
    -\int_{D_n^T\times S^{(n)}}\prod_{i=1}^n p_i(t_i,\bm s_i\mid \mathcal H_{i-1}) \, d\mathcal H_n\Bigg]>0,\\
    &\forall n\ge 2.
    \end{aligned}
\end{equation}

Consequently, any sequence of positive constants $\{c_n\}_{n\ge 1}$ yields a family of conditional densities $c_n\,p_n(t_n\mid \mathcal H_{n-1})$ with the same score whenever there exists $n$ such that $c_n\neq 1$. In this case, even perfectly matching the score fails to recover $c_n$ and therefore does not uniquely determine the conditional density. Moreover, one can verify that there indeed exists a sequence $\{c_n\}$ that is not identically $1$ yet still satisfies the normalization condition in \cref{eq:normalizing_conditional_density}. Therefore, we must impose an additional assumption on our parametric class, as in \cref{thm:unique_wsm,thm:unique_awsm}, to ensure that the minimizer of the population score-matching loss recovers the ground-truth parameter. While this condition can hold for parametric models, it is far too restrictive for nonparametric intensity models (e.g., intensities parameterized by deep neural networks). We present a simple example that validates our argument. 
\begin{example}
    Consider a temporal point process with constant intensity $\lambda^*$, and suppose we estimate the intensity using AWSM. If the model intensity $\hat \lambda_{T,n}(t\mid \mathcal H_{n-1})$ exactly minimizes the explicit AWSM loss, i.e., $\mathcal L_{}^{\textnormal{AWSM, T}}(\hat \lambda_T) = 0$, then one can verify that it suffices for $\hat \lambda$ to have the form
    \begin{equation*}
        \hat \lambda_{T,n}(t\mid \mathcal H_{n-1}) 
        = \frac{\lambda^*}{1-\alpha \exp\left(-\lambda^*(t-t_{n-1})\right)}, 
        \quad \alpha\in \mathbb R.
    \end{equation*}
    In particular, any $\alpha\in \mathbb R$ yields zero explicit AWSM loss, while only $\alpha = 0$ recovers the true intensity. This shows that, for a general intensity model, score matching does not uniquely identify the underlying intensity.
\end{example}

\subsection{Remedy}
We propose a remedy for the spatio-temporal point process case, as our experiments based on deep neural networks primarily focus on this setting. 
We defer the remedy for a general finite point process through WSM in \cref{sec:WSM_remedy}.
We first restate the decomposition of the true spatio–temporal intensity as
\begin{equation*}
    \lambda_n(t,s\mid \mathcal H_{n-1})=\lambda_{T,n}(t\mid \mathcal H_{n-1})\, f_{S,n}(s\mid \mathcal H_{n-1},t).
\end{equation*}
We denote the spatial component by $f_S$ because it is a conditional density and normalizes to $1$ on $\mathcal S$ (a detailed derivation is in \cref{preliminary}). To resolve the issue above, we propose learning a parameterized conditional cumulative probability,
\begin{equation*}
    \hat F_{n}(\mathcal H_{n-1}):= \hat{\Pr}(N\geq n\mid \mathcal H_{n-1}):\; D_{n-1}^T\times S^{n-1}\rightarrow [0,1].
\end{equation*}
This function takes all information up to $n\!-\!1$ and predicts whether the sequence continues. It acts as a learned normalizing constant that adjusts the density to the correct level. This is also equivalent to learning a conditional survival probability (i.e., learning $G_n(T|\mathcal H_{n-1}))$, and we will refer this method as survival classification. 

We learn $\hat F$ via the cross-entropy loss
\begin{equation*}
    \mathcal J^{\text{Survival}}(\hat F) \;=\; -\mathbb E\!\left[\sum_{n=1}^{N}\log \hat F_{n}(\mathcal H_{n-1})+\log\!\big(1-\hat F_{N+1}(\mathcal H_{N})\big)\right],
\end{equation*}
which amounts to classifying $Y=\bm 1\{N(\mathcal X)\ge n\}$ from $X=\mathcal H_{n-1}$. We use two models $\tilde \lambda_{T,n}(t\mid \mathcal H_{n-1})$ and $\tilde f_{S,n}(\bm s\mid \mathcal H_{n-1})$ to compute the scores as in \cref{remark:link_intensity_to_score} and the AWSM loss, and then learn these functions. After estimating the conditional cumulative probability, construct the estimated temporal intensity as,
\begin{align*}
    \hat \lambda_{T,n}(t_n\mid \mathcal H_{n-1})&:= \frac{\tilde  G_n(t_n\mid \mathcal H_{n-1})\,\tilde \lambda_{T,n}(t_n\mid \mathcal H_{n-1})}{\tilde G(t_n\mid \mathcal H_{n-1})+\dfrac{1-\tilde G_n(T\mid \mathcal H_{n-1})}{\hat F_n(\mathcal H_{n-1})}-1},\\
    \tilde G_n(t\mid \mathcal H_{n-1})&:=\exp\!\big(-\tilde \Lambda_n(t\mid \mathcal H_{n-1})\big),\\ 
    \tilde \Lambda_n(t\mid \mathcal H_{n-1})&:=\int_{t_{n-1}}^t \tilde \lambda_{T,n}(\tau\mid \mathcal H_{n-1})\,d\tau,
\end{align*}
then we have the following proposition.
\begin{proposition}\label{prop:effectiveness_of_loss}
    Suppose $(\tilde \lambda_T,\hat F)$ satisfies
\begin{equation*}
  (\tilde \lambda_T,\hat F)= 
  \argmin \mathcal L(\tilde \lambda_T,\hat F),
  \quad
  \mathcal L(\tilde \lambda_T, \hat F) 
  := \mathcal L^{\textnormal{AWSM,T}}(\tilde \lambda_T)
   + \mathcal J^{\textnormal{Survival}}(\hat F).
\end{equation*}
    Then $\hat \lambda_T = \lambda_T^*$ w.p.~1.
\end{proposition}

Thus, learning a survival probability enables recovery of the temporal intensity. Let the model’s conditional spatial intensity be $\tilde f_S$; the normalized spatial density is $\hat f:=\tilde f_{S}\big/\!\int \tilde f_{S}$, then we also attain the consistent spatial density, thus recover the whole intensity. Notice that the training objective above involves no integration. Consequently, we obtain a fully integration-free objective that recovers the intensity. The conditional cumulative probability $\hat F_n(\mathcal H_{n-1})$ can be implemented as a simple binary classifier that takes the history as input. 

During inference, we sometimes need to compute the log-likelihood, for example when evaluating the log-likelihood on the test data in experiments. For the normalized model, we have 
\begin{equation}\label{eq:true_ll}
    \begin{aligned}
            ll_{T} &= \sum_{n=1}^N \left[\underbrace{\log \tilde \lambda_n - \tilde \Lambda_n-\log\!\big(1-\tilde G_n(T)\big)+\log \hat F_n}_{\text{likelihood of }p_n(t_n\mid \mathcal H_{n-1})}\right]
    +\underbrace{\log\!\big(1-\hat F_{N+1}\big)}_{\text{likelihood of } p(N(\mathcal X)=N\mid \mathcal H_N)},\\
ll_S &= \sum_{n=1}^N \big[\log \tilde f_{S,n}-\log \textstyle\int \tilde f_{S,n}\big],\qquad ll \;=\; ll_T+ll_S.
    \end{aligned}
\end{equation}
A detailed derivation of \cref{eq:true_ll} can be found in \cref{sec:derivation_of_ll}. To evaluate the log-likelihood, one needs the integrals of the surrogate temporal intensity $\tilde \lambda_n$ and the unnormalized spatial density. So the integration arises only in inference stage.

\subsection{Multivariate Spatio-temporal Point Processes}

In this section, we focus on how to add an extra loss term to estimate multivariate spatio-temporal point processes. First, multivariate spatio-temporal point processes mean that the point also carries a type information, besides the timestamp and location. For a regular multivariate point processes on $(0,T)\times \mathcal S \times [K]$, the class of conditional density function becomes $p_n(t_n, \bm s_n, k_n|t_1,\bm s_1, k_1,\ldots, t_{n-1}, \bm s_{n-1}, k_{n-1})$. We further decompose it by: 
\begin{equation*}
    p_n(t_n,\bm s_n, k_n|\mathcal H_{n-1})=p_{T,n}(t_n|\mathcal H_{n-1})f_{S,n}(\bm s_n|\mathcal H_{t_n})f_{K,n}(k_n|\mathcal H_{\bm s_{n}}), 
\end{equation*}
where $f_{K,n}(k_n|\mathcal H_{\bm s_n})$ is the conditional mass of $k_n$ given $\mathcal H_{\bm s_n}$ and
\begin{equation*}
    \mathcal H_{\bm s_{n}}=(t_1,\bm s_1,k_1,\ldots, t_{n-1},\bm s_{n-1}, k_{n-1},t_n,\bm s_n). 
\end{equation*}
The estimation is performed using the original AWSM loss and an additional cross-entropy (CE) loss: 
\begin{equation}\label{eq:full_AWSM}
    \begin{aligned}
        \mathcal J^{\textnormal{CE}}(\hat f_K)& = -\mathbb E [\sum_{n=1}^{N(\mathcal X)}\log \hat f_{K,n}(k_n|\mathcal H_{\bm s_n})],\\
    \mathcal J(\tilde \lambda_{T}, \tilde f_{S},  \tilde f_K, \hat F) &=\mathcal J^{\textnormal{AWSM}}_{h_T,h_S}(\tilde \lambda_T, \tilde f_S)+\alpha_{\textnormal{surv}}\mathcal J^{\text{Survival}}(\hat F) +\alpha_K\mathcal J^{\textnormal{CE}}(\tilde f_K), 
\end{aligned}
\end{equation}
with balancing coefficients $\alpha_{\textnormal{surv}}$ and $\alpha_K$.

\section{Experiments} 
\label{experiment} 
This section validates our proposed (A)WSM on statistical and deep point process models.
For statistical models, we focus on verifying whether (A)WSM can accurately recover the ground-truth parameters. For deep point-process models, we show that our new training procedure recovers the intensity with MLE-level accuracy and is both efficient and scalable. 

\subsection{Statistical Point Processes}
\label{synthetic}
In this section, we conduct experiments with parametric statistical models from the same family as the data-generating process and examine the estimation bias and variance of the (A)WSM objective.
\paragraph{Baselines and Metrics}
We consider two baselines (1) \textbf{MLE}, (2) implicit \textbf{(A)SM}~\citep{SahaniBM16, zhang2023integration, li2023smurf}, i.e., score matching without adding weight functions. To compare the performance of different methods,  we use the mean absolute error (\textbf{MAE}, $|\hat{\theta}-\theta|$) between the ground-truth parameters and the estimates as a metric since the ground-truth parameters are known.

\paragraph{Datasets} We validate the effectiveness of (A)WSM using three sets of synthetic data. 
(1) \textbf{Spatial Poisson Process}: This dataset is simulated from a spatial inhomogeneous Poisson process with an intensity function $\lambda(\bm x)=\exp(\theta (\sin x_1 + \cos x_2))$ on $(-2\pi, 2\pi)^2$, $\theta=2$. 
(2) \textbf{2-Variate Temporal Hawkes Processes}: This dataset is simulated from a $2$-variate Hawkes processes with conditional intensity: $\lambda_n(t, k|\mathcal H_{n-1}) = \mu_{k}+\sum_{i<n}\alpha_{k_ik}\exp(-5(t-t_i))$, with $T=10$, $\mu_1=\mu_2=1$, $\alpha_{11}=1.6, \alpha_{12}=0.2$, $\alpha_{21}=\alpha_{22}=1$. 
(3) \textbf{Univariate Spatio-temporal Gaussian Hawkes Process}: This dataset is simulated from a univariate Hawkes process with Gaussian decay triggering kernel $\lambda_n(t,\bm s|\mathcal H_{n-1})=\mu+\sum_{i<n}\exp[-\beta(t-t_i)]\frac{C}{2\pi(t-t_i)}\exp(-\frac{1}{2}\|\bm s-\bm s_i\|^2)$. The parameters are set to be $\mu=0.5, C=1,\beta=2$. 

\paragraph{Training Protocol}
We assume that we know the ground-truth model but do not know its parameters. Therefore, we use the ground-truth model as the training model. 
The purpose is to verify whether the estimator can recover the ground-truth parameters. 
For each dataset, we collect a total of $1000$ sequences. 
We 
run $500$ iterations of gradient descent using Adam~\citep{kingma2014adam} as the optimizer for all scenarios. 
For MLE, the intensity integral is computed through numerical integration, with the number of integration nodes set to $100$ to achieve a considerable level of accuracy.  For the 2-variate temporal Hawkes experiment, we use the CE loss for type matching.
We change the random seed $3$ times to compute the mean and standard deviation of MAE.

\paragraph{Results}
In \cref{table: synthetic Data Experiment}, we report the MAE of parameter estimates for three models trained by MLE, (A)SM, and (A)WSM on the synthetic datasets.  
We can see that both MLE and (A)WSM achieve small MAE on three types of data. However, the MAE of (A)SM is large. As we have theoretically demonstrated earlier, this is because MLE and (A)WSM estimators are consistent. In contrast, (A)SM, due to the absence of the required regularity conditions in the three cases, has an incomplete objective and cannot accurately estimate parameters. 

\begin{table*}[t]
\centering
\caption{The MAE of three models trained by MLE, (A)SM, and (A)WSM on the synthetic dataset. 
For the 2-variate processes, we only present the estimation results for some parameters here. The rest is shown in \cref{table:additional_experiments}.}

\label{table: synthetic Data Experiment}
\begin{sc}
\resizebox{\textwidth}{!}{
\begin{tabular}{c|c|ccc|ccc}
    \toprule
    \multirow{2}{*}{Estimator} & \multicolumn{1}{c|}{Poisson} & \multicolumn{3}{c|}{2-Variate Temporal Hawkes}  & \multicolumn{3}{c}{Spatio-temporal Gaussian Hawkes} \\
    \cmidrule{2-8}
        & $\theta$ & $\alpha_{11}$ & $\alpha_{12}$ &  $\mu_{1}$ & $\mu$ & $\beta$ & $C$ \\
    
    \midrule
    (A)WSM & $0.07_{\pm 0.01}$ & $0.041_{\pm 0.041}$& {$0.026_{\pm 0.001}$ }&$\bm{0.011_{\pm 0.010}}$ &$0.153_{\pm 0.162}$  & {$0.022_{\pm 0.023}$} &$0.060_{\pm 0.066}$ 
    \\
     \midrule
    (A)SM & $0.27_{\pm 0.01}$ & $1.600_{\pm 0.001}$& $0.200_{\pm 14.30}$ & $0.700_{\pm 0.272}$ &$0.487_{\pm 0.313}$  & $0.988_{\pm 0.083}$ & $1.420_{\pm 0.313}$ 
    \\
    \midrule
    MLE  & \bm{$-0.02_{\pm 0.10}$} & \bm{$0.028_{\pm 0.015}$} & $\bm{0.014_{\pm 0.002}}$ & {$0.012_{\pm 0.006}$} &$\bm{0.098_{\pm 0.107}}$ & $\bm{0.017_{\pm 0.019}}$ & \bm{$0.051{\pm 0.049}$}\\
    \bottomrule
\end{tabular}}
\end{sc}
\end{table*}

\subsection{Deep Point Processes}
In this section, we parameterize the intensity using deep neural network and evaluate whether our AWSM method can accurately and efficiently recover the intensity, on both synthetic datasets where the true intensity is known and real-world datasets where the true intensity is unknown. 

\paragraph{Baselines and Metrics}
We consider two baselines: (1) \textbf{MLE}, (2) Denoising Score Matching (\textbf{DSM})~\citep{li2023smurf,li2024beyond}. The result of (A)SM is omitted since it fails on our tasks. We use test log-likelihood (\textbf{TLL}) as a principal metric. Notably, on spatio-temporal datasets, we use $\textnormal{TLL}_S$ and $\textnormal{TLL}_T$ for spatial test log-likelihood and temporal test log-likelihood. We also use the event type prediction accuracy (\textbf{ACC}) as metrics for multivariate point processes. Besides, we also record the total runing time (\textbf{RT}) in minutes of each method, deployed on the same model, same device with same data-loading strategy.

\paragraph{Datasets}
We perform experiments on both synthetic and real-world datasets, with a total number of 12 datasets. For synthetic datasets, we visualize the learned intensities by different training objectives and the ground-truth intensity. For real-world datasets, we record their TLL, ACC and RT. To be specific, we consider: 
\begin{itemize}
    \item (Multivariate) Temporal Point Processes: (1) \textbf{2-Variate Half-sin Hawkes Processes}: this is a synthetic 2-variate temporal Hawkes processes with triggering kernel $g_{ij}(\tau)=\alpha_{ij}sin(\tau), \tau\in (0,\pi)$.
    (2) \textbf{2-Variate Exponential-decay Hawkes Processes}: this is a synthetic 2-variate Hawkes processes with triggering kernel $g_{ij}(\tau)=\alpha_{ij}\exp(-\beta \tau)$.
    (3) \textbf{StackOverflow}~\citep{jure2014snap}: this dataset has two years of user awards on StackOverflow. Each user received a sequence of badges and there are $K = 22$ types of badges. 
    (4) \textbf{Taobao}~\citep{xue2022hypro}: this dataset comprises user activities on Taobao (in total $K=17$ event types). 
    (5) \textbf{Retweet}~\citep{zhao2015seismic}: this dataset includes sequences indicating how each novel tweets are forwarded by other users. Retweeter categories serve as event types $K=3$. 
    (6) \textbf{Taxi}~\citep{whong2014foiling}: this dataset comprises taxi pick-up and drop-off incidents in New York City. There are in total $K = 10$ event types. 
    For each dataset, we follow the default training/dev/testing split in the repository. 

    \item (Multivariate) Spatio-temporal Point Processes: (1) \textbf{Gaussian Spatio-temporal Hawkes Process}~\citep{zhou2022neural}: this is a univariate spatio-temporal process with $\lambda_n(\bm s, t|\mathcal H_{n-1})=\mu g_0(\bm s)+\sum_{i<n}g_1(t,t_i)g_2(\bm s, s_i)$. (2) \textbf{Gaussian Self-correcting Process}~\citep{zhou2022neural}: this is a univariate spatio-temporal process with $\lambda_n(\bm s, t|\mathcal H_{n-1})=\mu \exp(g_0(\bm s)\beta t-\sum_{i<n}\alpha g_2(\bm s, \bm s_i))$. (3) \textbf{Earthquake}~\citep{usgs_earthquake_catalogue_2020}: this dataset contains the location and time of all earthquakes in Japan from 1990 to 2020. 
    (4) \textbf{COVID19}~\citep{nytimes_covid19_data_2020}: this dataset is released publicly by The New York Times (2020) on daily COVID-19 cases in New Jersey state. The number of events per sequence ranges from 3 to 323, consisting of time and locations. 
    (5) \textbf{Citibike}~\citep{citibike_system_data_2019}: this dataset consists of rental events from a bike sharing service in New York City, consisting of time and locations. 
    (6) \textbf{Football}~\citep{yeung2023transformer}: this dataset records football event data retrieved from the WyScout Open Access Dataset. Each event signifies an action made by the player, with time, location and event type of the action. 
    For each dataset, we follow the default training/dev/testing split in the repository. 
\end{itemize}

\paragraph{Models}
In principle, our method works for any intensity-based model that allows taking derivative to its input. In practice, for temporal datasets, we focus on two of the most popular attention-based Hawkes process models: 
\textbf{SAHP}~\citep{zhang2020self} and \textbf{THP}~\citep{simiao2020transformer}. For spatio-temporal datasets, we mainly adopt the codebase of \citet{li2024beyond} and use their THP-based \textbf{SMASH} model. For the additional survival part $\hat F_n(\mathcal H_{n-1})$, we use a MLP-classifier that takes as input the history embedding from the Transformer. The detailed implementation is deferred to \cref{sec:network_details}. 


\paragraph{Preprocessing and Training Protocol}
For the temporal coordinate, if a dataset is already approximately confined to a fixed observation window $(0,T)$, we keep the original timestamps and set $T$ to be the ceiling of the maximum timestamp observed over the entire dataset. Then we insert a ``zero element'' with $t_0=0$. If the events are recorded over disparate calendar periods, we instead re-center each sequence by subtracting the time of its first event so that all sequences start at time $0$. Then we use ceiling of the maximum re-centered timestamps as $T$. 
For the two-dimensional spatial coordinates, we take an axis-aligned rectangle as the spatial observation domain $\mathcal S$ whose $x$- and $y$- limits  are given by the minimum and maximum coordinates observed in the dataset. 
On top of this preprocessing, for a subset of datasets we further apply min–max normalization to the temporal coordinate, the spatial coordinates, or both.
Again during training, we always initialize the history of the first event $t_1$ with a dummy event at $t_0=0, \bm s_0=(0,0)$ and $k_0=0$, which serves as a common initial history for all sequences.

For each dataset, we train 3 seeds with the same epochs and report the mean and standard deviation of the best TLL, ACC and RT. When using MLE, we adopt numerical integration to calculate the intensity integral. We select the number of integration nodes such that all three methods have similar training speed and memory consumption. 
When training with DSM, we add additive Gaussian noise to the locations and multiplicative log-normal noise to the timestamps. More details of training and testing hyperparameters are provided in \cref{sec:additional experiment}. 

\paragraph{Results}
We summarize our empirical findings on deep point process models as follows.
\begin{itemize}
    \item \textit{AWSM recovers the ground-truth intensity with deep point process models.} We visualize the learned intensity alongside the ground-truth intensity on synthetic datasets. The result on the temporal synthetic dataset is shown in \cref{fig:synthetic_temporal_intensity_visualize}, and the result on the spatio-temporal synthetic dataset is shown in \cref{fig:synthetic_spatial_temporal_intensity_visualize}. They all approximately match the ground-truth intensity. 

    \item \textit{AWSM achieves TLL and ACC comparable to MLE.} In \cref{table:results_on_temporal_dataset_split}, we report the performance of SAHP and THP trained on four real-world temporal datasets under three training objectives. In \cref{table:results_on_ST_dataset}, we report the performance of SMASH on four real-world spatio-temporal datasets under the same set of objectives. Across all datasets, models trained with MLE and AWSM exhibit very similar performance in terms of both TLL and ACC on the test data, indicating that AWSM is consistent with MLE in the sense that they yield comparable model parameters. Overall, DSM is slightly inferior to AWSM and MLE. We also observe that the performance of DSM is sensitive to the noise scale and requires substantial hyperparameter tuning. Besides, we visualize some intensities on real-world datasets in \cref{fig:earthquake_map_visual} and \cref{learned_intensity}.


\item \textit{Under comparable training budgets, AWSM attains MLE-level accuracy with similar or lower running time.}
As reported in \cref{table:results_on_temporal_dataset_split} and \cref{table:results_on_ST_dataset}, the training times of the three objectives are deliberately kept in the same range: for MLE we do not use an overly coarse quadrature that would make it artificially fast but inaccurate. Instead, all methods are run under broadly comparable computational budgets. In this fair setting, the score-based objectives typically achieve accuracy on par with MLE while being slightly more efficient, mainly because they bypass the explicit computation of the normalizing integral. We further analyze this accuracy–efficiency trade-off in \cref{sec:advantage_of_AWSM}.

\end{itemize}

\begin{figure}[htbp]
    \centering

    \begin{subfigure}[b]{0.33\textwidth}
        \includegraphics[width=\linewidth]{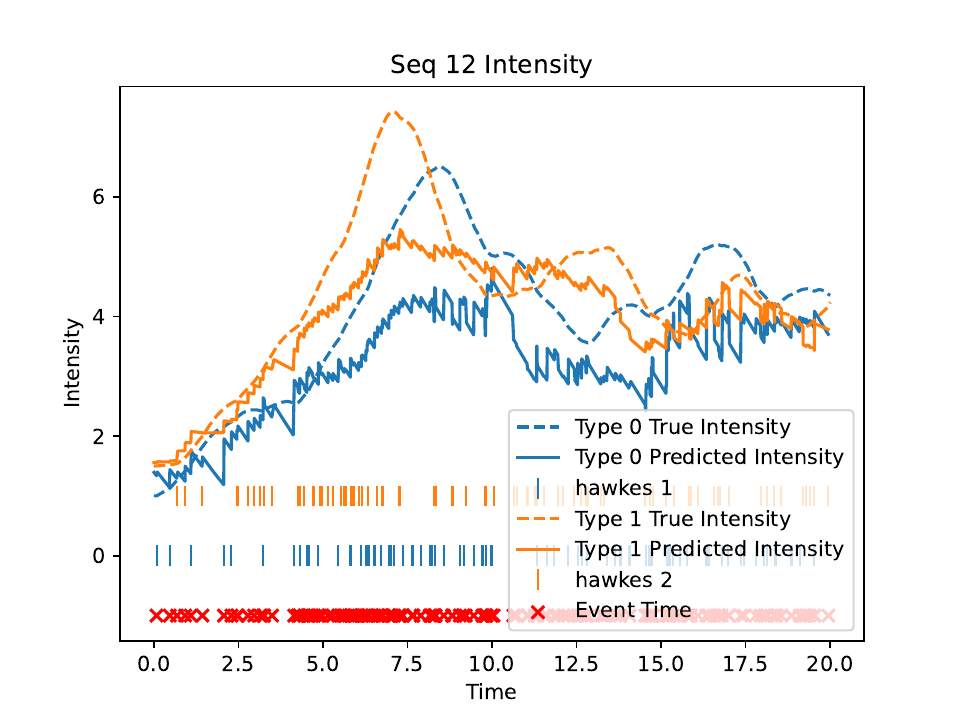}
        \subcaption{MLE}
        \label{fig:intensity_seq12_mle_thp}
    \end{subfigure}%
    \hfill
    \begin{subfigure}[b]{0.33\textwidth}
        \includegraphics[width=\linewidth]{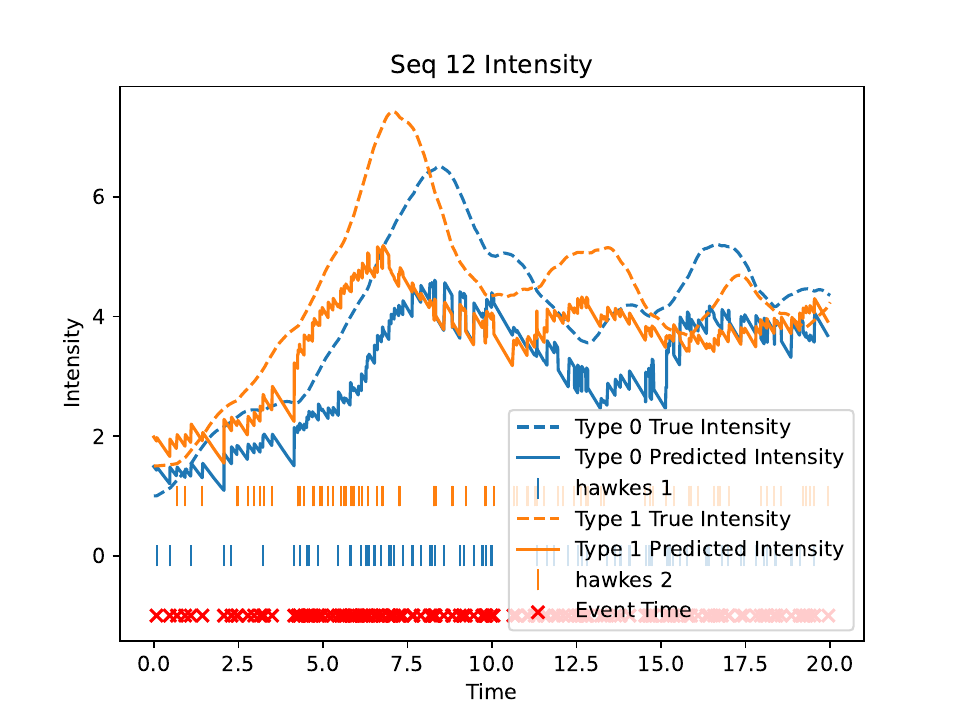}
        \subcaption{DSM}
        \label{fig:intensity_seq12_dsm_thp}
    \end{subfigure}%
    \hfill
    \begin{subfigure}[b]{0.33\textwidth}
        \includegraphics[width=\linewidth]{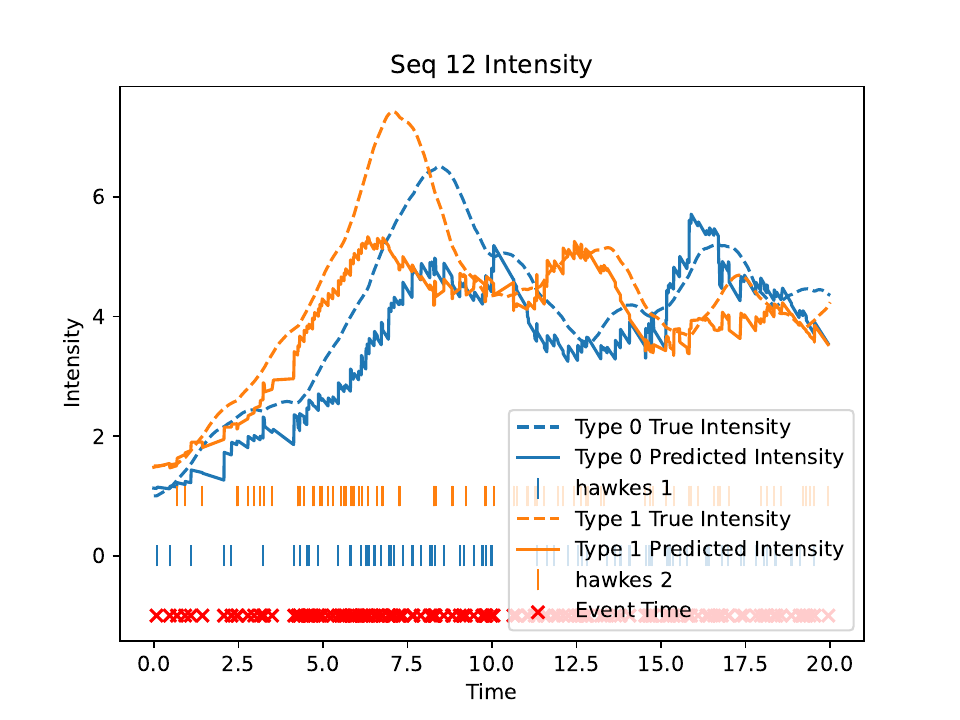}
        \subcaption{AWSM}
        \label{fig:intensity_seq12_wsm_thp}
    \end{subfigure}

    \vspace{0.15cm}

    \begin{subfigure}[b]{0.33\textwidth}
        \includegraphics[width=\linewidth]{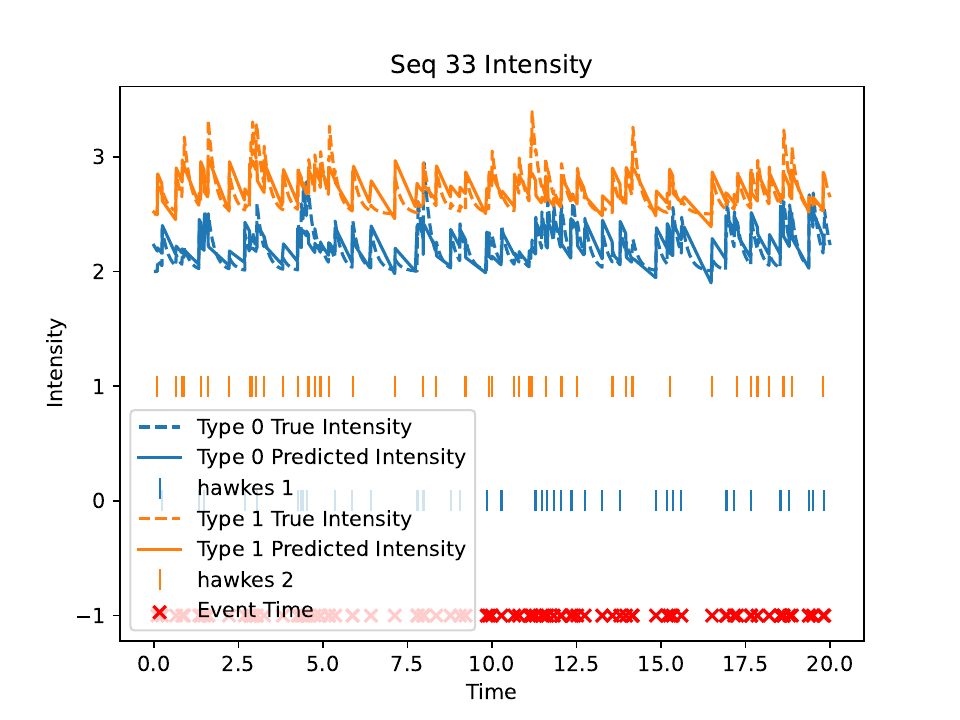}
        \subcaption{MLE}
        \label{fig:intensity_seq33_mle_thp}
    \end{subfigure}%
    \hfill
    \begin{subfigure}[b]{0.33\textwidth}
        \includegraphics[width=\linewidth]{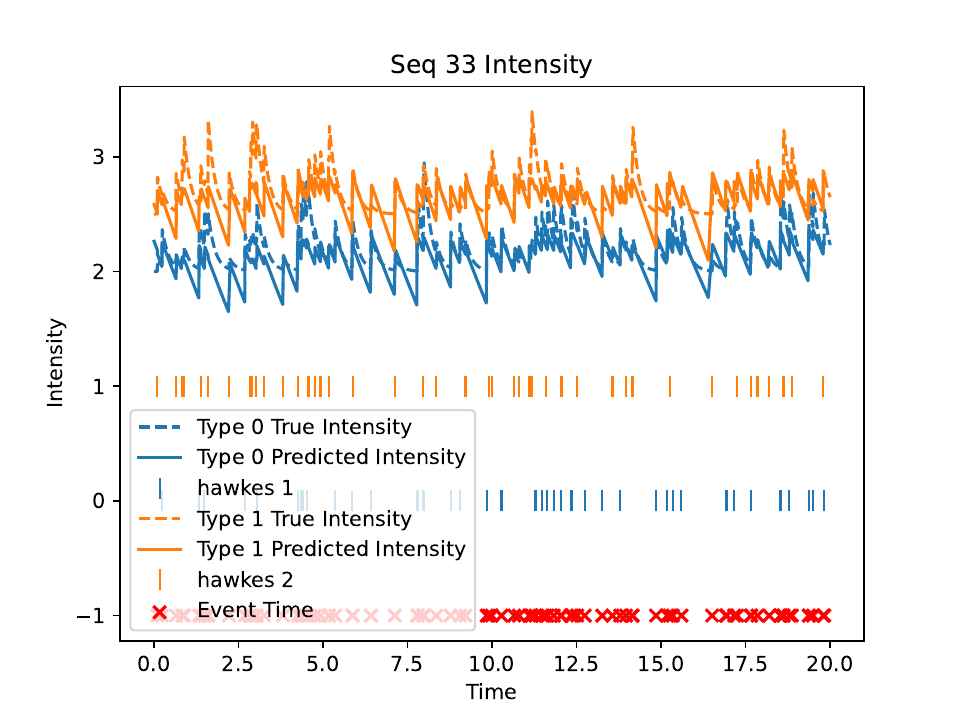}
        \subcaption{DSM}
        \label{fig:intensity_seq33_dsm_thp}
    \end{subfigure}%
    \hfill
    \begin{subfigure}[b]{0.33\textwidth}
        \includegraphics[width=\linewidth]{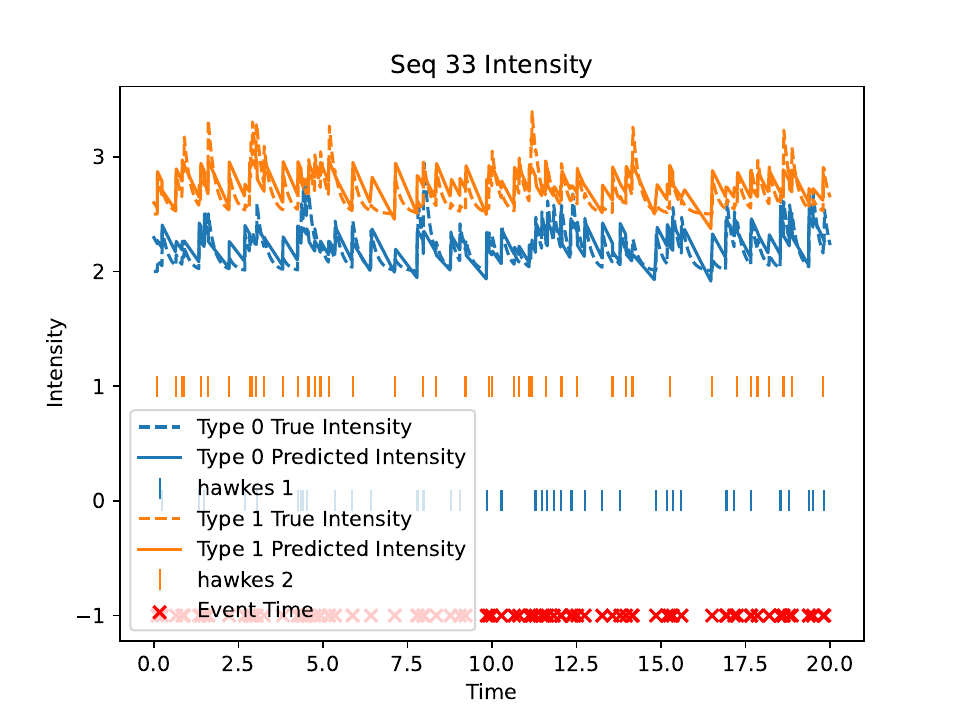}
        \subcaption{AWSM}
        \label{fig:intensity_seq33_wsm_thp}
    \end{subfigure}

    \vspace{0.15cm}

    \begin{subfigure}[b]{0.33\textwidth}
        \includegraphics[width=\linewidth]{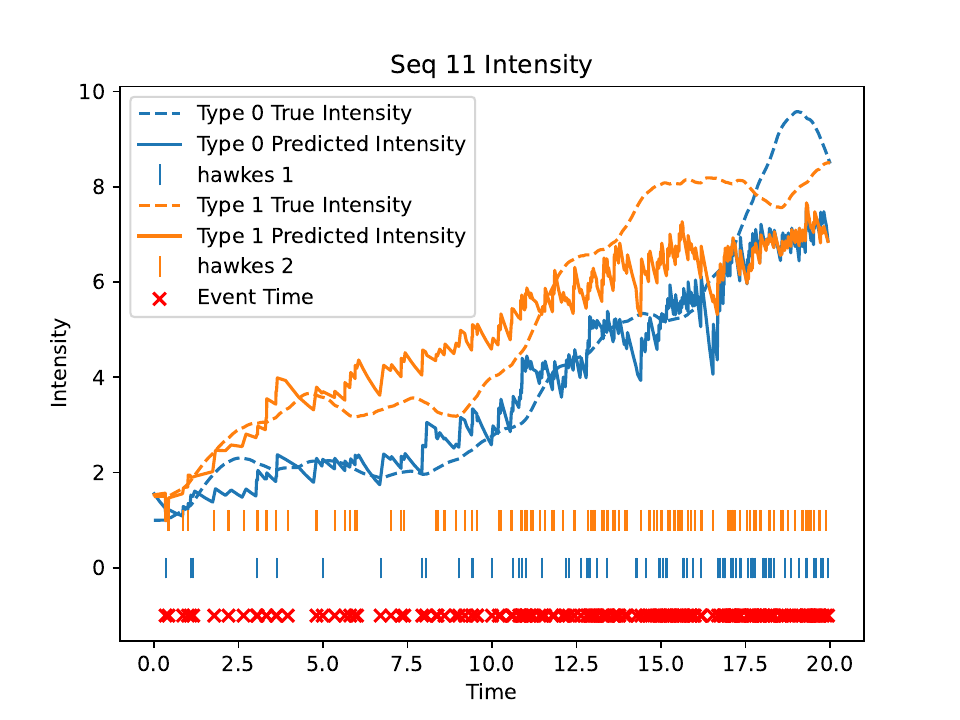}
        \subcaption{MLE}
        \label{fig:intensity_seq11_mle_sahp}
    \end{subfigure}%
    \hfill
    \begin{subfigure}[b]{0.33\textwidth}
        \includegraphics[width=\linewidth]{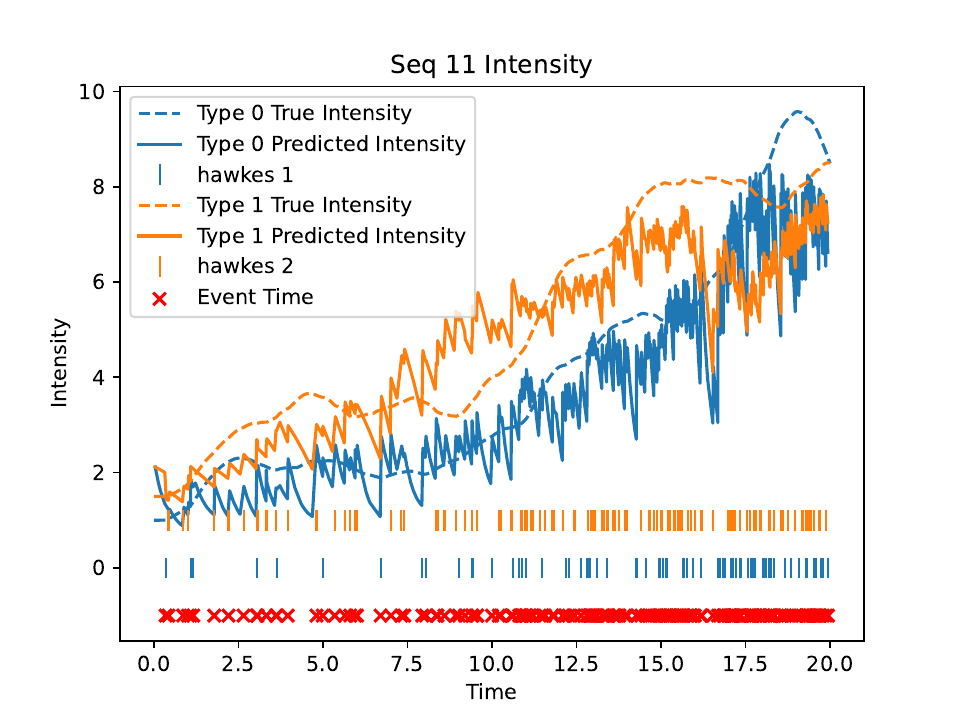}
        \subcaption{DSM}
        \label{fig:intensity_seq11_dsm_sahp}
    \end{subfigure}%
    \hfill
    \begin{subfigure}[b]{0.33\textwidth}
        \includegraphics[width=\linewidth]{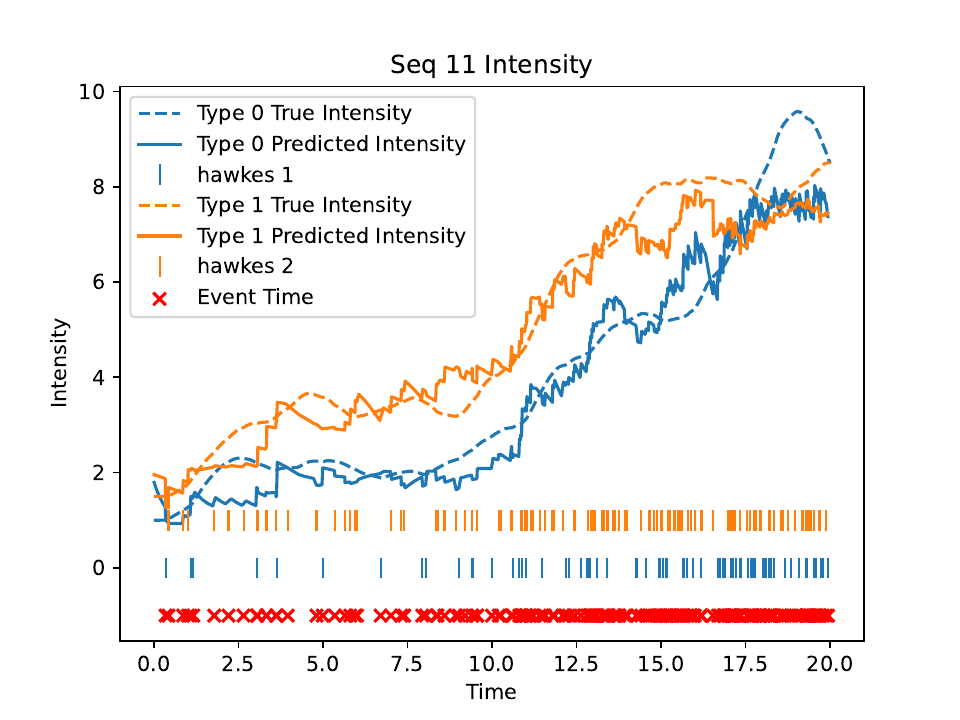}
        \subcaption{AWSM}
        \label{fig:intensity_seq11_wsm_sahp}
    \end{subfigure}

    \vspace{0.15cm}

    \begin{subfigure}[b]{0.33\textwidth}
        \includegraphics[width=\linewidth]{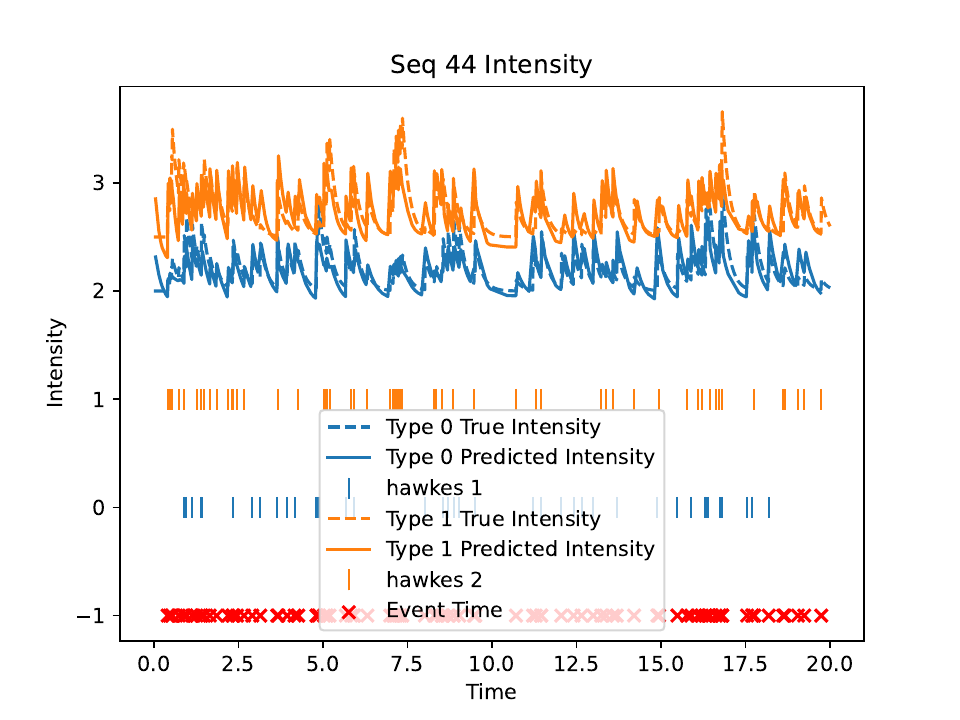}
        \subcaption{MLE}
        \label{fig:intensity_seq44_mle_sahp}
    \end{subfigure}%
    \hfill
    \begin{subfigure}[b]{0.33\textwidth}
        \includegraphics[width=\linewidth]{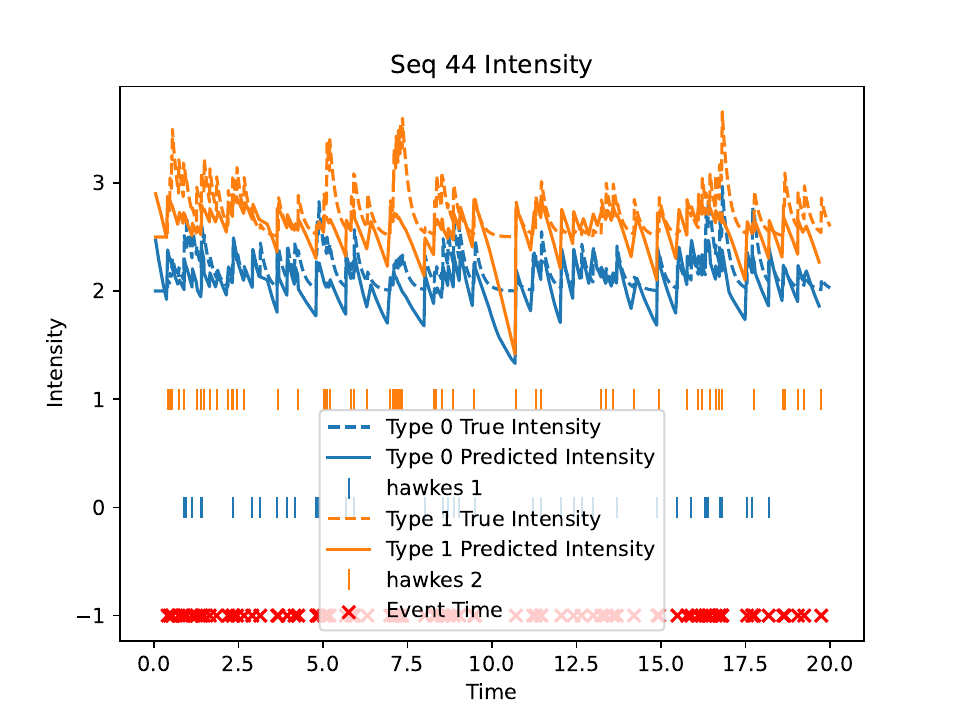}
        \subcaption{DSM}
        \label{fig:intensity_seq44_dsm_sahp}
    \end{subfigure}%
    \hfill
    \begin{subfigure}[b]{0.33\textwidth}
        \includegraphics[width=\linewidth]{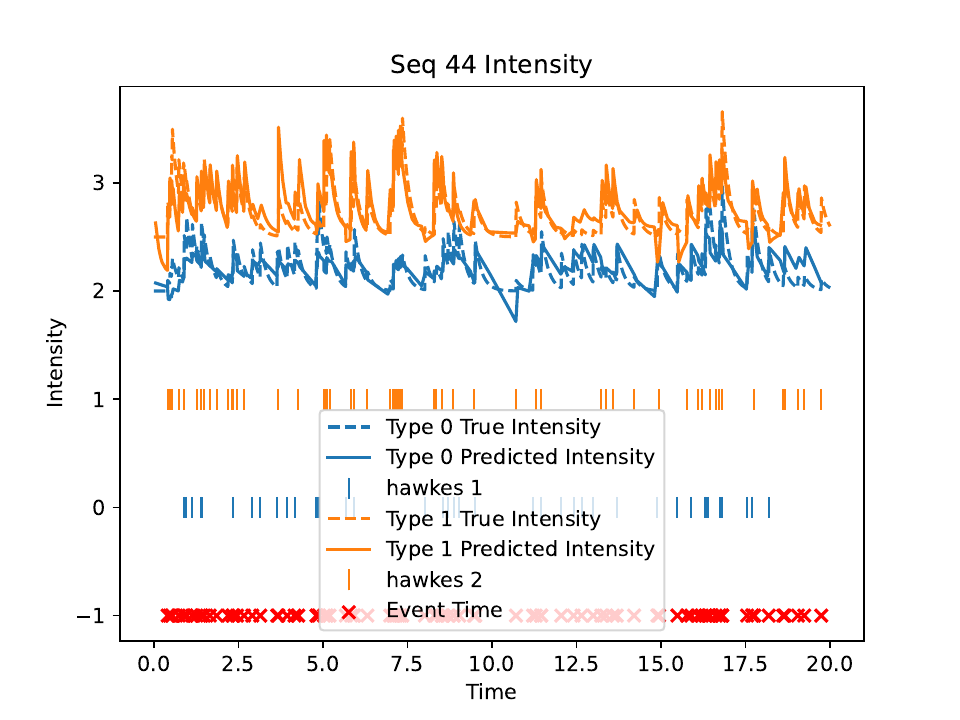}
        \subcaption{AWSM}
        \label{fig:intensity_seq44_wsm_sahp}
    \end{subfigure}

    \caption{Ground-truth and learned intensities on two temporal synthetic datasets. \textbf{Rows}: (1) Half-Sin, THP; (2) Exp-decay, THP; (3) Half-Sin, SAHP; (4) Exp-decay, SAHP. \textbf{Columns}: MLE, DSM, AWSM.}
\label{fig:synthetic_temporal_intensity_visualize}
\end{figure}


\begin{figure}[htbp]
    \centering
    \begin{subfigure}{0.48\textwidth}
        \centering
        \includegraphics[width=\linewidth]{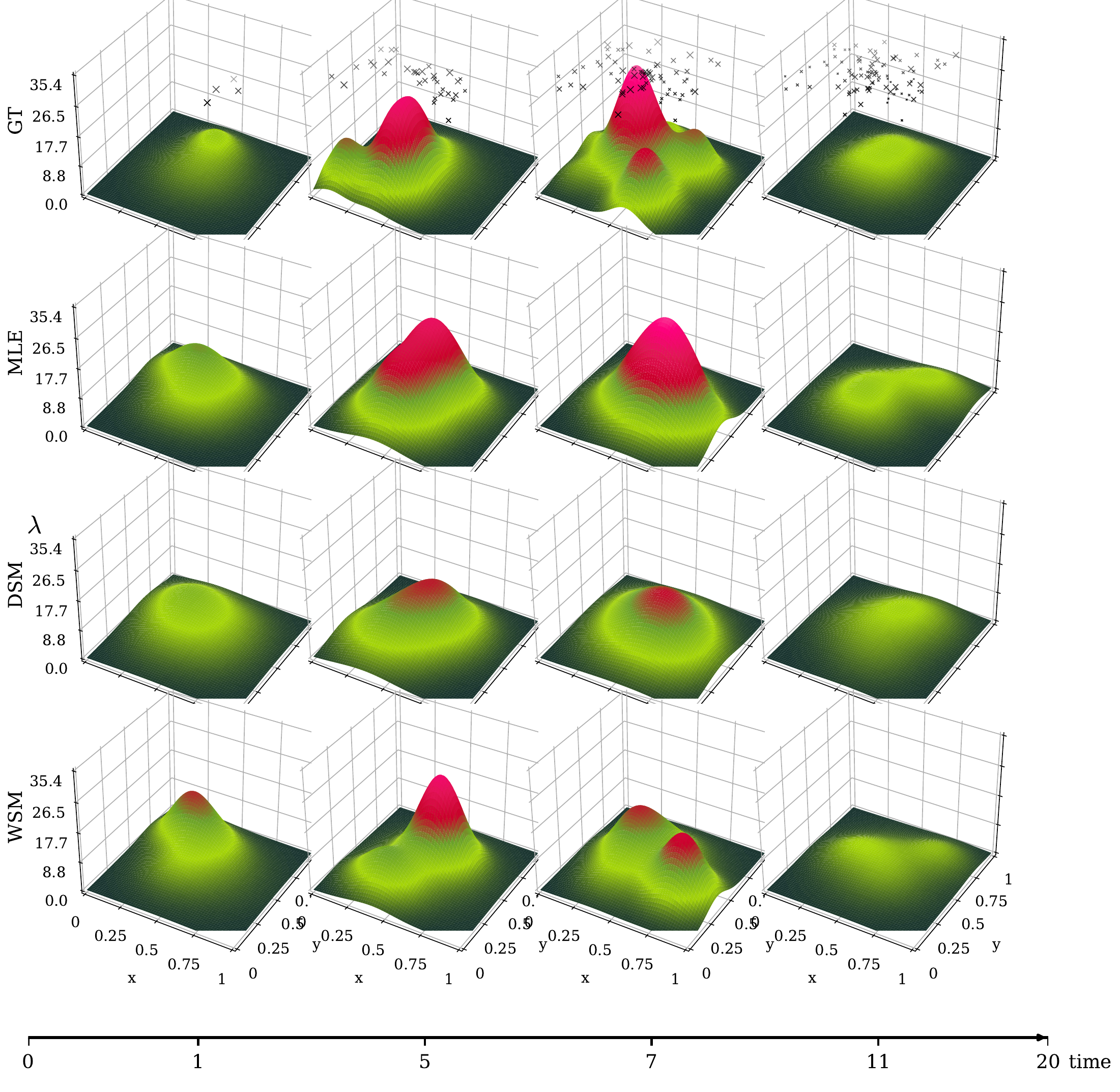} 
        \caption{Spatio-temporal Hawkes 1}
        \label{fig:sub1}
    \end{subfigure}
    \hfill
    \begin{subfigure}{0.48\textwidth}
        \centering
        \includegraphics[width=\linewidth]{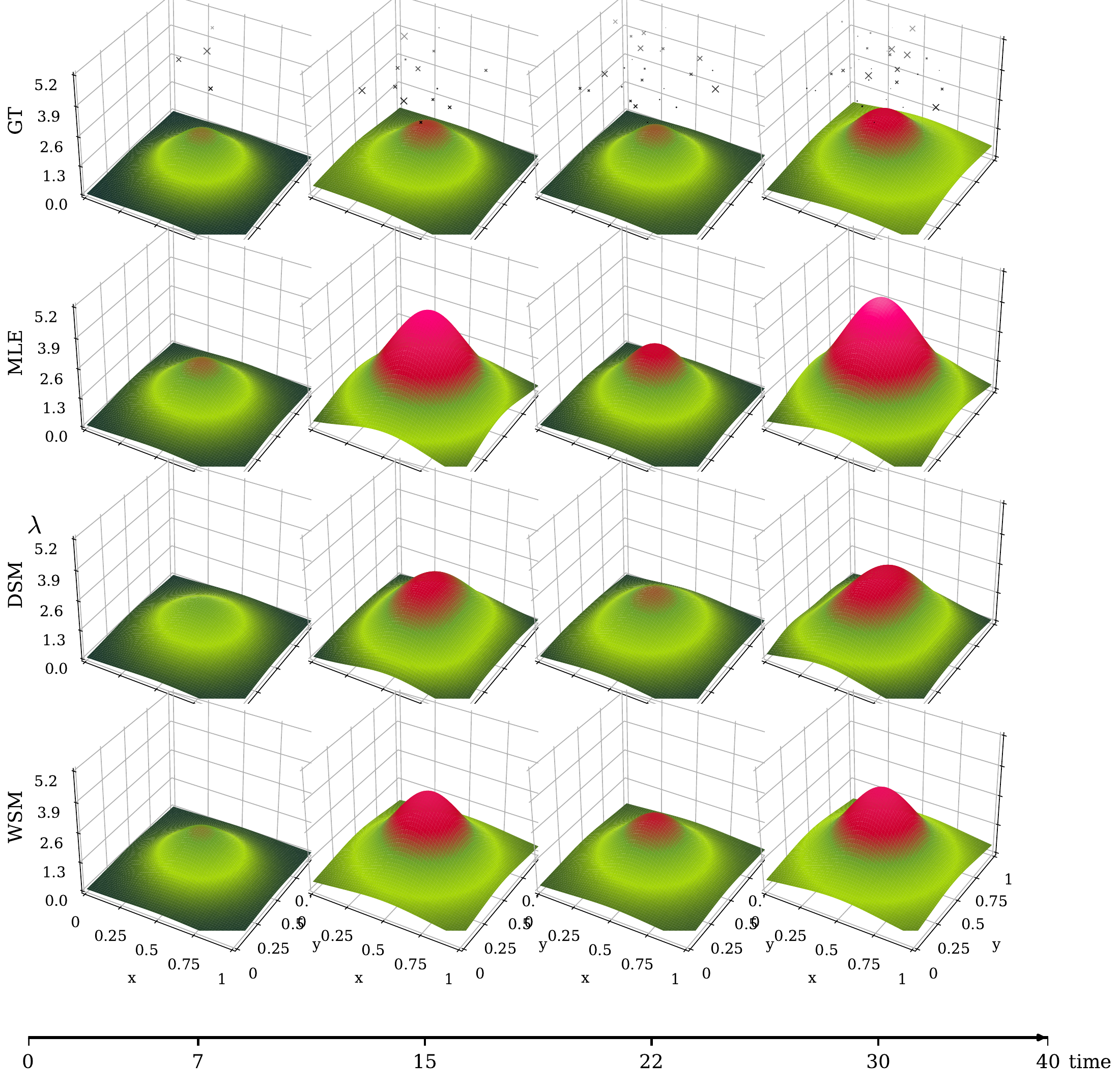} 
        \caption{Spatio-temporal Hawkes 2}
        \label{fig:sub2}
    \end{subfigure}


    \begin{subfigure}{0.48\textwidth}
        \centering
        \includegraphics[width=\linewidth]{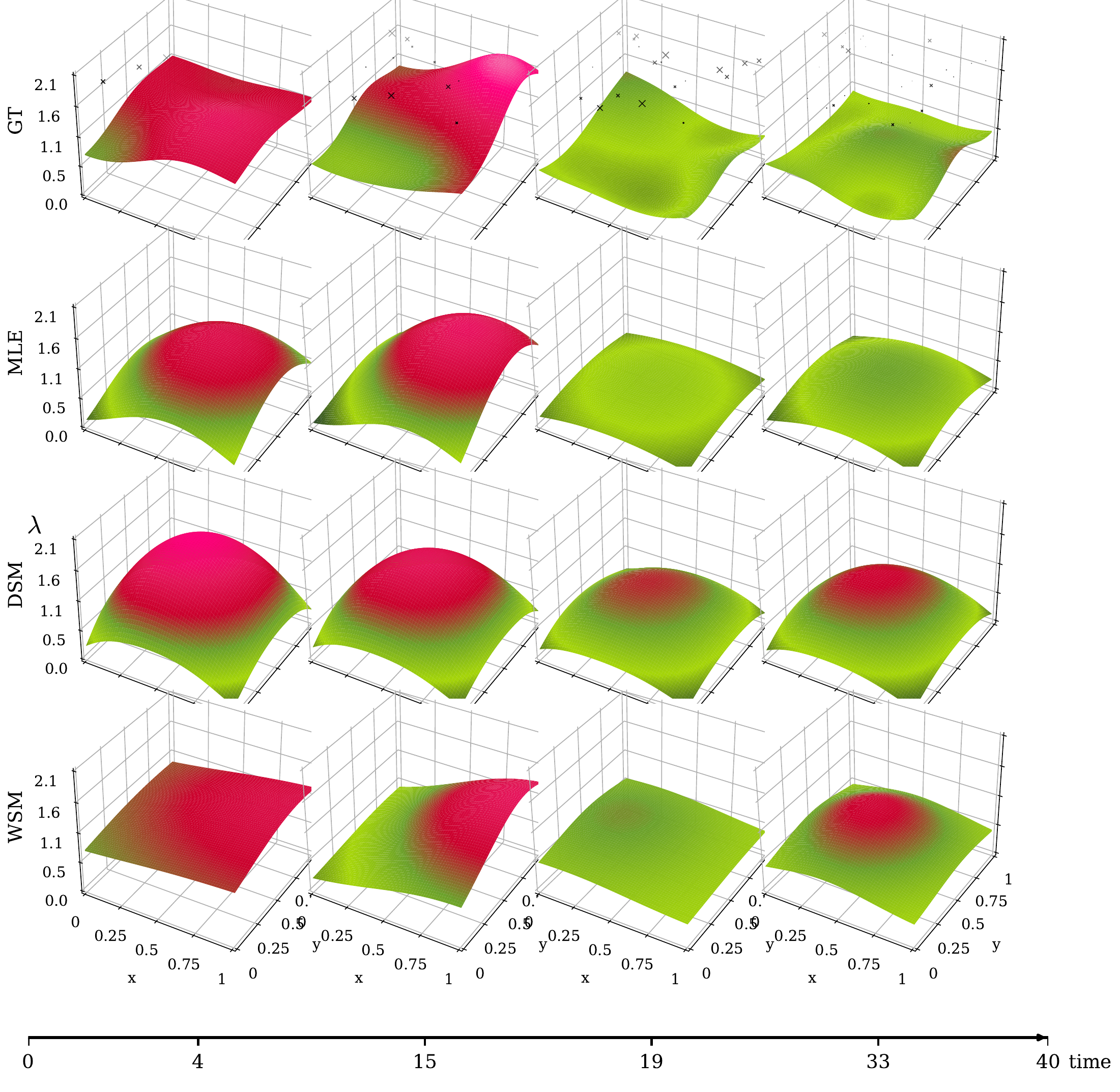} 
        \caption{Self-correcting Process 1}
        \label{fig:sub3}
    \end{subfigure}
    \hfill
    \begin{subfigure}{0.48\textwidth}
        \centering
        \includegraphics[width=\linewidth]{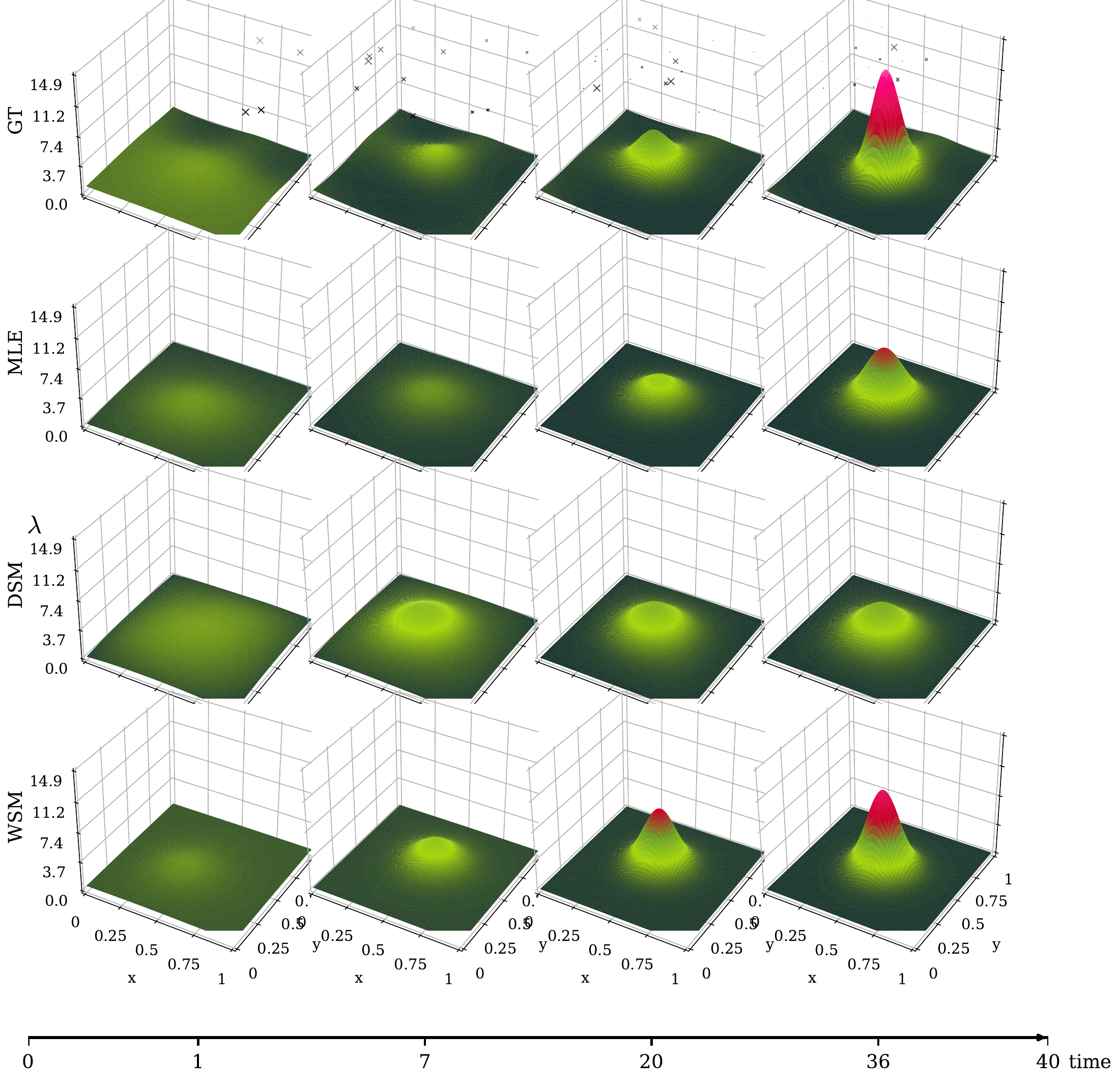} 
        \caption{Self-correcting Process 2}
        \label{fig:sub4}
    \end{subfigure}

    \caption{Ground-truth and learned intensities on two spatio-temporal synthetic datasets. \textbf{Top}: ground-truth intensity; \textbf{Middle}: learned intensities by MLE and DSM, \textbf{Bottom}: learned intensity by AWSM.}
    \label{fig:synthetic_spatial_temporal_intensity_visualize}
\end{figure}

\begin{figure}[t]
  \centering
  \setlength{\tabcolsep}{2pt} 

  \begin{subfigure}{\linewidth}
    \centering
    \includegraphics[width=0.18\linewidth]{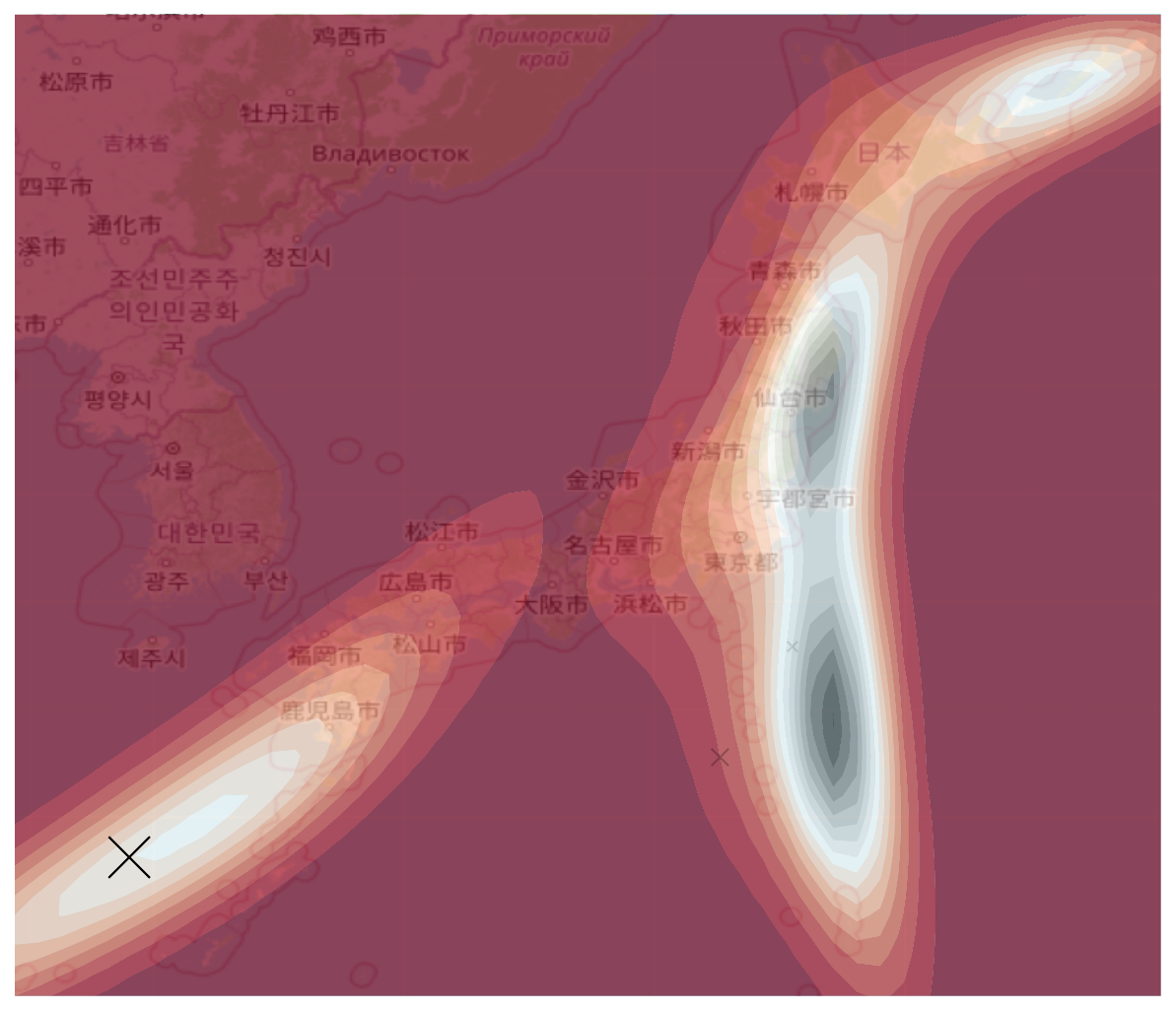}
    \includegraphics[width=0.18\linewidth]{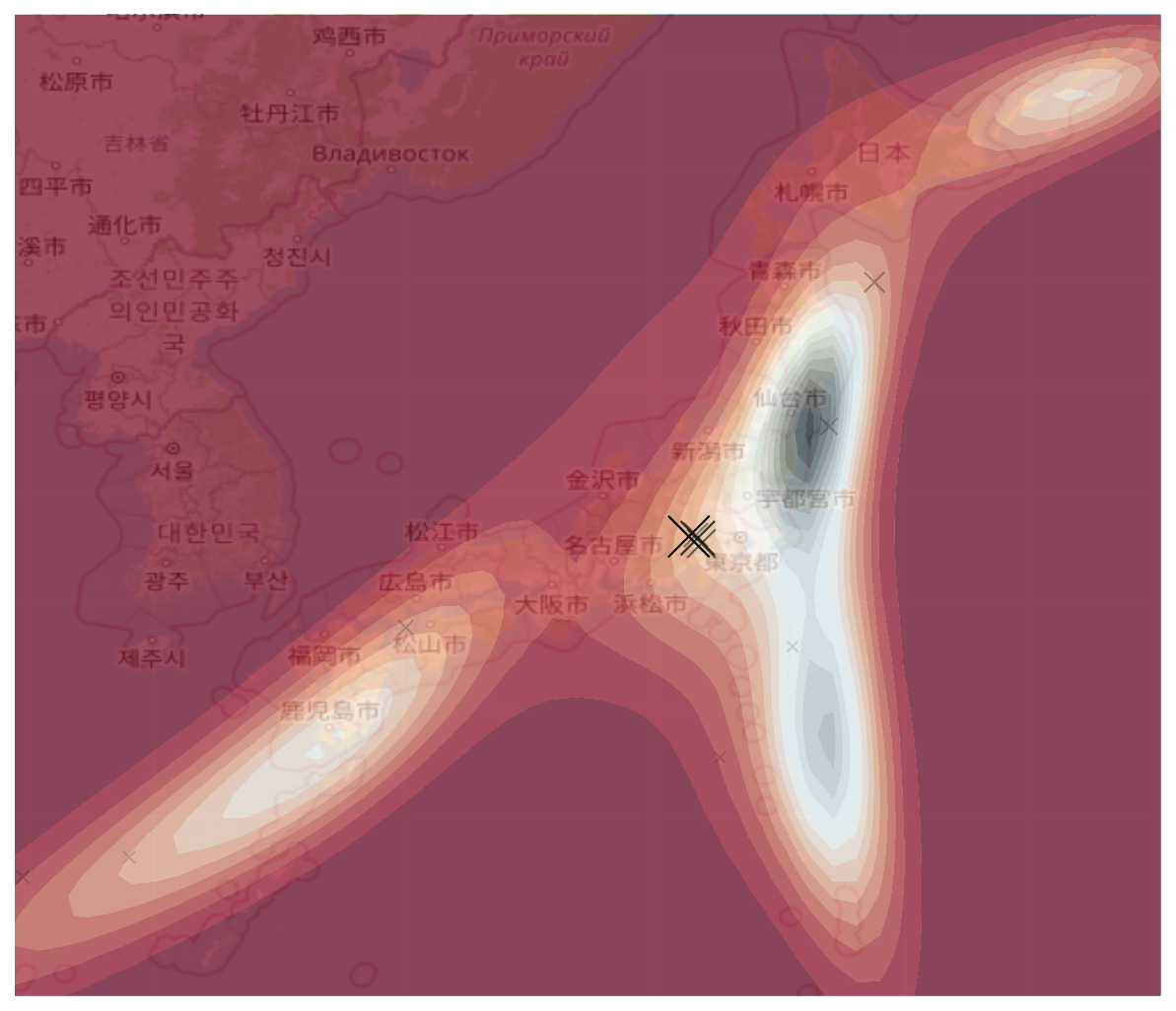}
    \includegraphics[width=0.18\linewidth]{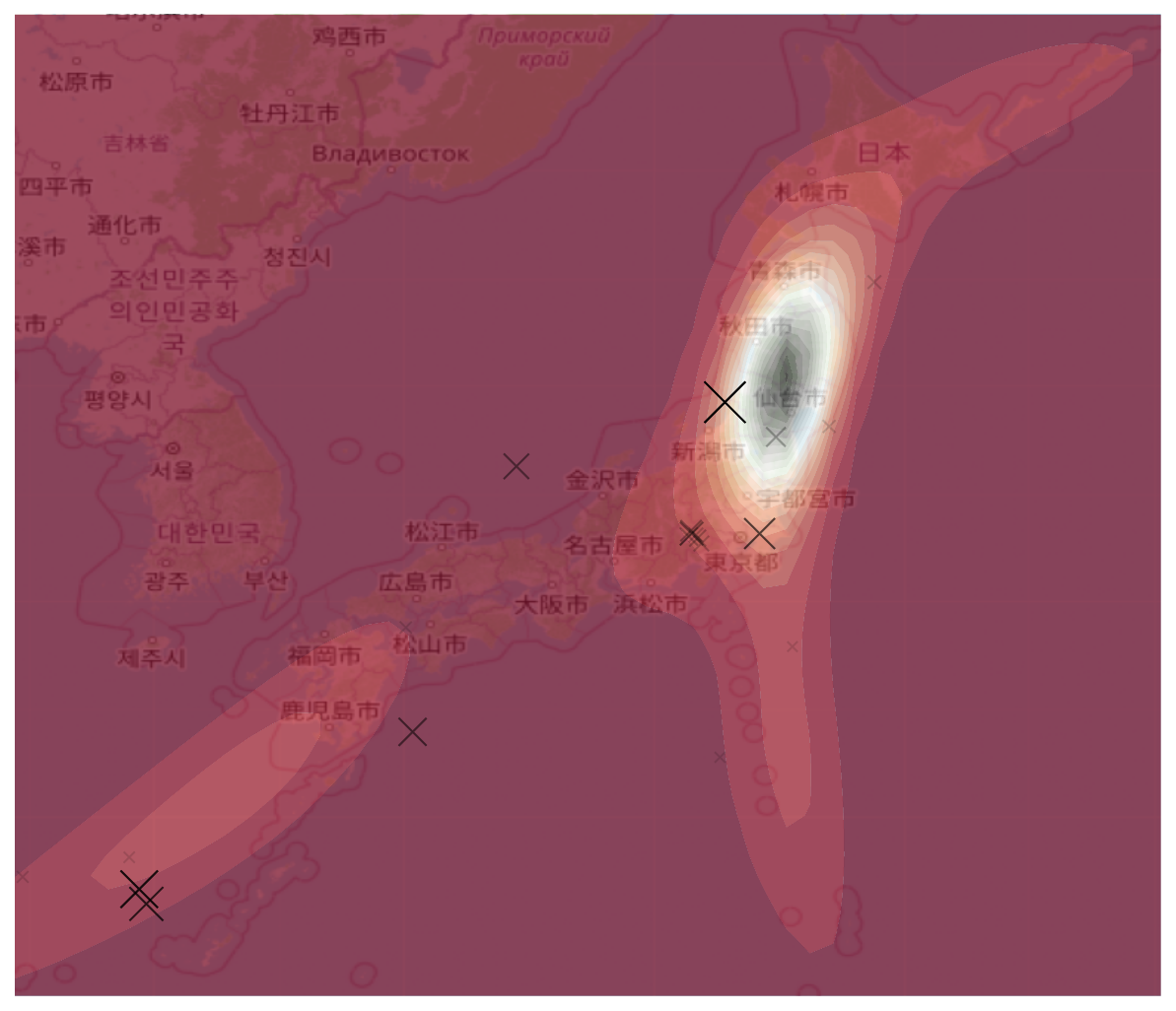}
    \includegraphics[width=0.18\linewidth]{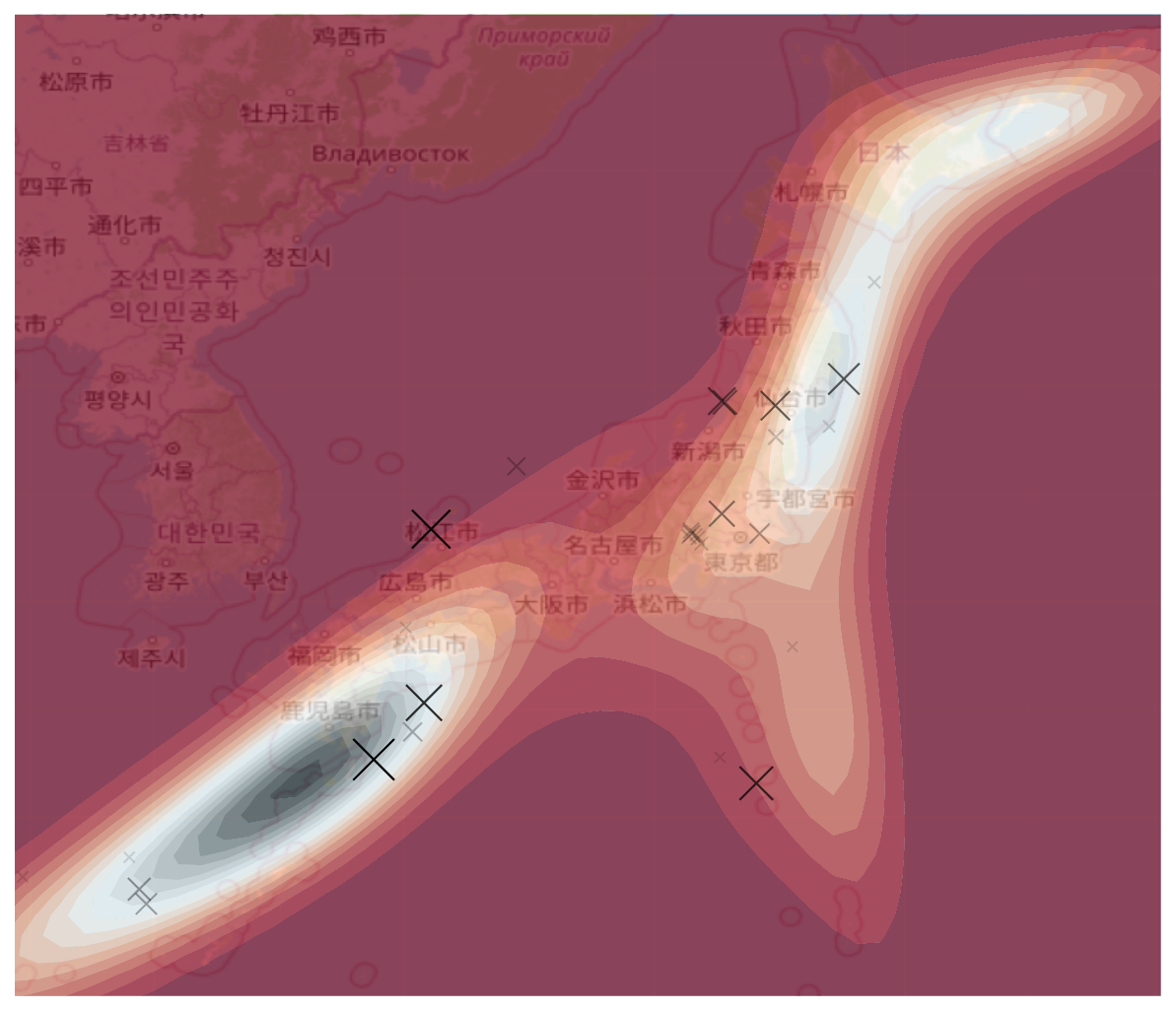}
    \includegraphics[width=0.18\linewidth]{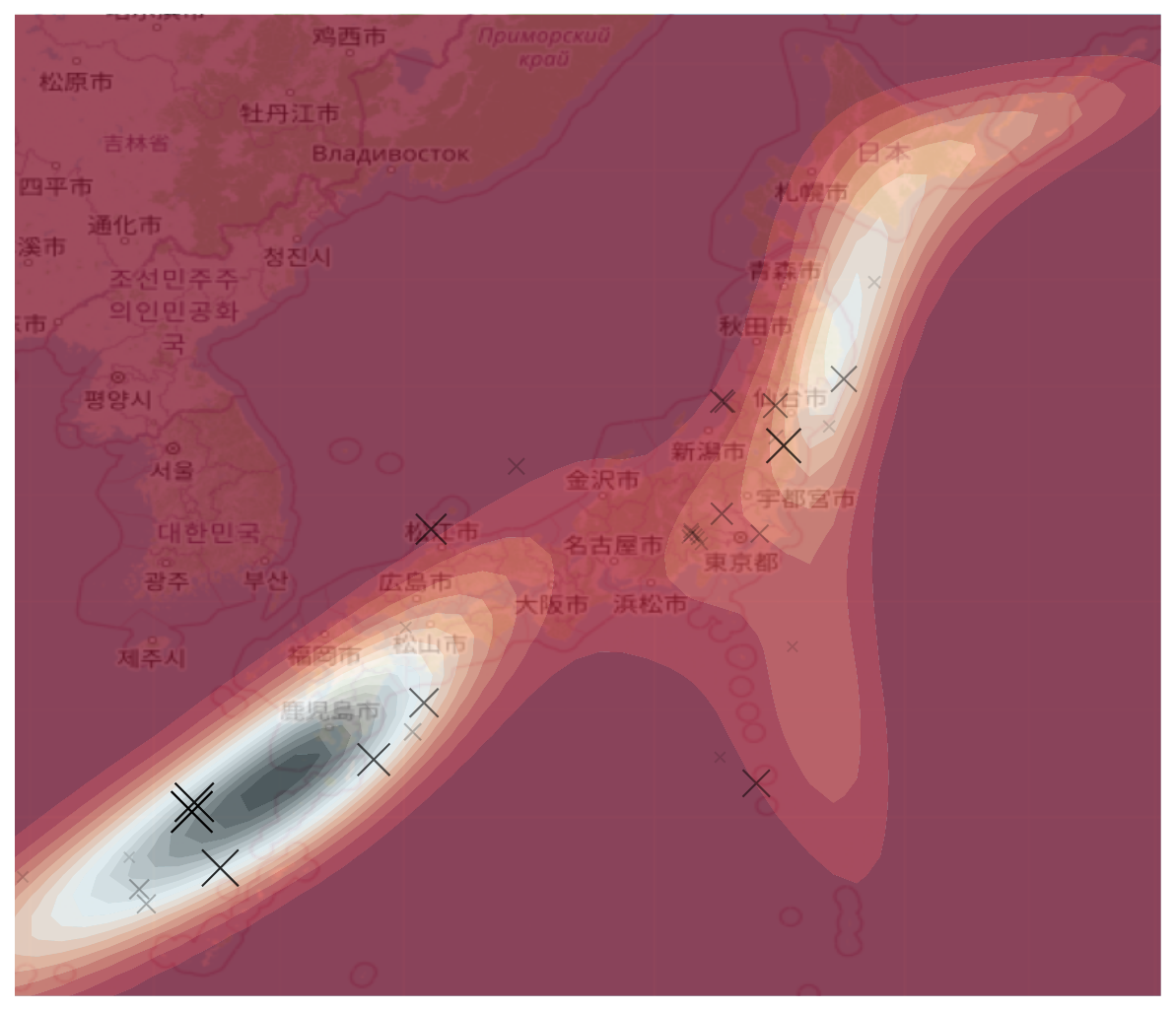}
    \caption{MLE on Earthquake}
  \end{subfigure}

  \medskip

  \begin{subfigure}{\linewidth}
    \centering
    \includegraphics[width=0.18\linewidth]{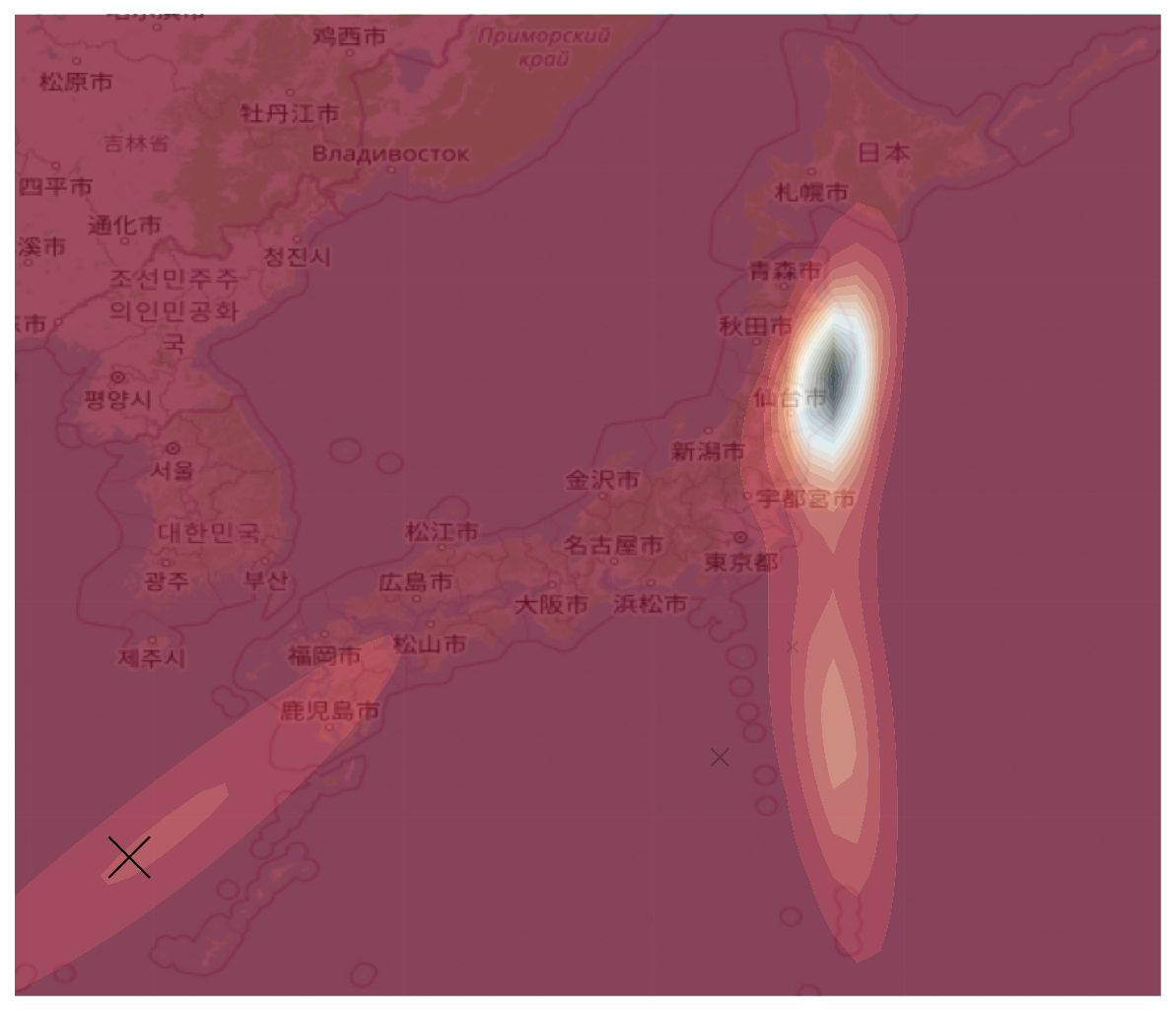}
    \includegraphics[width=0.18\linewidth]{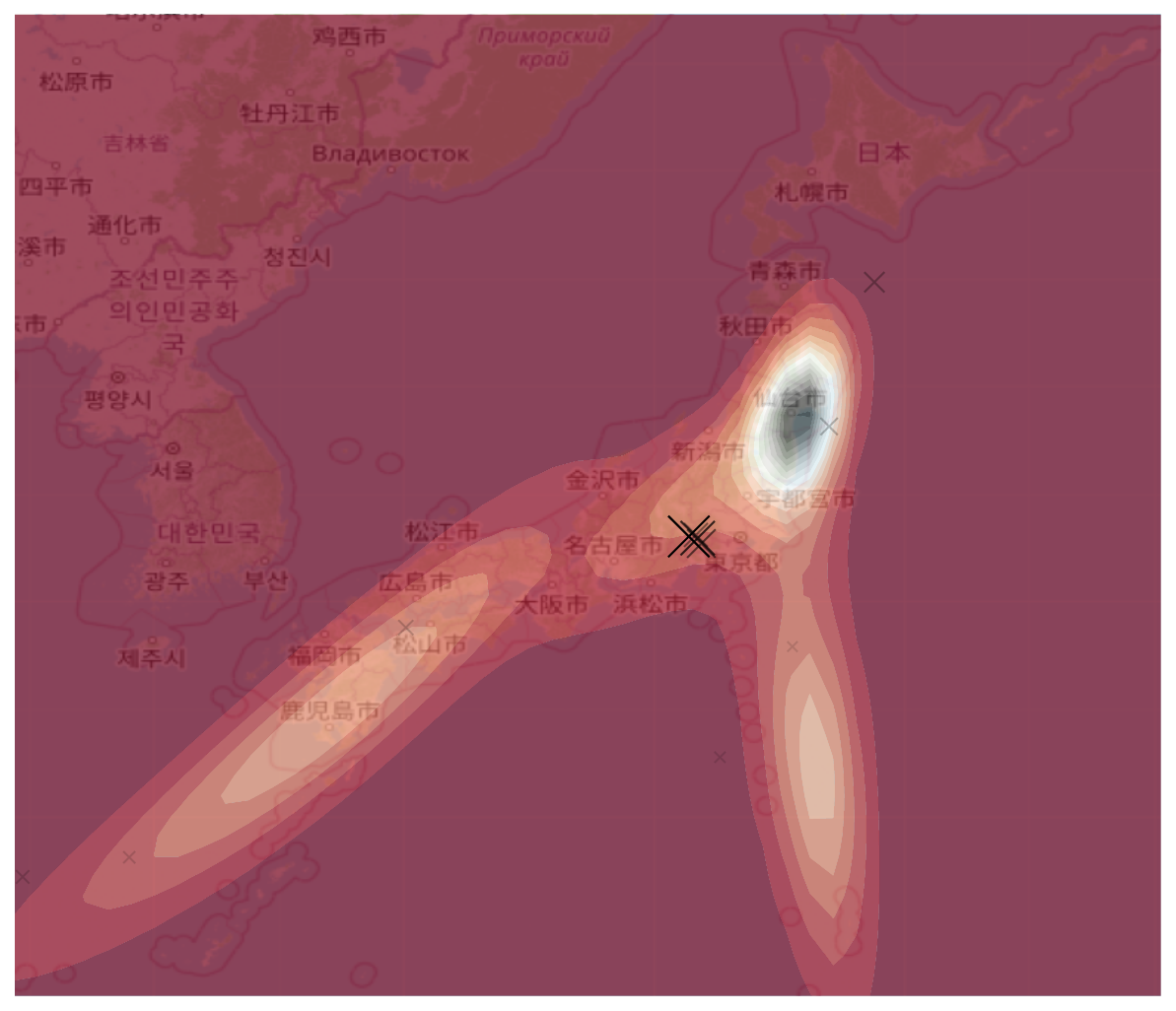}
    \includegraphics[width=0.18\linewidth]{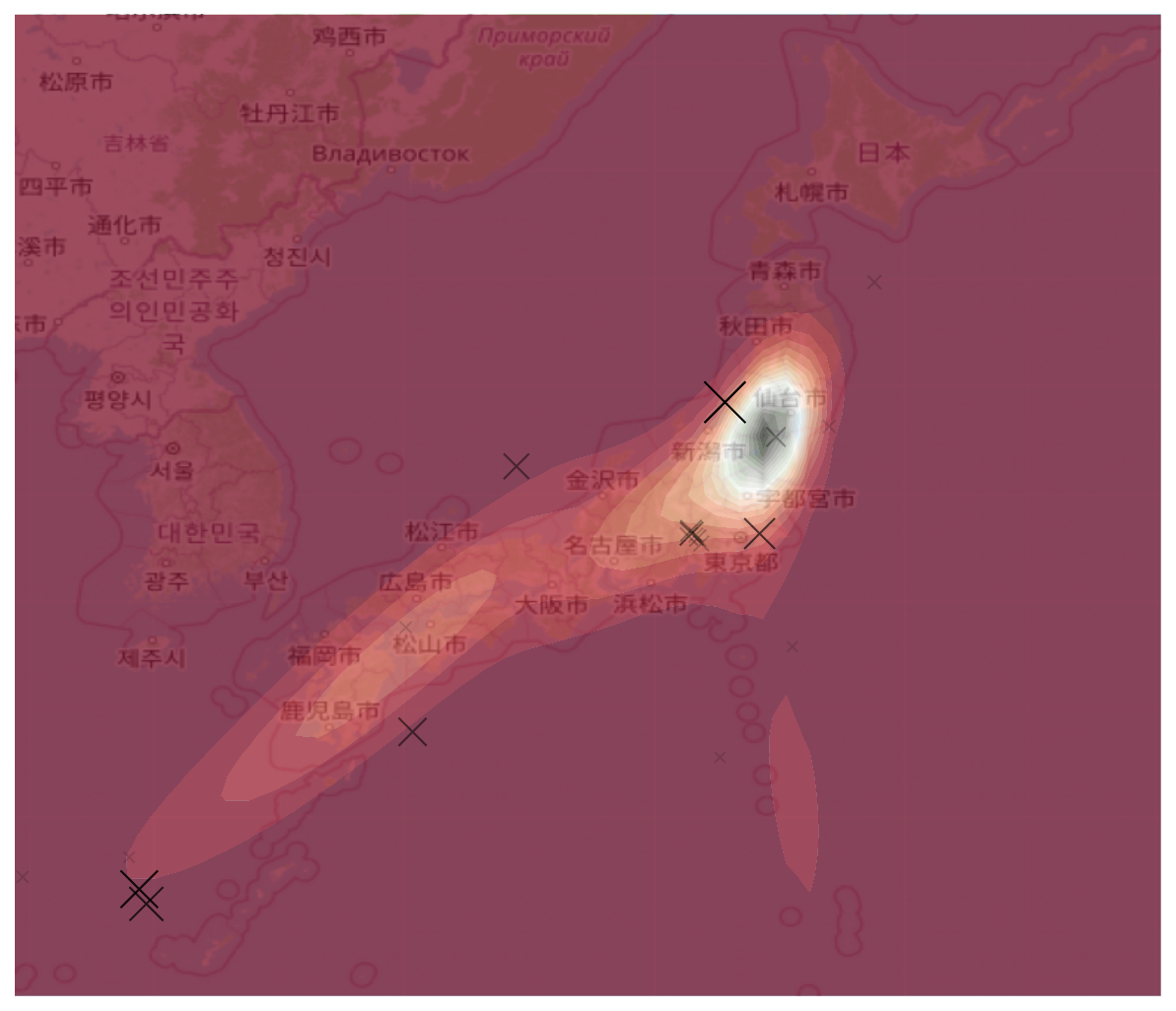}
    \includegraphics[width=0.18\linewidth]{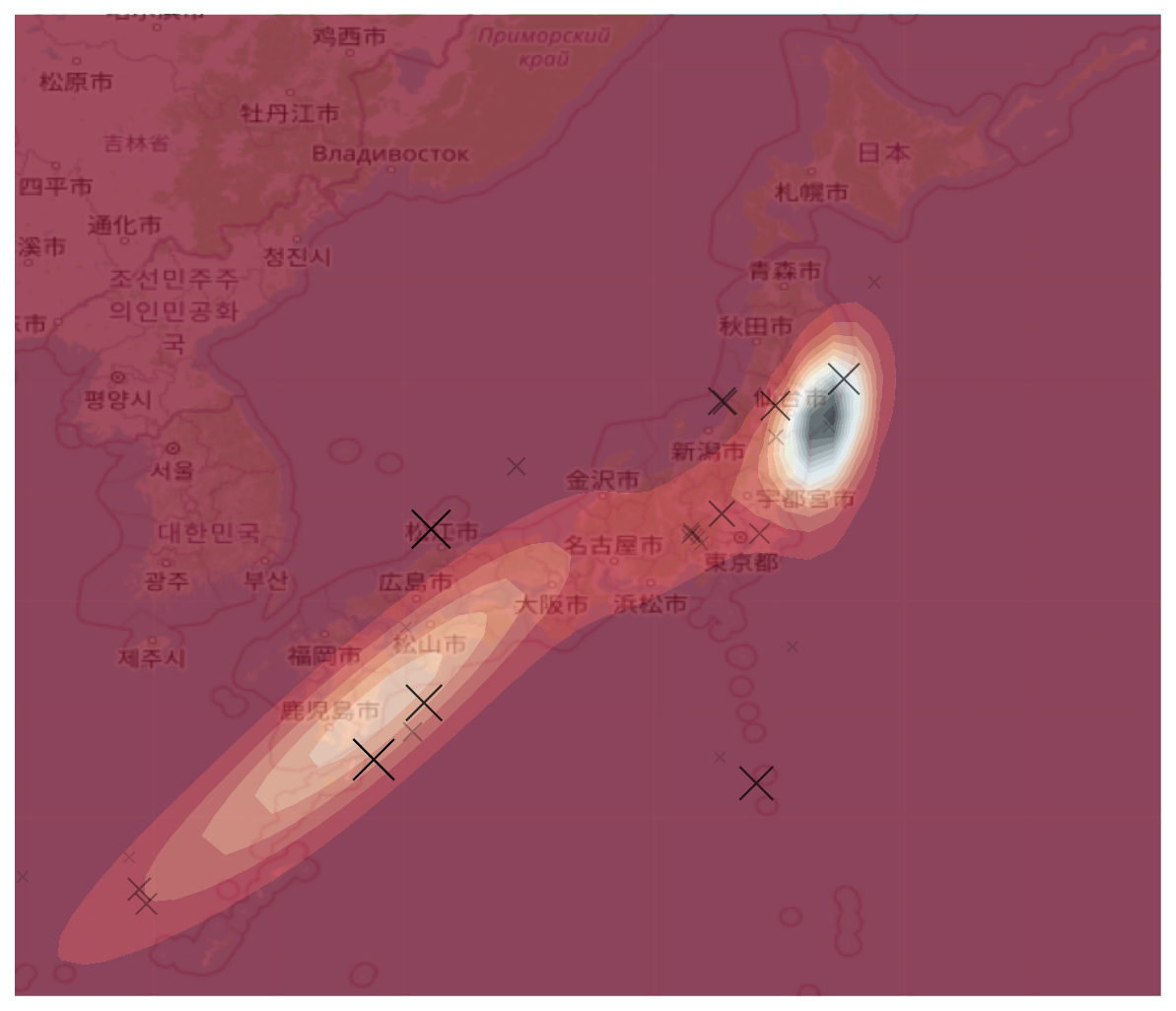}
    \includegraphics[width=0.18\linewidth]{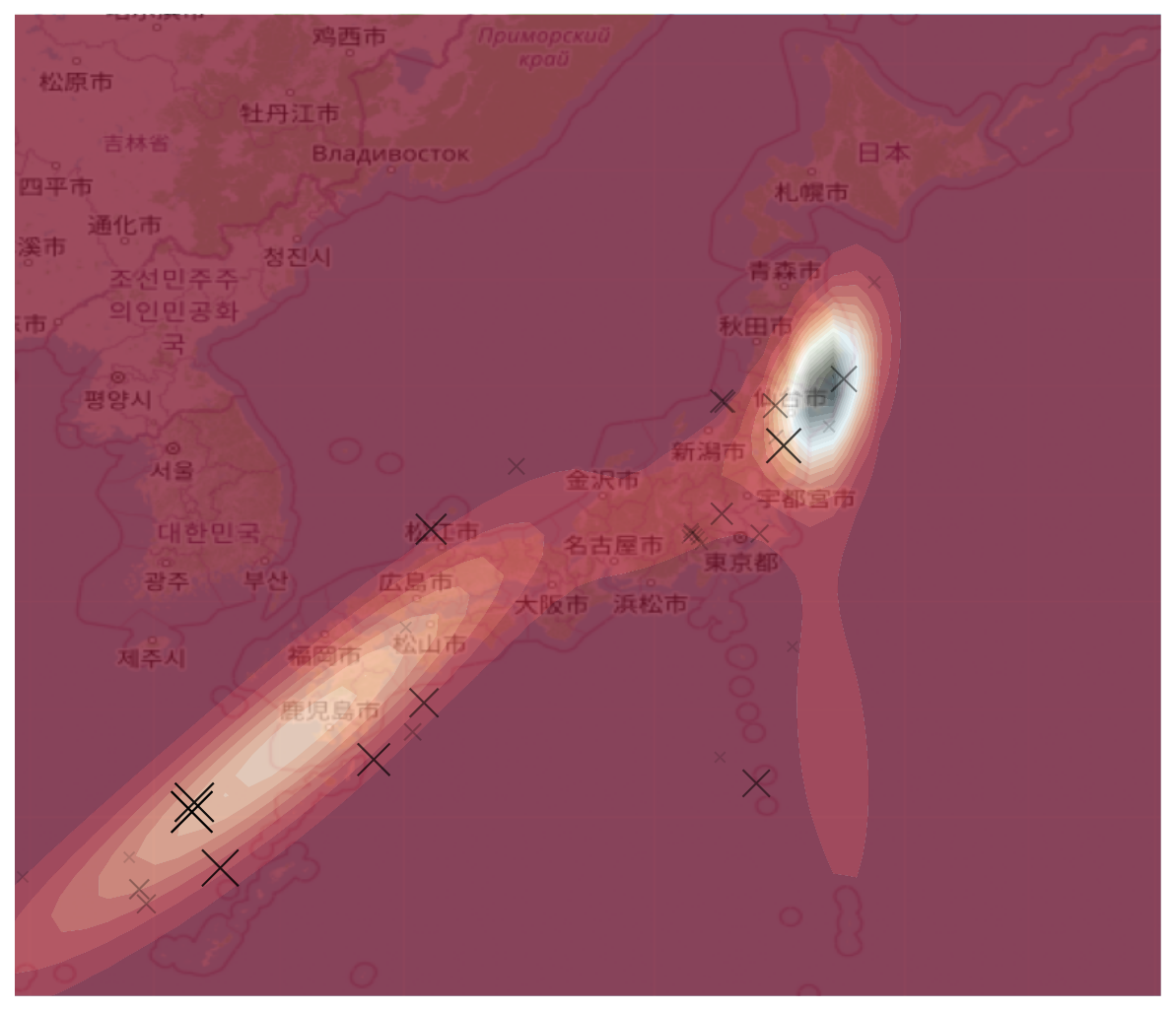}
    \caption{AWSM on Earthquake}
  \end{subfigure}

  \medskip

  \begin{subfigure}{\linewidth}
    \centering
    \includegraphics[width=0.18\linewidth]{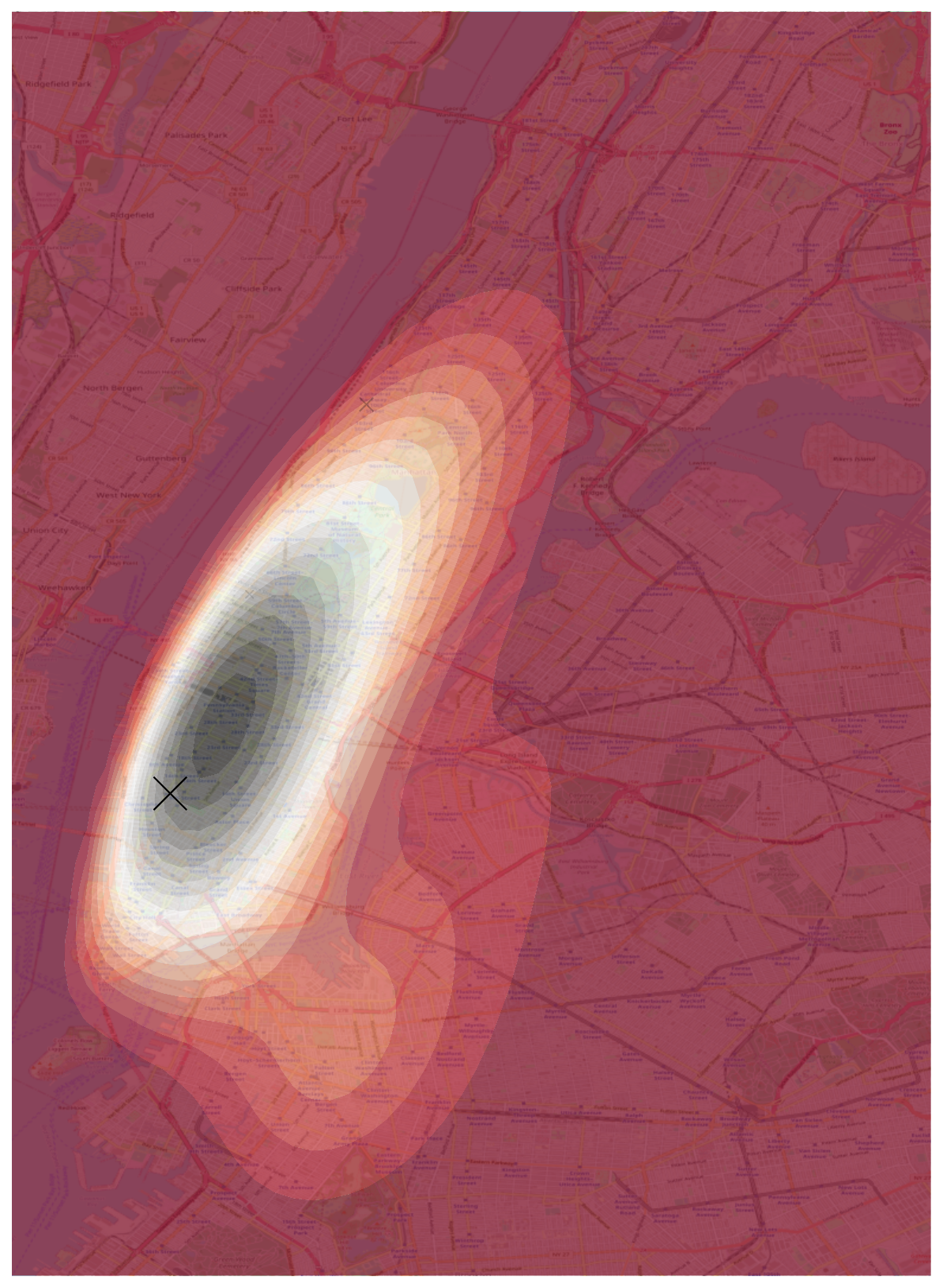}
    \includegraphics[width=0.18\linewidth]{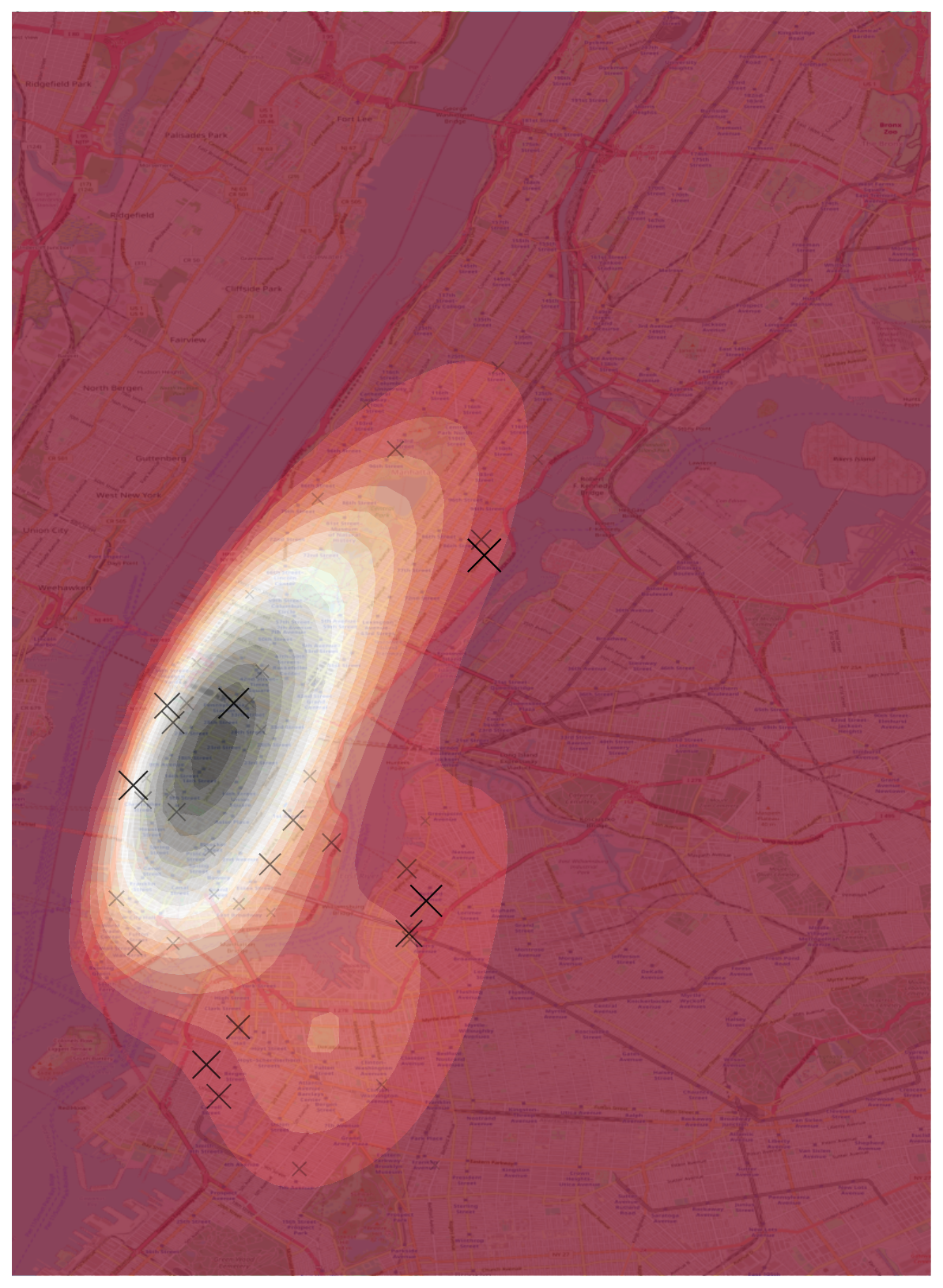}
    \includegraphics[width=0.18\linewidth]{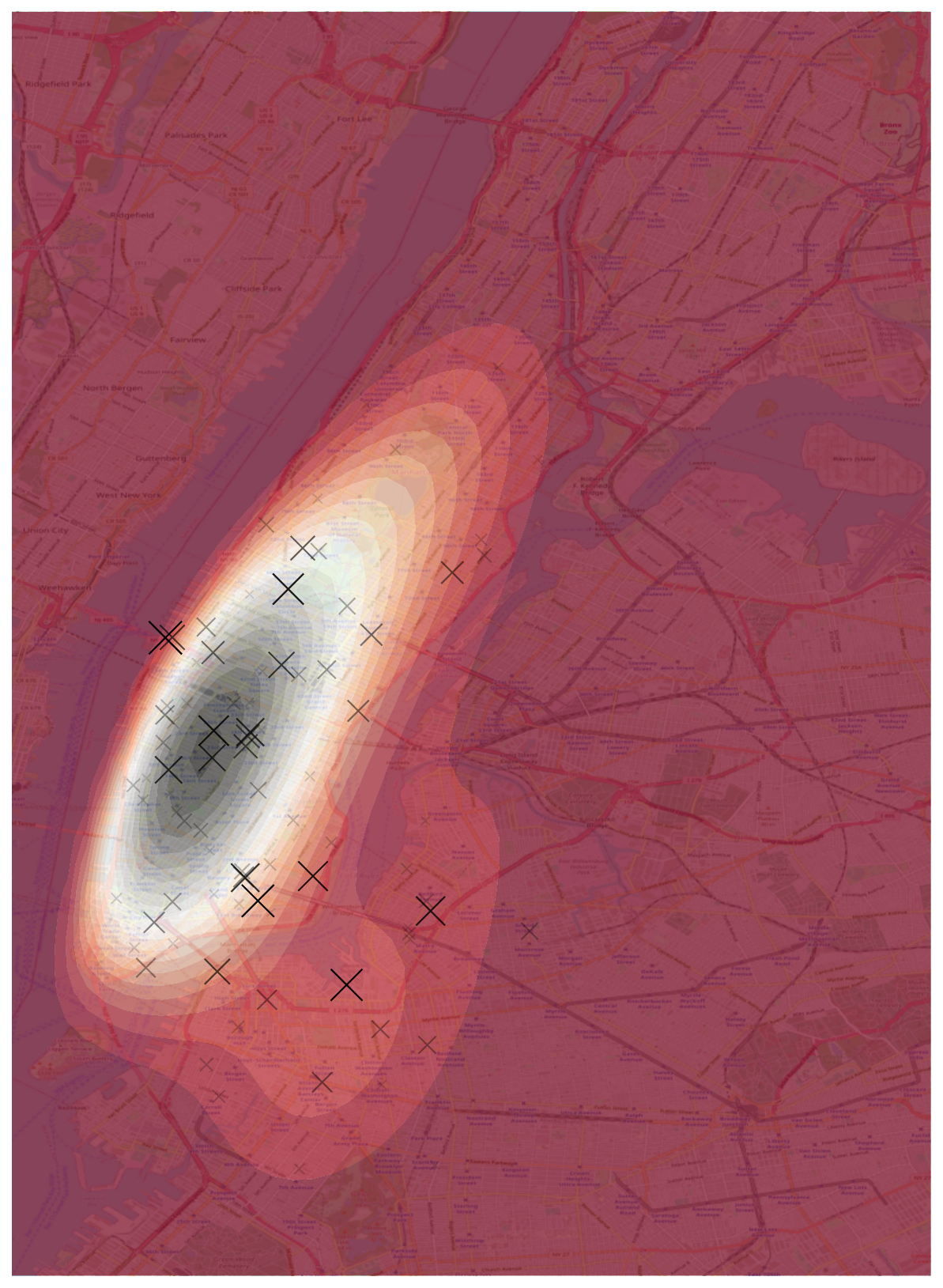}
    \includegraphics[width=0.18\linewidth]{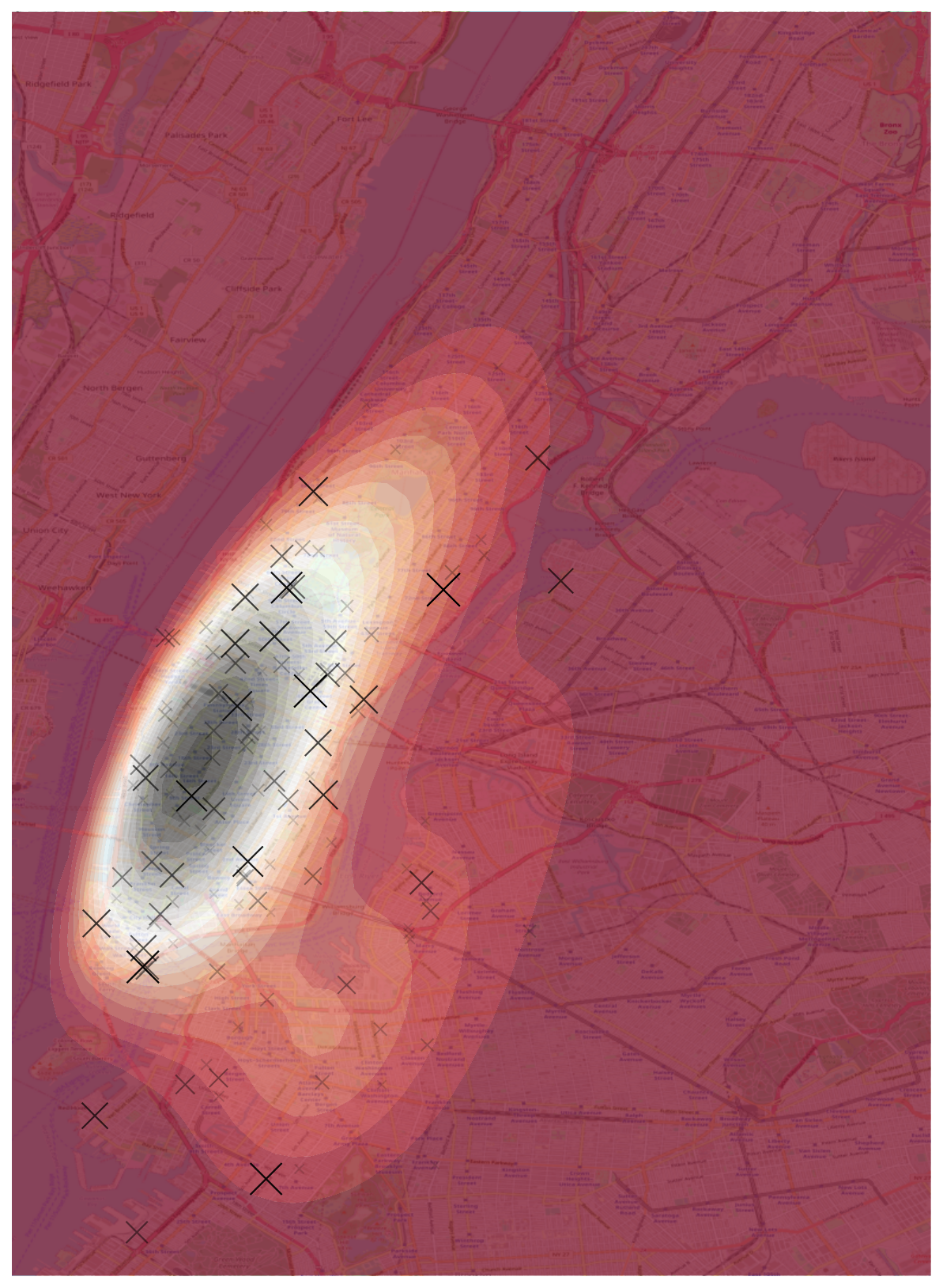}
    \includegraphics[width=0.18\linewidth]{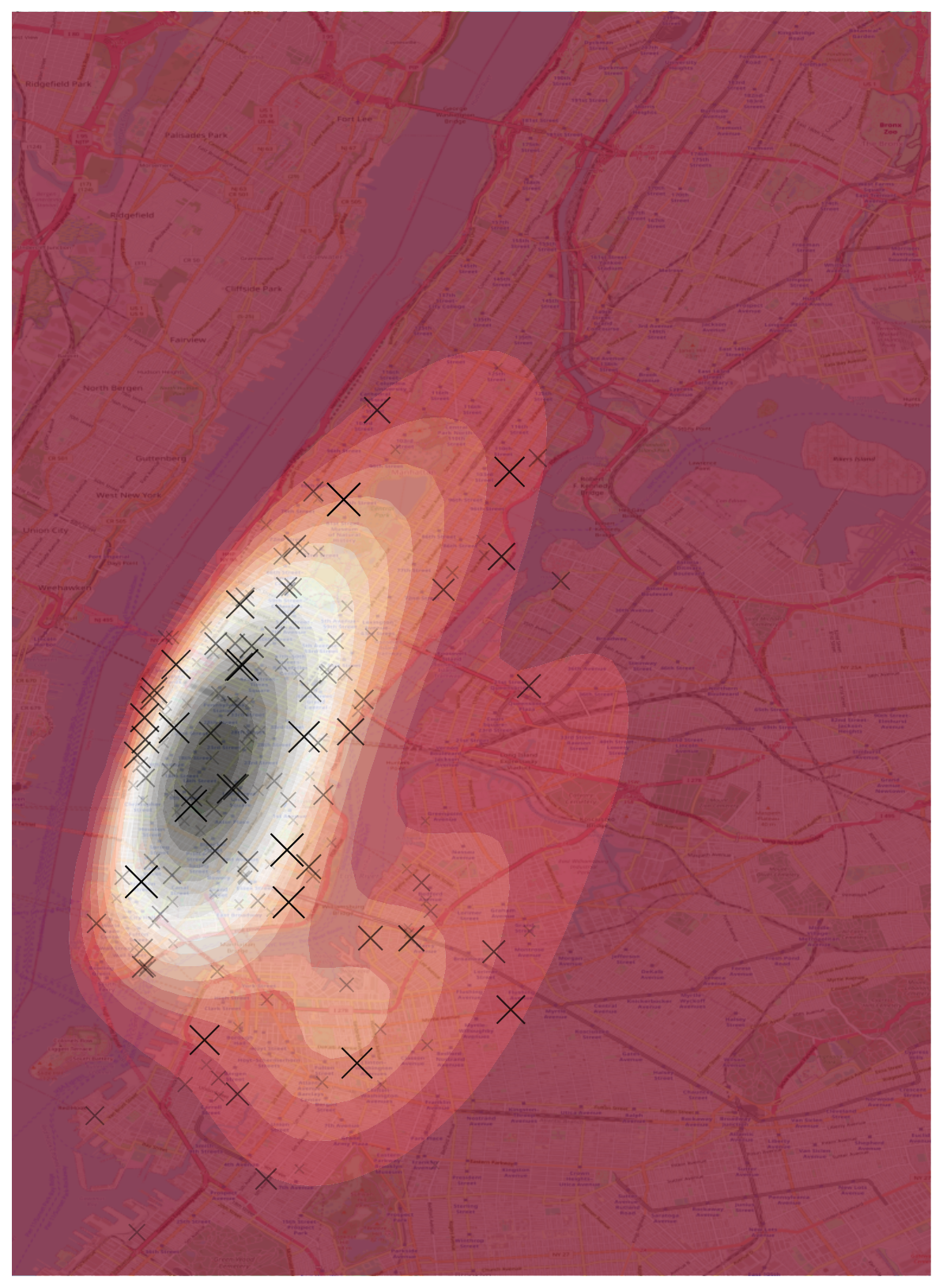}
    \caption{MLE on Citibike}
  \end{subfigure}

  \medskip

  \begin{subfigure}{\linewidth}
    \centering
    \includegraphics[width=0.18\linewidth]{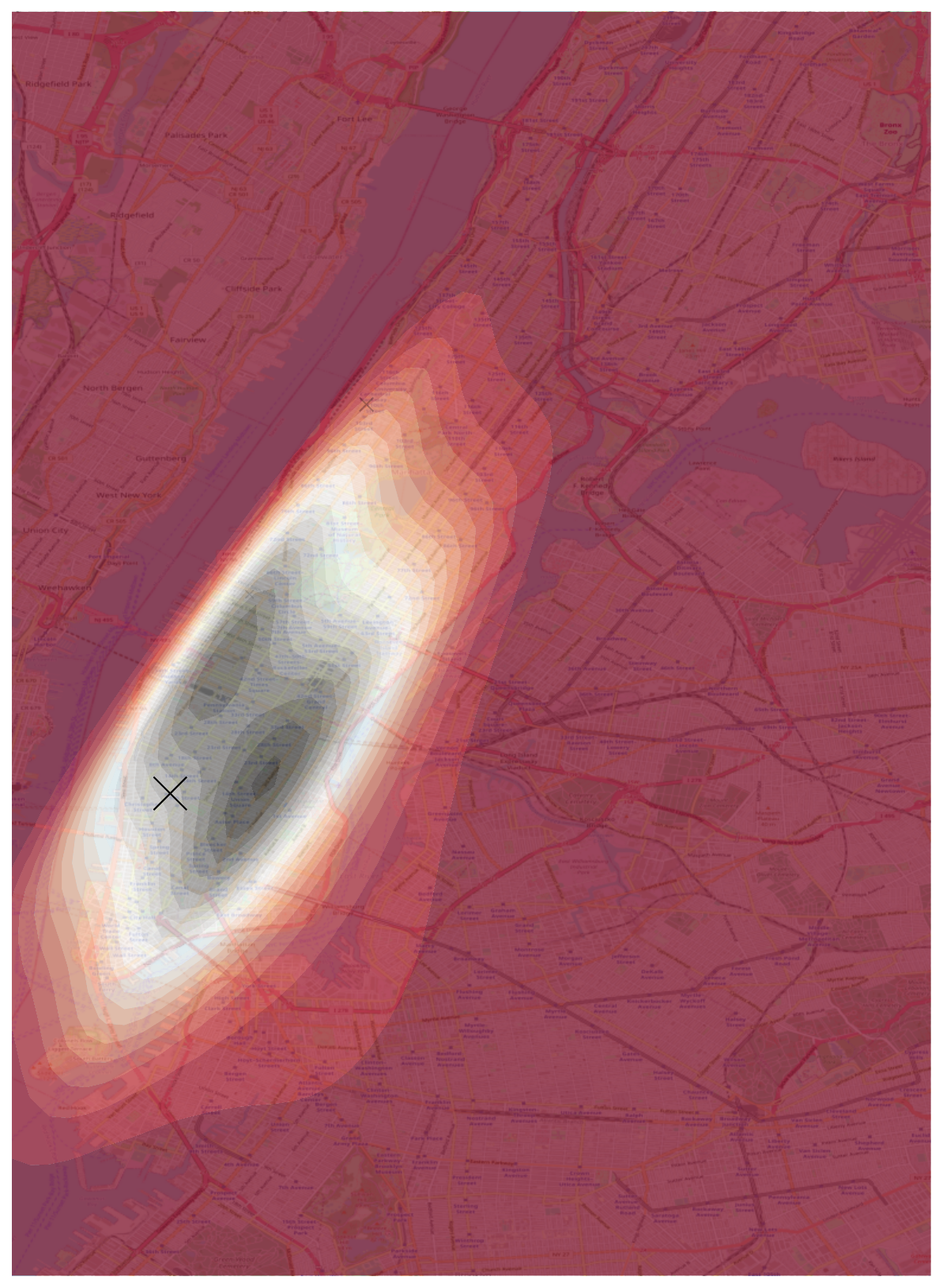}
    \includegraphics[width=0.18\linewidth]{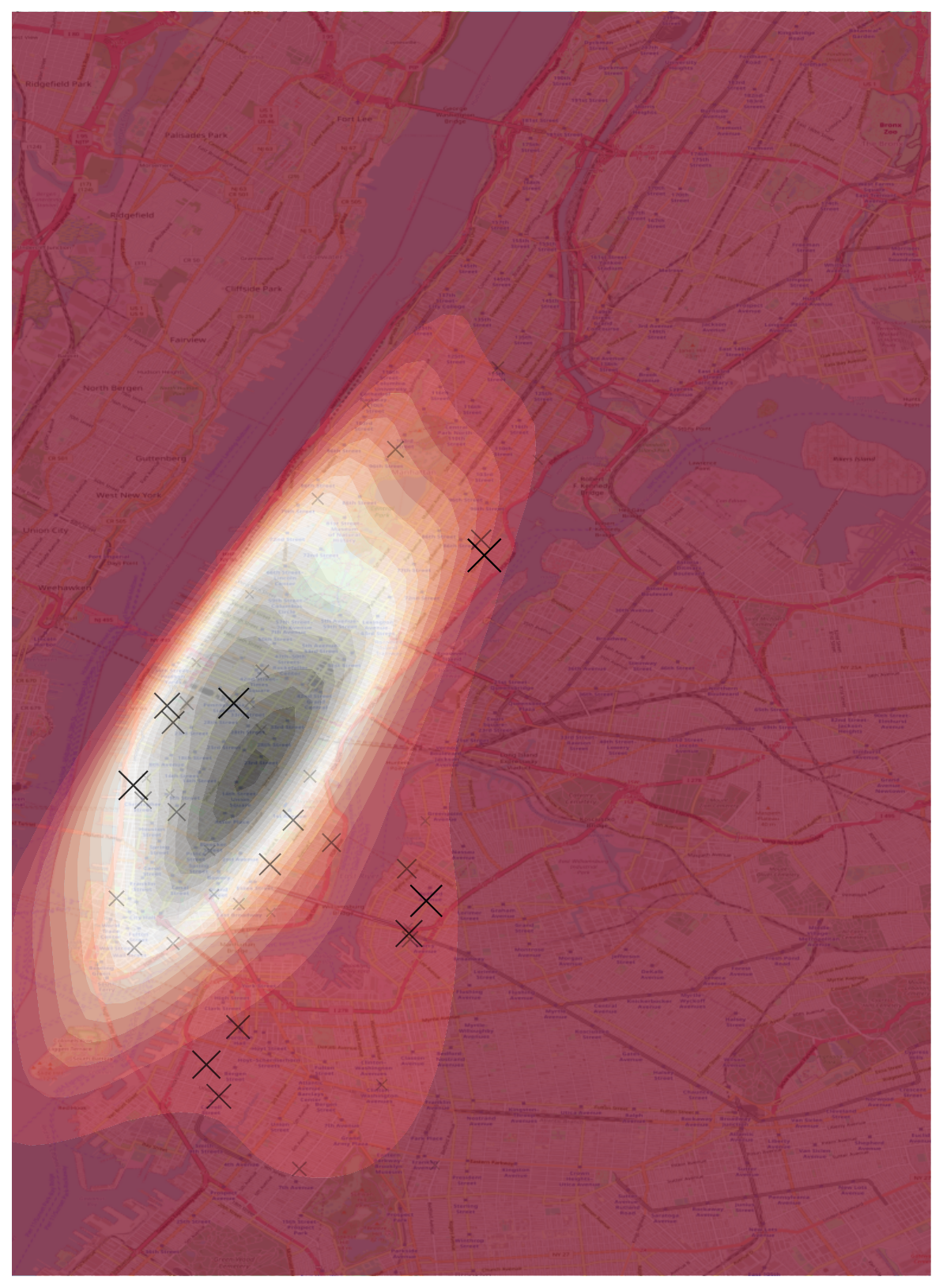}
    \includegraphics[width=0.18\linewidth]{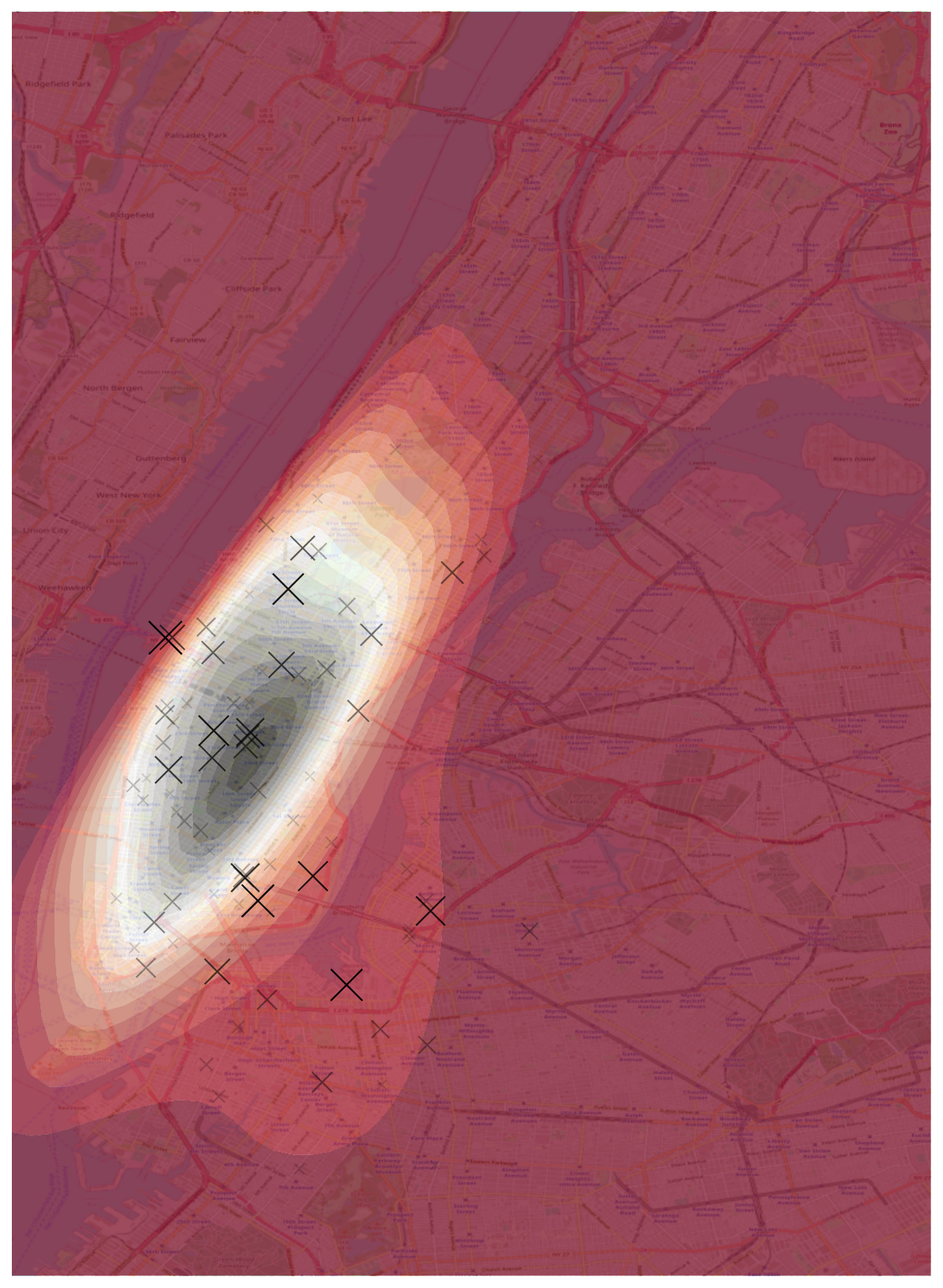}
    \includegraphics[width=0.18\linewidth]{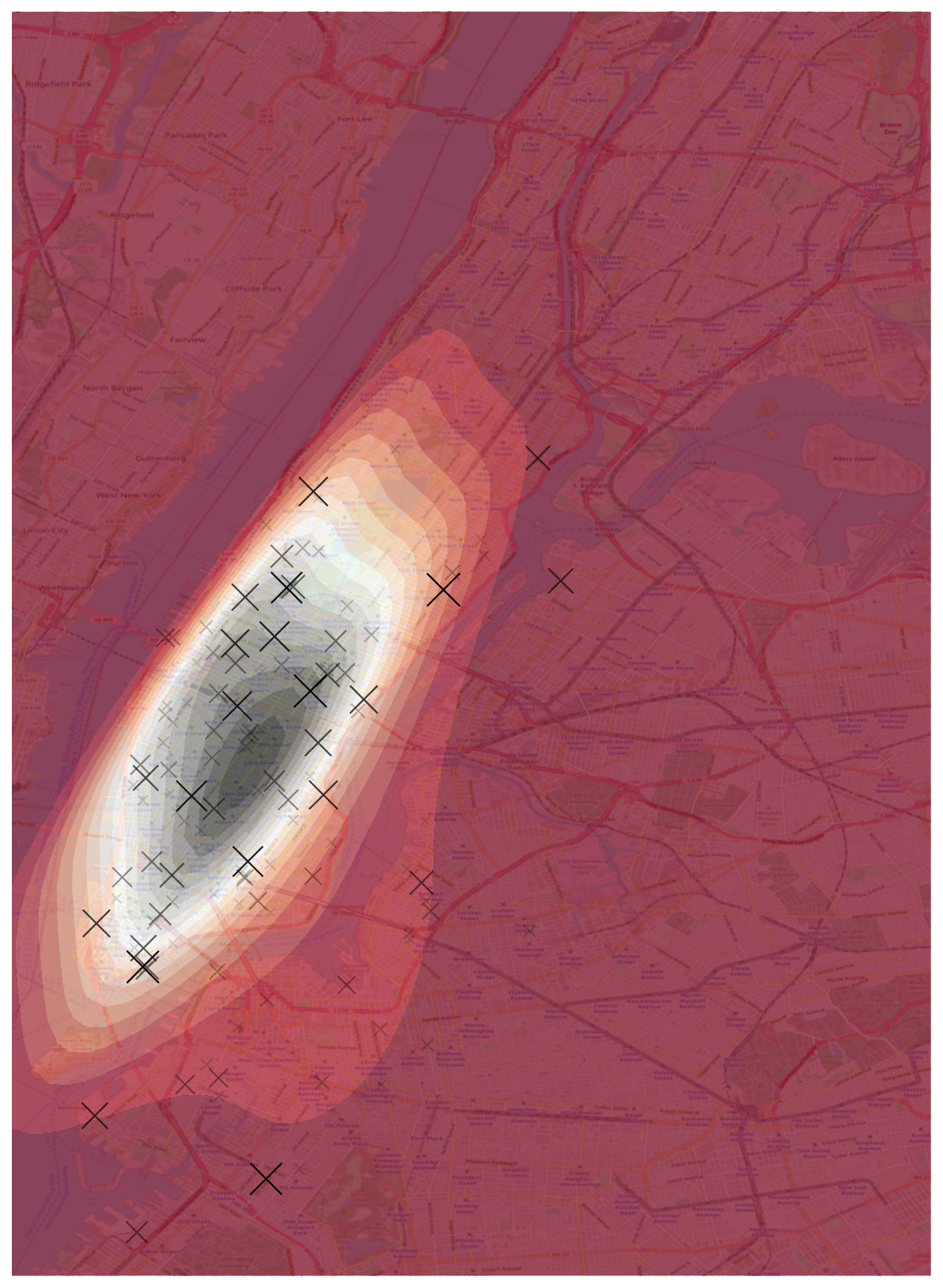}
    \includegraphics[width=0.18\linewidth]{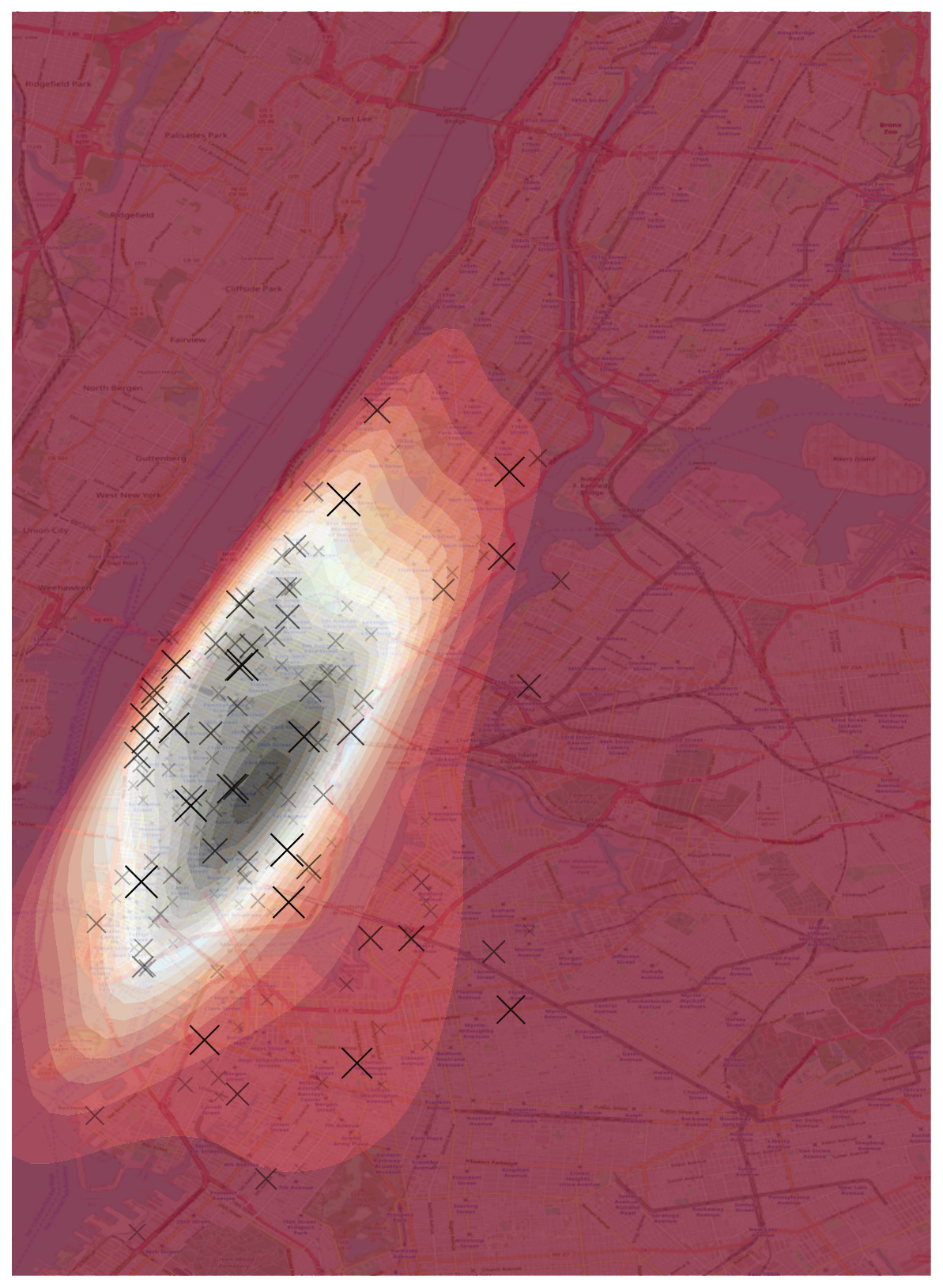}
    \caption{AWSM on Citibike}
  \end{subfigure}

  \caption{Snapshots of the conditional intensity learned by SMASH trained with MLE and AWSM on the Earthquake and CitiBike datasets. From left to right: intensity at different timestamps.  Observed events are marked with “$\times$”, whose influence decays over time; brighter regions indicate higher intensity.}
  \label{fig:earthquake_map_visual}
\end{figure}

\begin{figure}[t]
    \centering

    \begin{subfigure}{0.48\linewidth}
        \centering
        \includegraphics[width=\linewidth]{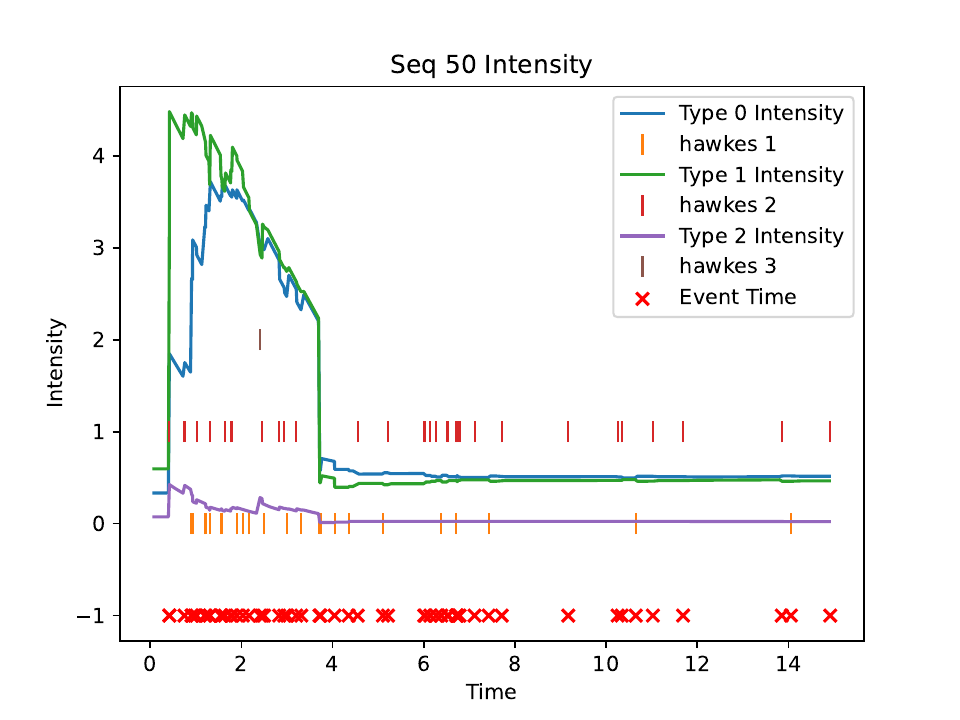}
        \caption{MLE}
        \label{fig:long}
    \end{subfigure}
    \hspace{0.01\linewidth} 
    \begin{subfigure}{0.48\linewidth}
        \centering
        \includegraphics[width=\linewidth]{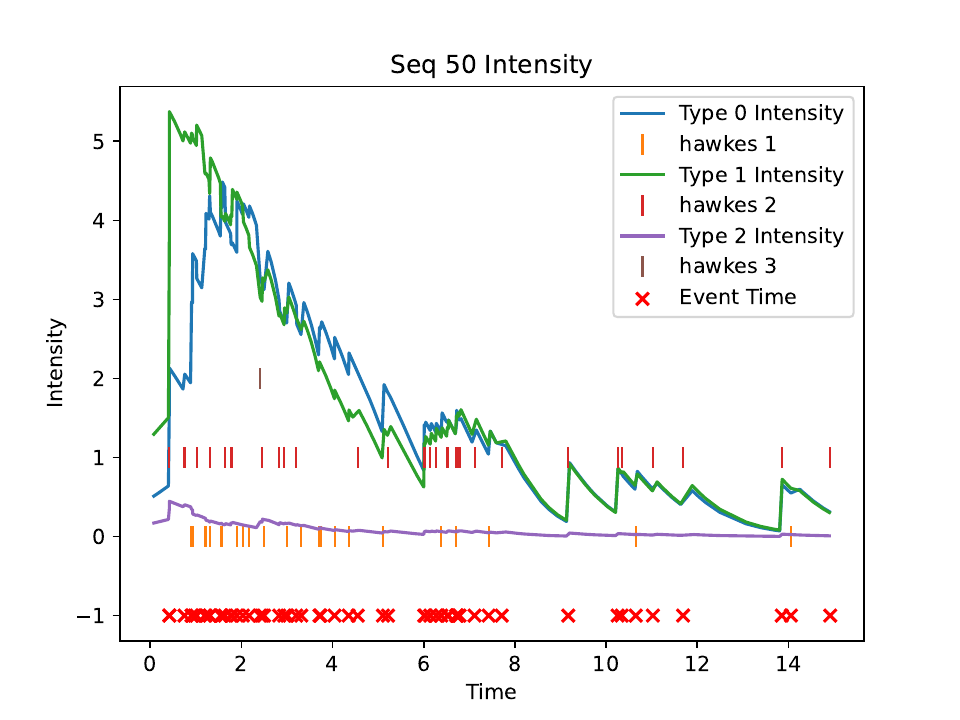}
        \caption{AWSM}
        \label{fig:left}
    \end{subfigure}

    \caption{Learned intensities from THP with MLE and AWSM objectives on Retweet dataset.} 
    \label{learned_intensity}
\end{figure}

\begin{figure}[t]
  \centering
  \begin{subfigure}{0.48\linewidth}
    \centering
    \includegraphics[width=\linewidth]{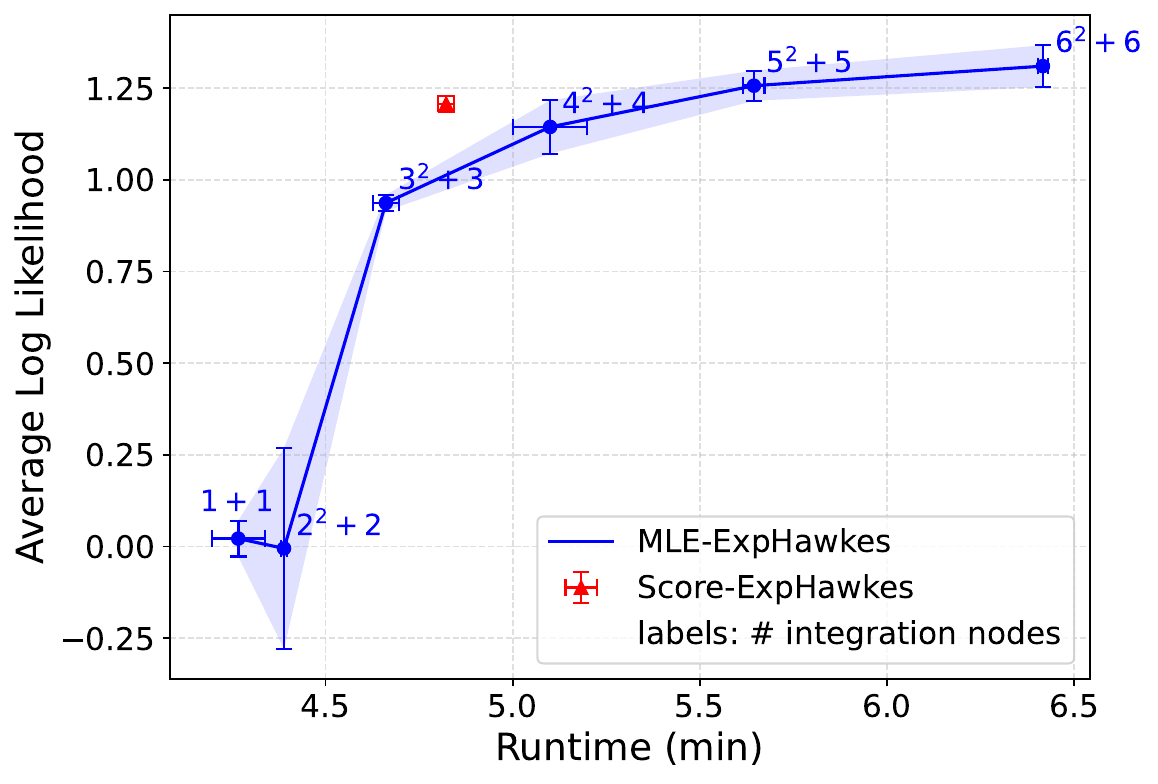}
    \caption{Result on Earthquake Dataset}
  \end{subfigure}
  \hfill
  \begin{subfigure}{0.48\linewidth}
    \centering
    \includegraphics[width=\linewidth]{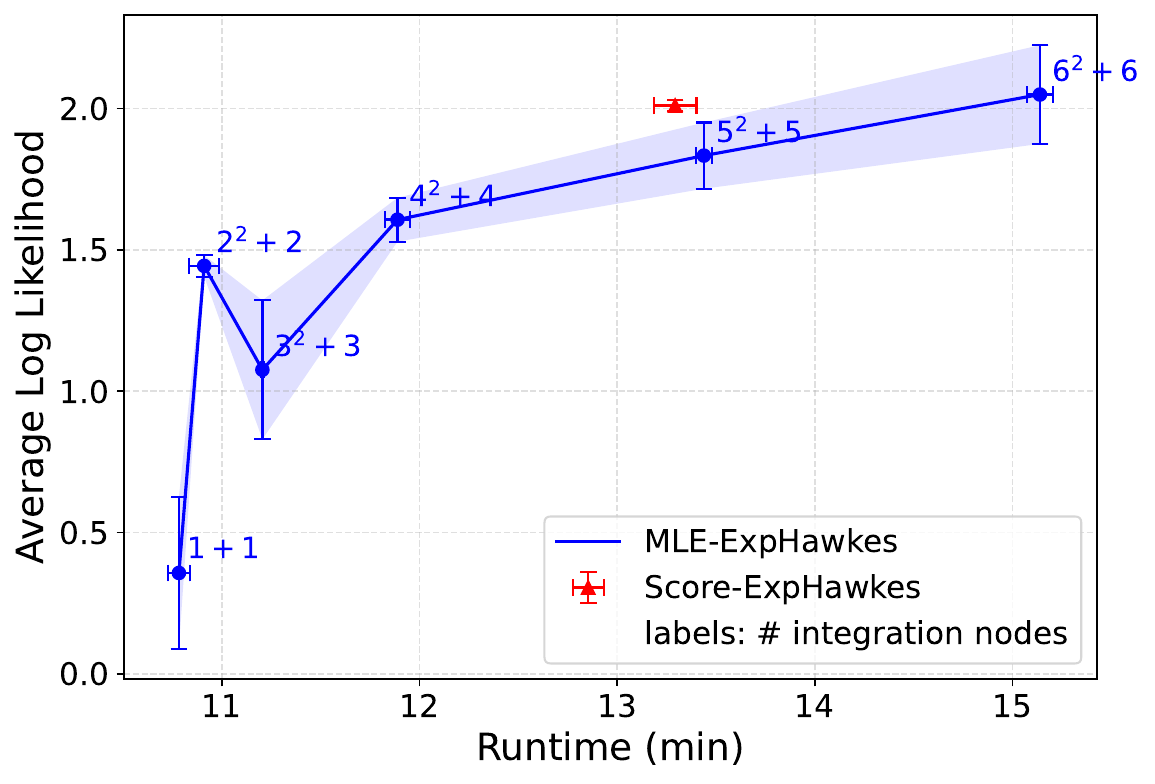}
    \caption{Result on Citibike Dataset}
  \end{subfigure}

  \caption{
  Average TLL versus RT for different choices of integration nodes. 
  In each panel, the blue solid line with circular markers corresponds to the MLE estimator, 
  and the red triangle corresponds to the AWSM estimator; error bars indicate one standard deviation 
  over repeated runs. The text labels next to the blue markers indicate the integration-node configuration. 
  For example, $2^2+2$ means using two spatial quadrature nodes in each spatial dimension and two temporal 
  quadrature nodes between any two consecutive event times. 
  }
  \label{fig:time_accuracy_tradeoff}
\end{figure}




\begin{table}[t]
    \centering
    \caption{Performance of SAHP and THP trained on four real-world temporal datasets with three different training objectives.}
    \label{table:results_on_temporal_dataset_split}
    \setlength{\tabcolsep}{2pt}
    \renewcommand{\arraystretch}{0.9}
    \small  

    \begin{subtable}{\linewidth}
        \centering
        \begin{sc}
        \resizebox{\linewidth}{!}{%
        \begin{tabular}{l|ccc|ccc}
            \toprule
            & \multicolumn{3}{c}{taxi}
            & \multicolumn{3}{c}{retweet} \\
            \cmidrule(lr){2-4}
            \cmidrule(lr){5-7}
            & TLL($\uparrow$) & ACC($\uparrow$) & RT(mins) ($\downarrow$)
            & TLL($\uparrow$) & ACC($\uparrow$) & RT(mins) ($\downarrow$) \\
            \midrule
            \multicolumn{7}{c}{\textbf{SAHP}} \\
            \midrule
            MLE
            & \bm{$0.316_{\pm 0.004}$} & \bm{$0.927_{\pm 0.001}$} &$3.457_{\pm 0.089}$ 
            & $-0.756_{\pm 0.620}$     & $0.578_{\pm 0.028}$      & $6.766_{\pm 0.132}$ \\
            DSM
            & $-0.196_{\pm 0.037}$     & $0.912_{\pm 0.001}$      & \bm{$3.119_{\pm 0.026}$}
            & $-0.503_{\pm 0.031}$     & \bm{$0.598_{\pm 0.002}$} & $6.079_{\pm 0.053}$ \\
            WSM
            & $0.213_{\pm 0.074}$      & $0.926_{\pm 0.000}$      & $3.170_{\pm 0.081}$
            & \bm{$0.437_{\pm 1.249}$} & $0.556_{\pm 0.061}$      & \bm{$5.963_{\pm 0.068}$} \\
            \midrule
            \multicolumn{7}{c}{\textbf{THP}} \\
            \midrule
            MLE
            & $0.308_{\pm 0.010}$      & $0.910_{\pm 0.000}$      & $3.255_{\pm 0.047}$
            & $-0.528_{\pm 0.004}$     & $0.599_{\pm 0.000}$      & $7.274_{\pm 0.054}$ \\
            DSM
            & $0.091_{\pm 0.017}$      & \bm{$0.911_{\pm 0.000}$} & \bm{$3.101_{\pm 0.076}$}
            & $-0.460_{\pm 0.013}$     & $0.600_{\pm 0.001}$      & $6.981_{\pm 0.030}$   \\
            WSM
            & \bm{$0.619_{\pm 0.383}$} & $0.876_{\pm 0.012}$      & $3.141_{\pm 0.024}$
            & \bm{$-0.455_{\pm 0.024}$}& \bm{$0.601_{\pm 0.001}$} & \bm{$6.971_{\pm 0.090}$} \\
            \bottomrule
        \end{tabular}
        }
        \end{sc}
    \end{subtable}

    \vskip 0.1in

    \begin{subtable}{\linewidth}
        \centering
        \begin{sc}
        \resizebox{\linewidth}{!}{%
        \begin{tabular}{l|ccc|ccc}
            \toprule
            & \multicolumn{3}{c}{stackoverflow}
            & \multicolumn{3}{c}{taobao} \\
            \cmidrule(lr){2-4}
            \cmidrule(lr){5-7}
            & TLL($\uparrow$) & ACC($\uparrow$) & RT(mins) ($\downarrow$)
            & TLL($\uparrow$) & ACC($\uparrow$) & RT (mins)($\downarrow$) \\
            \midrule
            \multicolumn{7}{c}{\textbf{SAHP}} \\
            \midrule
            MLE
            & \bm{$1.762_{\pm 0.006}$} & \bm{$0.468_{\pm 0.000}$} & $3.219_{\pm 0.018}$
            & \bm{$4.930_{\pm 0.118}$} & $0.436_{\pm 0.000}$      &  $1.327_{\pm 0.010}$\\
            DSM
            & $1.510_{\pm 0.033}$      & $0.440_{\pm 0.005}$      & $4.500_{\pm 0.025}$
            & $2.705_{\pm 2.328}$      & $0.376_{\pm 0.044}$      & $1.476_{\pm 0.002}$ \\
            WSM
            & $1.618_{\pm 0.363}$      & $0.423_{\pm 0.004}$      & \bm{$2.758_{\pm 0.016}$}
            & $2.688_{\pm 0.080}$      & \bm{$0.437_{\pm 0.001}$} & \bm{$1.200_{\pm 0.010}$} \\
            \midrule
            \multicolumn{7}{c}{\textbf{THP}} \\
            \midrule
            MLE
            & \bm{$1.712_{\pm 0.009}$} & \bm{$0.456_{\pm 0.000}$} & $3.410_{\pm 0.048}$
            & \bm{$4.218_{\pm 0.017}$} & $0.436_{\pm 0.001}$      & $2.830_{\pm 0.027}$ \\
            DSM
            & $1.649_{\pm 0.019}$      & $0.455_{\pm 0.001}$      & $4.118_{\pm 0.010}$
            & $3.909_{\pm 0.034}$      & $0.438_{\pm 0.002}$      & $2.853_{\pm 0.005}$ \\
            WSM
            & $1.544_{\pm 0.050}$      & $0.449_{\pm 0.003}$      & \bm{$3.227_{\pm 0.022}$}
            & $3.083_{\pm 0.070}$      & \bm{$0.436_{\pm 0.000}$} & \bm{$2.735_{\pm 0.057}$} \\
            \bottomrule
        \end{tabular}
        }
        \end{sc}
    \end{subtable}

\end{table}

\begin{table}[htbp]
    \centering
    \caption{Performance of SMASH trained on four real-world spatio-temporal datasets with three different training objectives.}
    \label{table:results_on_ST_dataset}
    \setlength{\tabcolsep}{2pt}
    \renewcommand{\arraystretch}{0.9}
    \small  

    \begin{subtable}{\linewidth}
        \centering
        \begin{sc}
        \begin{tabular}{l|ccc|ccc}
            \toprule
            & \multicolumn{3}{c}{Earthquake}
            & \multicolumn{3}{c}{COVID19} \\
            \cmidrule(lr){2-4}
            \cmidrule(lr){5-7}
            Method
            & $\textnormal{TLL}_T$($\uparrow$) & $\textnormal{TLL}_S$($\uparrow$) & RT(Mins)($\downarrow$)
            & $\textnormal{TLL}_T$($\uparrow$) & $\textnormal{TLL}_S$($\uparrow$) & RT(Mins)($\downarrow$) \\
            \midrule
            MLE
            & $0.486_{\pm 0.031}$            & $0.644_{\pm 0.048}$            & $6.000_{\pm 0.049}$
            & \bm{$2.219_{\pm 0.006}$}       & \bm{$1.360_{\pm 0.019}$}       & $10.006_{\pm 0.136}$ \\
            DSM
            & \bm{$0.494_{\pm 0.008}$}       & $0.735_{\pm 0.041}$            & \bm{$5.739_{\pm 0.048}$}
            & $1.864_{\pm 0.444}$            & $0.560_{\pm 0.630}$            & $11.178_{\pm 0.671}$  \\
            WSM
            & $0.431_{\pm 0.027}$            & \bm{$0.776_{\pm 0.049}$}       & $5.964_{\pm 0.041}$
            & $2.210_{\pm 0.009}$            & $1.153_{\pm 0.014}$            & \bm{$9.772_{\pm 0.068}$} \\
            \bottomrule
        \end{tabular}
        \end{sc}
    \end{subtable}

    \vskip 0.1in

    \begin{subtable}{\linewidth}
        \centering
        \begin{sc}
        \resizebox{\linewidth}{!}{%
        \begin{tabular}{l|ccc|cccc}
            \toprule
            & \multicolumn{3}{c}{Citibike}
            & \multicolumn{4}{c}{Football} \\
            \cmidrule(lr){2-4}
            \cmidrule(lr){5-8}
            Method
            & $\textnormal{TLL}_T$($\uparrow$) & $\textnormal{TLL}_S$($\uparrow$) & RT (Mins)($\downarrow$)
            & $\textnormal{TLL}_T$($\uparrow$) & $\textnormal{TLL}_S$($\uparrow$) & ACC($\uparrow$) & RT (Mins)($\downarrow$) \\
            \midrule
            MLE
            & $1.056_{\pm 0.007}$            & $0.777_{\pm 0.119}$            & $13.283_{\pm 0.059}$
            & \bm{$5.478_{\pm 0.006}$}       & \bm{$0.376_{\pm 0.004}$}       & \bm{$0.825_{\pm 0.001}$} & $6.011_{\pm 0.016}$ \\
            DSM
            & $0.910_{\pm 0.053}$            & $0.368_{\pm 0.100}$            & $13.283_{\pm 0.041}$ 
            & $2.160_{\pm 2.187}$            & $0.094_{\pm 0.133}$            & $0.819_{\pm 0.004}$      & \bm{$4.944_{\pm 0.122}$} \\
            WSM
            & \bm{$1.079_{\pm 0.001}$}       & \bm{$0.931_{\pm 0.001}$}       & \bm{$13.222_{\pm 0.039}$}
            & $5.426_{\pm 0.044}$            & {$0.244_{\pm 0.054}$}          & $0.823_{\pm 0.001}$      & $6.050_{\pm 0.068}$ \\
            \bottomrule
        \end{tabular}
        }
        \end{sc}
    \end{subtable}

\end{table}

\subsection{Advantage of (A)WSM over MLE}\label{sec:advantage_of_AWSM}
The key advantage of (A)WSM over MLE is its avoidance of computing intensity integrals, which can be computationally intensive for complex point process models and impact MLE accuracy. We evaluate the TLL of MLE and AWSM on  two real-world datasets as the number of integration nodes varies. As shown in \cref{fig:time_accuracy_tradeoff} , with a limited number of nodes, MLE is faster but exhibits substantial estimation errors. Increasing the number of nodes reduces the error but significantly increases computation time. In this scenario, AWSM is faster than MLE with the same accuracy, thus offering better computational efficiency.

\subsection{Ablation on Survival Classification}
We show that the survival classification term is necessary to recover the ground-truth temporal intensity even when the score is accurately learned. 
We construct a univariate temporal point process with conditional intensity given by 
\begin{equation*}
    \lambda_n(t \mid \mathcal H_{n-1})
    = c\,\mathbf 1(n=1)
    + \frac{c}{1+\beta\exp\big(c(t-t_{n-1})\big)}\,\mathbf 1(n>1),
\end{equation*}
set $c=2$, $\beta=0.02$, and generate $2000$ sequences on $(0,50)$. A THP is trained as the intensity model with AWSM and DSM, each with or without the survival classification term, for $500$ epochs. As shown in \cref{fig:intensity_w&w/o_survival}, without survival classification the learned intensity collapses to an almost horizontal curve that clearly differs from the ground truth, whereas adding the survival classification term yields an intensity that closely matches the true one, confirming the importance of this component. We further provide a heuristic theoretical discussion of this phenomenon in \cref{sec:heuristic_discussion_on_survival}.

\begin{figure}[htbp]
    \centering
    \begin{subfigure}[b]{0.48\textwidth}
        \centering
        \includegraphics[width=\linewidth]{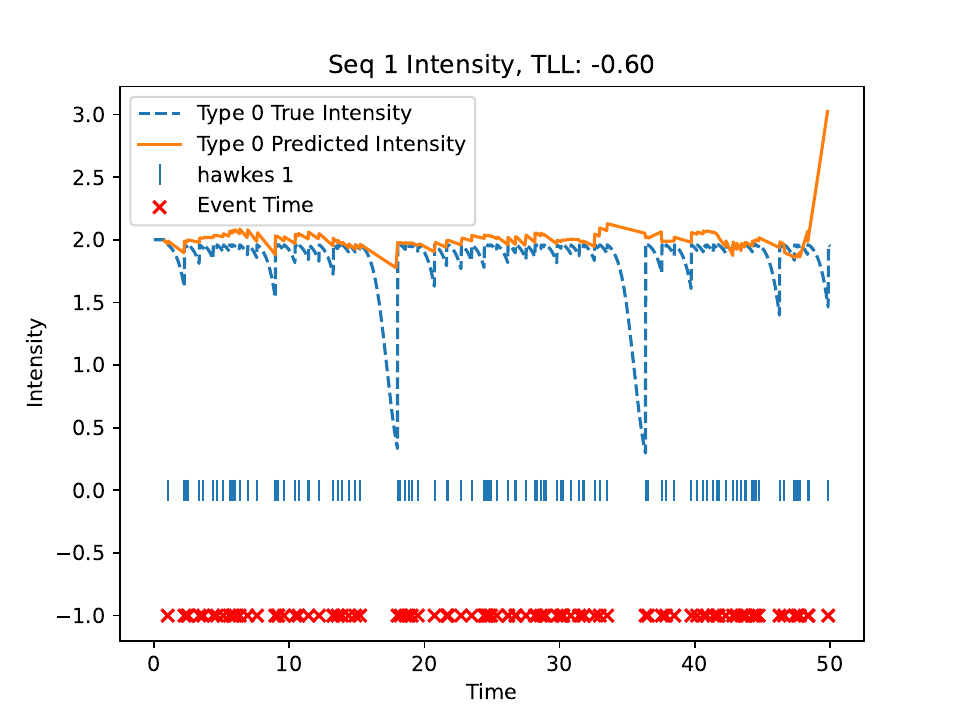}
        \caption{DSM, w/o survival classification}
        \label{fig:intensity_seq0}
    \end{subfigure}
    \hfill
    \begin{subfigure}[b]{0.48\textwidth}
        \centering
        \includegraphics[width=\linewidth]{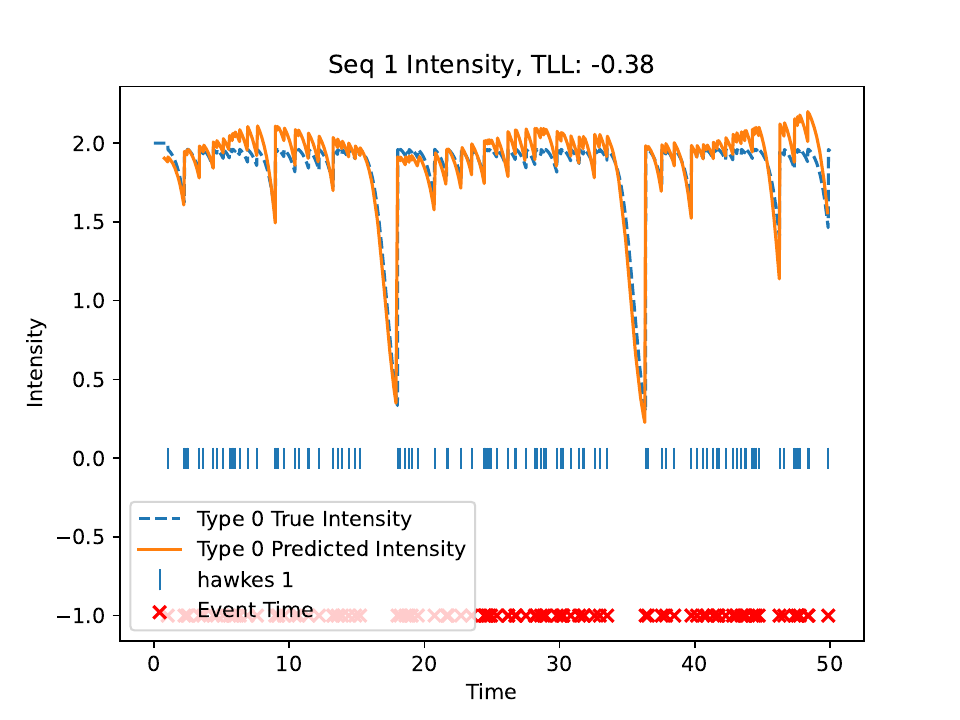}
        \caption{DSM, w/ survival classification}
        \label{fig:intensity_seq1}
    \end{subfigure}

    \vspace{0.4cm}

    \begin{subfigure}[b]{0.48\textwidth}
        \centering
        \includegraphics[width=\linewidth]{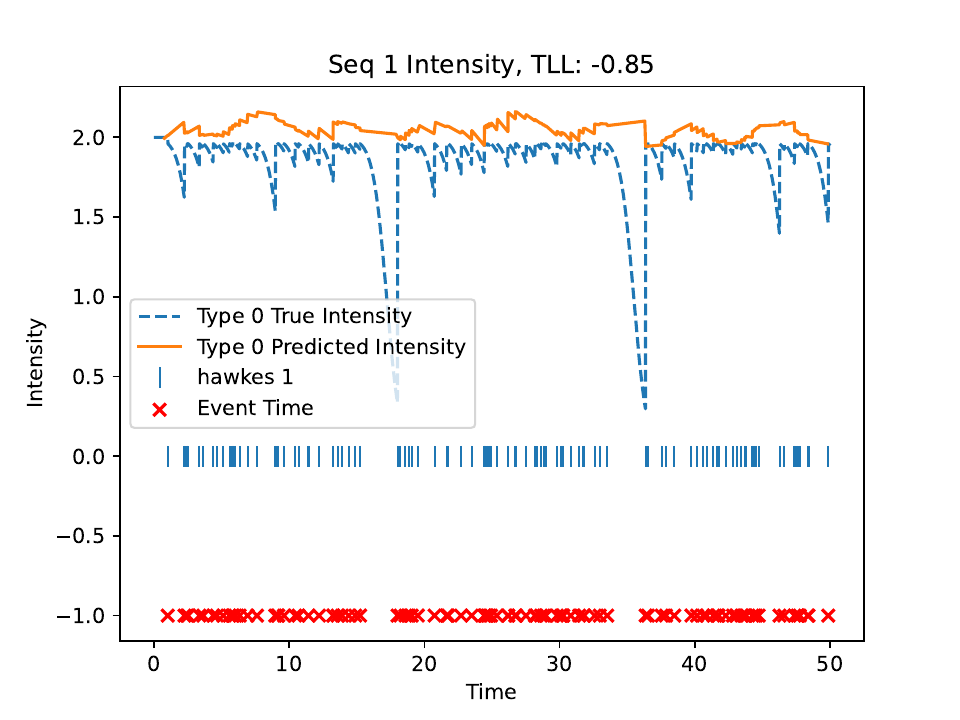}
        \caption{AWSM, w/o survival classification}
        \label{fig:intensity_seq2}
    \end{subfigure}
    \hfill
    \begin{subfigure}[b]{0.48\textwidth}
        \centering
        \includegraphics[width=\linewidth]{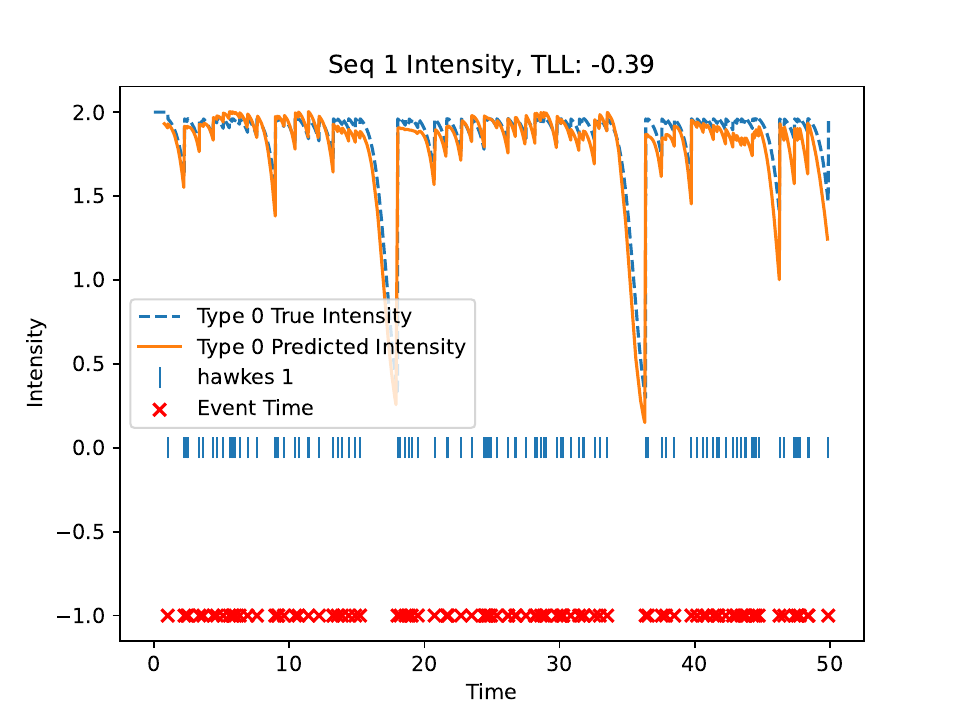}
        \caption{AWSM, w/ survival classification}
        \label{}
    \end{subfigure}
    \caption{Ground-truth and learned intensities of THP trained with AWSM and DSM objectives, with and without the survival classification term.}
    \label{fig:intensity_w&w/o_survival}
\end{figure}

\subsection{Comparison of Different Weight Functions}
Though we provide theoretical insight into the choice of an optimal weight function for our estimator, its validity still needs to be testified by experiments. 
Here, we consider the 2-variate temporal Hawkes process and compare the near-optimal weight function $h_{0,T}^{\textnormal{AWSM}}$ with two alternative choices that satisfy \cref{assum:awsm_regularity_of_weight}, namely the natural weight $h_{1,T}^{\textnormal{AWSM}}$ and the square-root weight $h_{2,T}^{\textnormal{AWSM}}$:   
\begin{align*}
        h_{1,T}^{\textnormal{AWSM}}(t_{n-1}, t_n)&:=(t_n-t_{n-1})(T-t_{n}),\\
        h_{2,T}^{\textnormal{AWSM}}(t_{n-1}, t_n)&:=\sqrt{(t_n-t_{n-1})(T-t_n)}. 
\end{align*}
All three weight functions can be applied in AWSM to recover ground-truth parameters, however with different convergence rates. 
We carry out experiments on synthetic data with the same setting as \cref{synthetic}. 
We measure their MAE for different sample sizes in \cref{fig:compare_weight} and find that $h_{0,T}^{\text{AWSM}}$ does achieve the best results among the three weight functions. 

\begin{figure*}[t]
    \begin{center}
    \adjustbox{valign=b}{
    \begin{minipage}{0.31\linewidth}
    \includegraphics[width=\columnwidth]{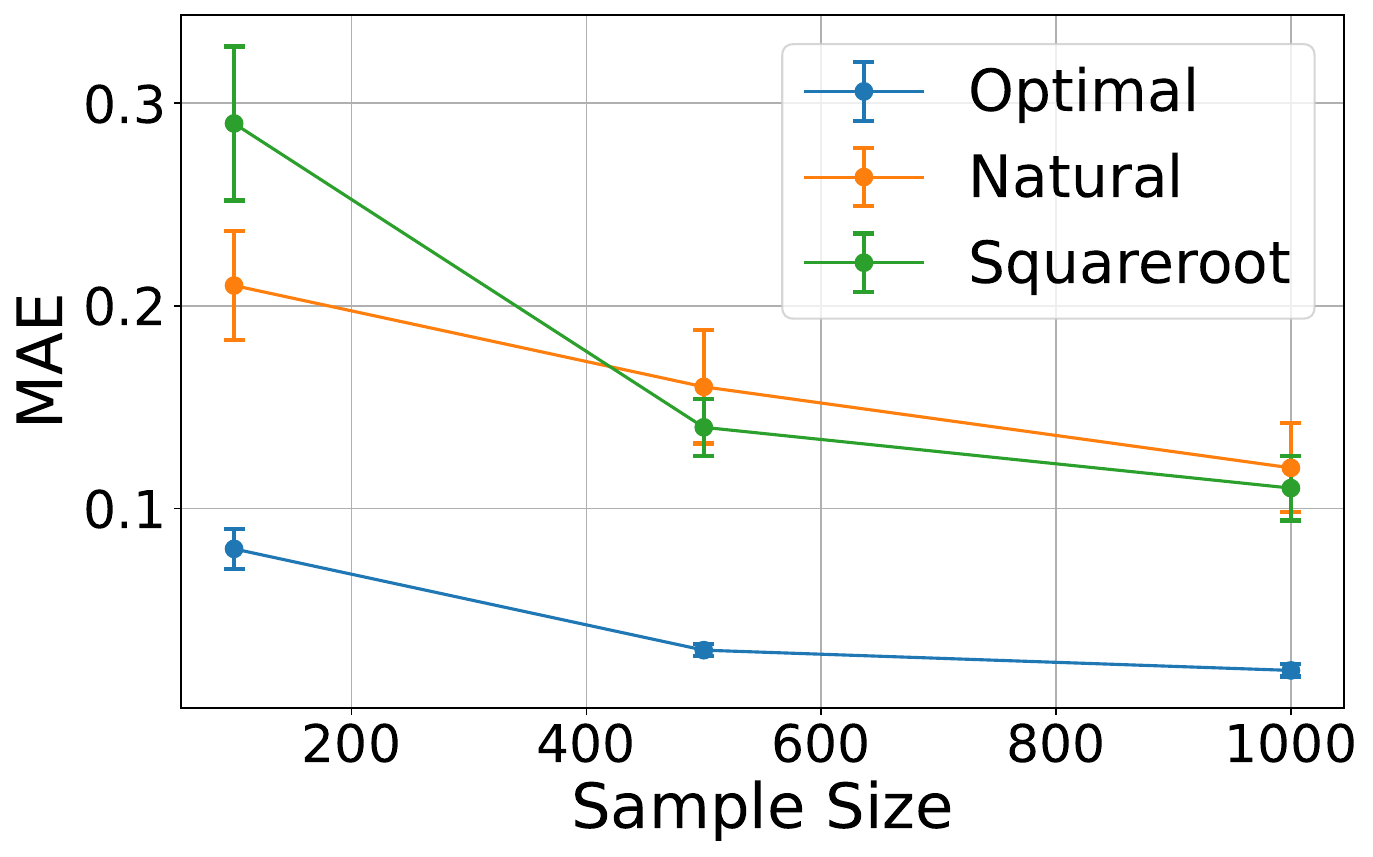}
    \subcaption{On parameter $\alpha_{11}$}
    \end{minipage}}
    \adjustbox{valign=b}{
    \begin{minipage}{0.31\linewidth}
    \includegraphics[width=\columnwidth]{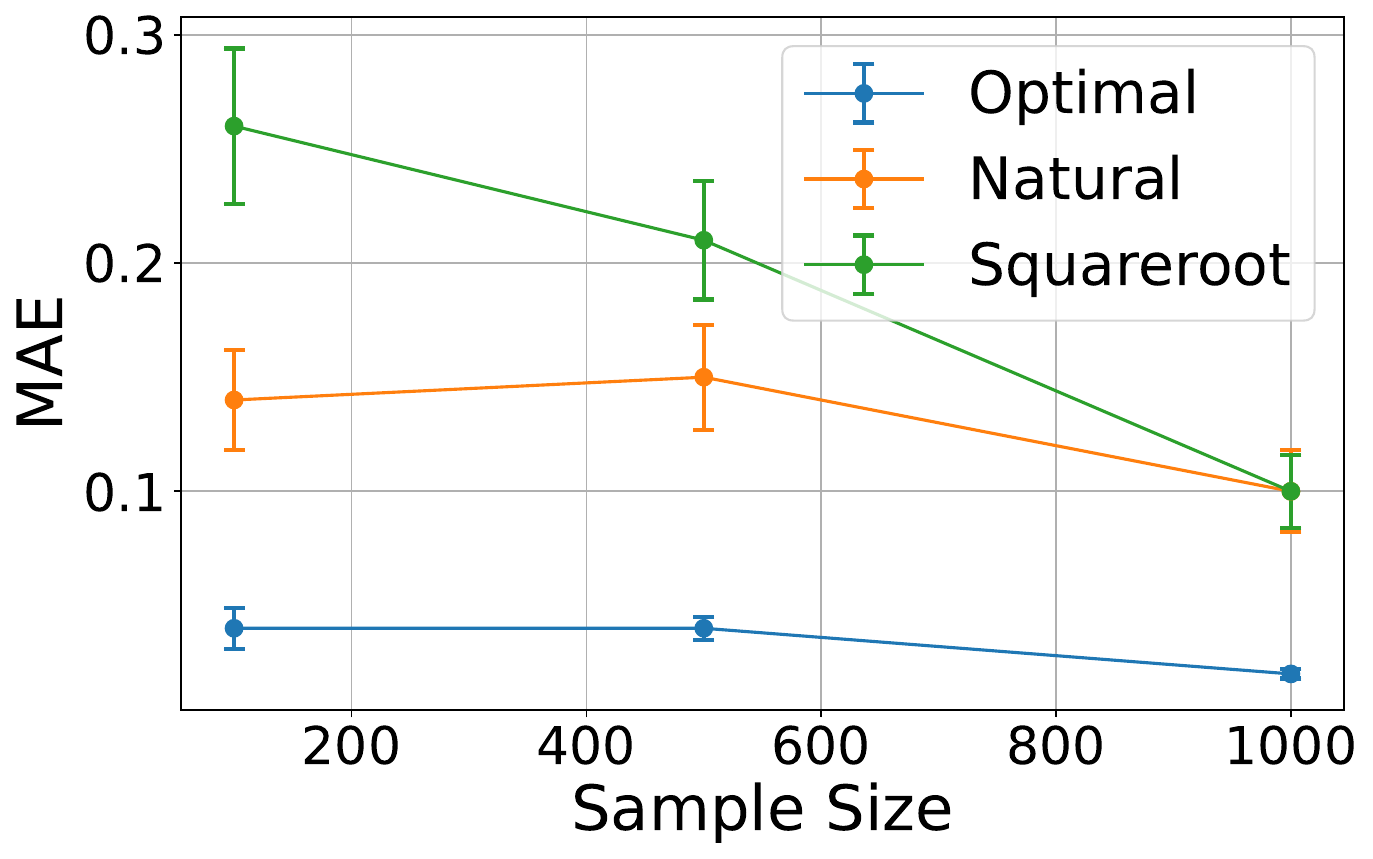}
    \subcaption{On parameter $\alpha_{21}$}
    \end{minipage}}
    \adjustbox{valign=b}{
    \begin{minipage}{0.31\linewidth}
    \includegraphics[width=\columnwidth]{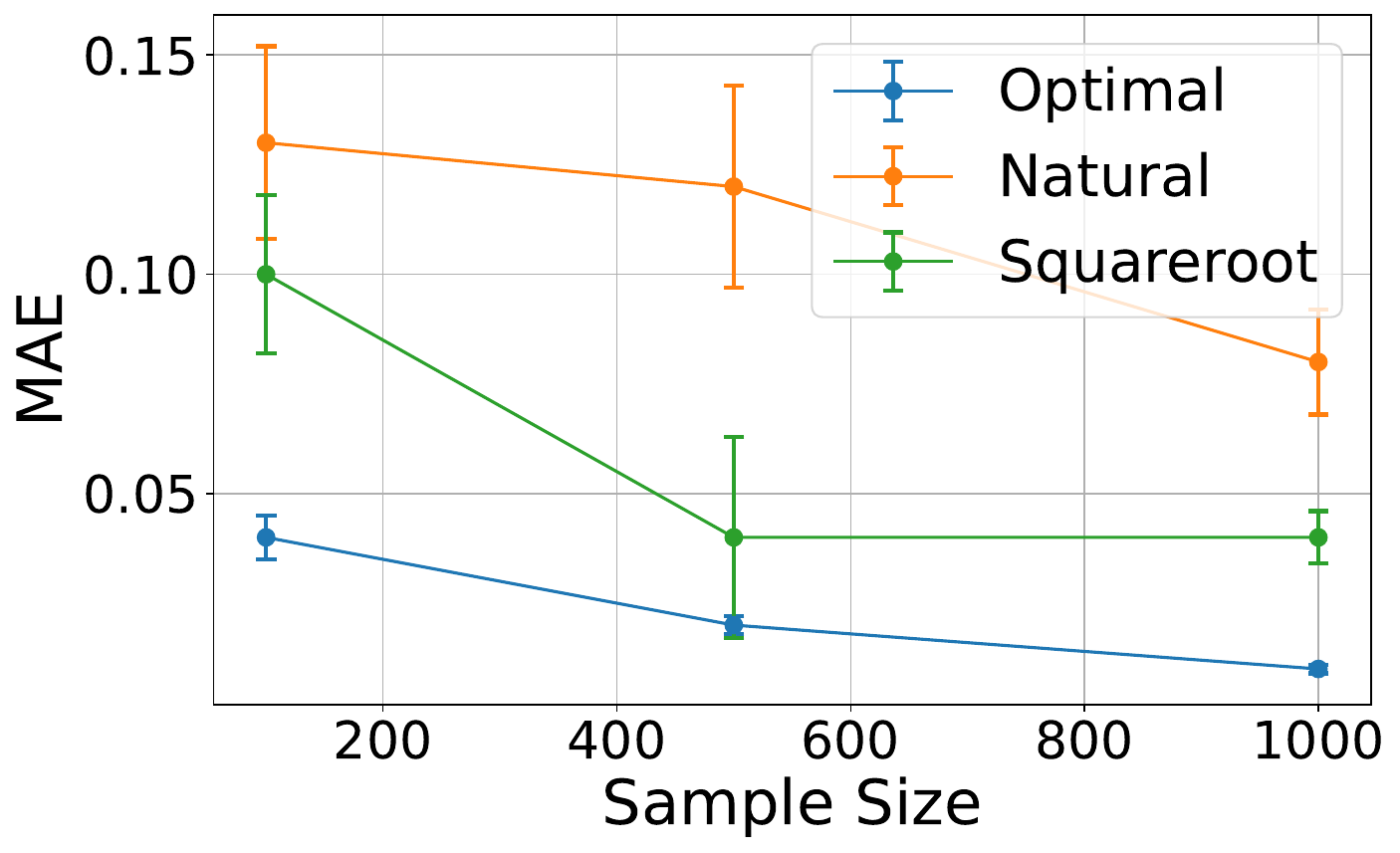}
    \subcaption{On parameter $\mu_1$}
    \end{minipage}}
    \end{center}
    \caption{MAE of parameter estimation versus sample size for three different weight functions on the 2-variate temporal Hawkes processes. Our near-optimal weight function greatly outperforms the rest two valid weight functions in all sample sizes.}
    \label{fig:compare_weight}
\end{figure*}

\label{sec:discussion}

\section{Conclusions}

Overall, this paper develops a formal framework for applying score matching to finite point processes and provides a unified theoretical and empirical study. We formally define the scores of point process via the Janossy density and introduce a weighted score matching objective with provable consistency, convergence guarantees, and a near-optimal choice of weights, together with an equivalent change-of-variables formulation. Within this framework, we further show that, due to delicate normalization issues specific to finite point processes, even a model whose score is perfectly matched to the truth need not recover the correct intensity or underlying distribution. To overcome this limitation, we propose a simple survival classification objective that enforces the correct normalization while remaining completely integration-free and applicable to general intensity-based models. Extensive experiments on synthetic and real-world temporal and spatio-temporal datasets, using both statistical and deep point process models, validate our theoretical findings: the proposed method reliably recovers ground-truth intensities and achieves performance comparable to MLE.



\newpage
\appendix
\addcontentsline{toc}{section}{Appendices}

\section{Proof of Results in \texorpdfstring{\cref{method}}{Section 3}}

\subsection{Proof of  \texorpdfstring{\cref{thm:unique_wsm}}{Theorem 9}}
\begin{proof}
If there exists another $\theta_1\neq \theta^*$ such that $\mathcal L^{\text{WSM}}_{h}(\theta_1)= 0$ then it holds,

\begin{align*}
    \mathcal L^{\text{WSM}}_{h}(\theta_1) = \sum_{N=1}^\infty \int_{V^{(N)}}j_{N,\theta^*}(\bm X_N)\sum_{n=1}^N\||\bm \psi_{n,\theta^*}(\bm X_N)-\bm \psi_{n,\theta_1}( \bm X_N)\|^2h(\bm x_n)d\bm X_N
\end{align*}

Since the square distance is non-negative and $h>0$ a.e., this implies
\begin{equation*}
    \bm \psi_{n,\theta^*}(\mathcal  X)=\bm \psi_{n,\theta_1}(\mathcal  X)\ a.s., \forall n\in \mathbb N_{+}, 
\end{equation*}
By assumption this imples $\theta^* = \theta_1$
\end{proof}
\begin{lemma}\textnormal{(Divergence Theorem)} \label{lemma:divergence_thm}
    Assume $V\subset \mathbb R^d$ is a Lipschitz domain with boundary $\partial V$. For any $f_1, f_2 \in H^1(V)$.
    \begin{equation*}
        \int_V \partial_{x_k} f_1(\bm x)f_2(\bm x)d\bm x = \int_{\partial V}f_1(\bm x) f_2(\bm x)\nu_k ds-\int_V f_1(\bm x)\partial_{x_k}f_2(\bm x) d\bm x,
    \end{equation*}
where $v_k$ is the $k$-th coordinate of the unit outward normal vector on $\partial V$.
\end{lemma}

\subsection{Proof of \texorpdfstring{\cref{thm:equivalence between EWSM and IWSM}}{Theorem 10}}
\begin{proof}

First we show that under \cref{assum:wsm_score_regularity} and \cref{assum:wsm_regularity_of_weight}, we have,
\begin{equation}\label{eq:boundedness_of_expecation}
    \sum_{n=1}^{N(\bm X)}\|\bm \psi_{n,\theta}(\bm X)\|^2h(\bm x_n), \sum_{n=1}^{N(\bm X)}\bm \psi_{n}(\bm X)\cdot \bm \psi_{n,\theta}(\bm X)h(\bm x_n) \in L^1(\mathbb P).
\end{equation}

We prove the first term is L-1 integrable, the other one follows the same derivation. This is because by Cauchy-Schwarz inequality, we have,
\begin{align*}
    \mathbb E [\sum_{n=1}^{N(\mathcal X)}\|\bm \psi_{n,\theta}(\mathcal X)\|^2h(\bm x_n)]\leq \Big(\mathbb E [\sum_{n=1}^{N(\mathcal X)} \|\psi_{n,\theta}(\mathcal X)\|^4]\Big)^{\frac{1}{2}}\Big(\mathbb E [\sum_{n=1}^{N(\mathcal X)} h^2(\bm x_n)]\Big)^{\frac{1}{2}},
\end{align*}
then by (A1) in \cref{assum:wsm_score_regularity} and (A3) in \cref{assum:wsm_regularity_of_weight}, we know RHS is bounded. Thus we verify $\sum_{n=1}^{N(\bm X)}\|\bm \psi_{n,\theta}(\bm X)\|^2h(\bm x_n) \in L^1(\mathbb P)$.

Now by \cref{eq:boundedness_of_expecation}, we can use Fubini theorem to exchange the integral and have,
\begin{equation}\label{eq:cross_term_wsm}
    \begin{aligned}
    \mathbb E[\sum_{n=1}^{N(\mathcal X)}&\bm \psi_n(\mathcal X)\cdot \bm \psi_{n, \theta}(\mathcal X) h(\bm x_n)]=\sum_{N=1}^\infty\frac{1}{N!}\int_{V^{(N)}}j_N(\bm X_N)\sum_{n=1}^N\bm \psi_n(\bm X_n)\cdot \bm \psi_{n, \theta}(\bm X_N)h(\bm x_n) d\bm X_N\\
    &=\sum_{N=1}^\infty\sum_{n=1}^N \frac{1}{N!}\int_{V^{(N-1)}}d\bm X_{N,-n}\int_{V}j_{N}(\bm X_N)\bm \psi_n(\bm X_N)\cdot \bm \psi_{n, \theta}(\bm X_N)h(\bm x_n)d\bm x_n.
\end{aligned}
\end{equation}
We further evaluate the inner integral by,
\begin{align*}
    &\int_{V}j_N(\bm X_N)\bm \psi_n(\bm X_N)\cdot \bm \psi_{n, \theta}(\bm X_N)h(\bm x_n) d\bm x_n\\&=\int_{V}\nabla_{\bm x_n}j_N(\bm X_N)\cdot \bm \psi_{n, \theta}(\bm X_N)h(\bm x_n)d\bm x_n\\
    &=\int_{\partial V}j_N(\bm X_N)\bm \psi_{n, \theta}(\bm X_N)\cdot \bm \nu_dh(\bm x_n)d\bm x_n-\int_{V}j_N(\bm X_N)\nabla_{\bm x_n} \cdot\big( \bm \psi_{n, \theta}(\bm X_N)h(\bm x_n)\big)d\bm x_n\\
    &=-\int_{V}j_N(\bm X_N) \Big[\text{Tr}(\nabla_{\bm x_n} \bm\psi_{n,\theta}(\bm X_N))h(\bm x_n)+\bm \psi_{n,\theta}(\bm X_N)\cdot \nabla_{\bm x_n}h(\bm x_n)\Big]d\bm x_n
\end{align*}
where the second equation is by divergence theorem and $\nabla_{\bm x_n}\cdot$ represents taking divergence w.r.t. $\bm x_n$. The first term in RHS of the second equation vanishes because of \cref{assum:wsm_regularity_of_weight}. Therefore we have,
\begin{align*}
    &-\mathbb E[\sum_{n=1}^{N(\mathcal X)}\bm \psi_n(\mathcal X)\cdot \bm \psi_{n, \theta}(\mathcal X) h(\bm x_n)] \\&=\sum_{N=1}^\infty\sum_{n=1}^N \frac{1}{N!}\int_{V^{(N-1)}}d\bm X_{N,-n}\int_{V}j_N(\bm X_N) \Big[\text{Tr}(\nabla_{\bm x_n} \bm\psi_{n,\theta}(\bm X_N))h(\bm x_n)+\bm \psi_{n,\theta}(\bm X_N)\cdot \nabla_{\bm x_n}h(\bm x_n)\Big]d\bm x_n\\
    &=\sum_{N=1}^\infty\frac{1}{N!}\int_{V^{(N)}}j_N(\bm X_N)\sum_{n=1}^N \Big[\text{Tr}(\nabla_{\bm x_n} \bm\psi_{n,\theta}(\bm X_N))h(\bm x_n)+\bm \psi_{n,\theta}(\bm X_N)\cdot \nabla_{\bm x_n}h(\bm x_n)\Big]d\bm X_N\\
    &=\mathbb E[\sum_{n=1}^{N(\mathcal X)}\text{Tr}(\nabla_{\bm x_n} \bm\psi_{n,\theta}(\bm X_N))h(\bm x_n)+\bm \psi_{n,\theta}(\bm X_N)\cdot \nabla_{\bm x_n}h(\bm x_n)]
\end{align*}

The second equation is by Fubini Theorem, this suffices to prove the theorem.

\end{proof}

Now we provide two parametric examples showing that many common parametric point processes satisfy \cref{assum:wsm_score_regularity} and \cref{assum:wsm_regularity_of_weight}.

\begin{example}\textnormal{(Inhomogeneous Poisson Process)}
    Consider a temporal inhomogeneous Poisson process on $(-\pi,\pi)$ with $\lambda(x)=\exp(\theta \sin (x))$. For compact parameter space, we have $\theta\lesssim \textnormal{diam}(\Theta)$. Then by Campbell's theorem, 
        \begin{equation*}
        \mathbb E[N(\mathcal X)]=\int_{(-\pi,\pi)}\exp(\theta^* sinx)dx \leq 2\pi \exp(\textnormal{diam}(\Theta))<\infty.\\
    \end{equation*}

    Therefore, to satisfy \cref{assum:wsm_score_regularity}, a sufficient condition is score are bounded. This holds in this example since,
    \begin{equation*}
        |\psi_{n,\theta}(\bm X)| = |\theta\cos x_n|\leq \textnormal{diam}(\Theta).
    \end{equation*}
    Therefore, the model score satisfies the regularity. Since $ \mathbb E[N(\mathcal X)]<\infty$, we know the distance function $h(x) = \min\{|x+\pi|, |x-\pi|\}$ will satisfy \cref{assum:wsm_regularity_of_weight}.
\end{example}

\begin{example}\textnormal{(Truncated Gaussian Point Process)}\label{ex:truncated_gaussian}
    Consider a d-dimensional finite point process on $V\subset \mathbb R^d$ with Janossy density,
    \begin{equation*}
        j_{N,\theta}(\bm X) = \frac{1}{(2\pi|\bm \Sigma|)^{\frac{dN}{2}}}\exp\{-\frac{1}{2}\sum_{n=1}^{N(\bm X)}(\bm x_n-\bm \mu)^\top\bm \Sigma^{-1}(\bm x_n-\bm \mu)\}Z(\bm \mu, \Sigma, V),
    \end{equation*}
    where $Z(\bm \mu, \bm \Sigma, V)<\infty$ is the normalizing constant. This is also an exponential family as we introduced in  \cref{ex:exponential}. Model score is $\bm \psi_{n,\theta}(\bm X) = -\bm \Sigma^{-1}(\bm x_n-\bm \mu)$. Notice that,
    \begin{align*}
        \mathbb E [\sum_{n=1}^{N(\mathcal X)} \|\psi_{n,\theta}(\mathcal X)\|^4]&= \mathbb E[\sum_{n=1}^{N(\mathcal X)}\Big((\bm x_n-\bm \mu)^\top\Sigma^{-2}(\bm x_n-\bm \mu)\Big)^2]\\
        &=\sum_{N=1}^\infty\frac{1}{N!}\int_{V^{(N)}}j_N(\bm X_N)\sum_{n=1}^N \Big((\bm x_n-\bm \mu)^\top\Sigma^{-2}(\bm x_n-\bm \mu)\Big)^2d\bm X_N\\
        &=Z(\bm \mu^*, \bm \Sigma^*, V)\sum_{N=1}^{\infty}\frac{\textnormal {Pr}(V)^{N}}{N!}\sum_{n=1}^N\mathbb E_{\bm x\sim \textnormal{Trunc}\mathcal  N(\bm \mu^*, \Sigma^*)} \Big((\bm x-\bm \mu)^\top\Sigma^{-2}(\bm x-\bm \mu)\Big)^2\\
        &\lesssim \sum_{N=1}^\infty \frac{N\textnormal{Pr(V)}^N}{N!} <\infty,
    \end{align*}

where $\textnormal{Pr}(V)$ is the probability of the sample located in $V$ under $\mathcal N(\bm \mu^*, \Sigma^*)$. For the last inequality, we need to assume that $\Theta$ is compact so that by the property of Gaussian, $\mathbb E_{\bm x\sim \textnormal{Trunc}\mathcal N (\bm \mu^*, \bm \Sigma^*)} \Big((\bm x-\bm \mu)^\top\Sigma^{-2}(\bm x-\bm \mu)\Big)^2$ can be upper bounded by some absolute constant. So far, we have shown that \cref{assum:wsm_score_regularity} is satisfied. Following the same derivation, one can show that with bounded distance function $h^{\textnormal{WSM}}_0(\bm x)=\text{dist}(\bm x, \partial V)$, \cref{assum:wsm_regularity_of_weight} is also satisfied.
\end{example}

\subsection{Remedy of non-identifying issue of WSM}\label{sec:WSM_remedy}
Let's assume we have a nonparametric model denoted as $\tilde j_N(\bm X)$ for the janossy density. Denote its score as $\hat {\bm \psi}_n$. For the case where the score of the Janossy density is matched, i.e.
\begin{equation*}
    \hat {\bm \psi}_n(\bm X)=\bm \psi_n(\bm X),\forall n \in [N], \forall \bm X\in V^{(N)}
\end{equation*}
This means that the modeled janossy density satisfies,
\begin{equation*}
    \tilde j_N(\bm X) = c_N j_N(\bm X),
\end{equation*}
where we utilize the fact that $\tilde j_N$ and $j_N$ should be symmetric to the $N$ points, and here $c_N$ is a constant that is not necessarily one. If we integrate both sides, we have,
\begin{equation*}
    \int_{V^{(N)}}\tilde j_N(\bm X)d\bm X = c_N J_N\Rightarrow c_N=\frac{ \int_{V^{(N)}}\tilde j_N(\bm X)d\bm X}{J_N},
\end{equation*}
where we abuse the notation $J_N$ for $J_N(V^{(N)})$. Then, we can estimate $c_N$ by 
\begin{equation*}
    \hat c_N:=\frac{\int_{V^{(N)}}\tilde j_N(\bm X)d\bm X}{\frac{\#\textnormal{ of trajectories with N points}}{\textnormal{total }\# \textnormal{ of  trajectories}}}.
\end{equation*}
Then $\frac{\tilde j_N}{\hat c_N} $ should be a consistent estimator for $j_N$.

\section{Proof of Results in \texorpdfstring{\cref{method2}}{Section 4}}

\subsection{Proof of \texorpdfstring{\cref{thm:unique_awsm}}{Theorem 13}}\label{sec:proof_of_awsm_recover}
\begin{proof}
By definition,
\begin{equation*}
  \mathcal L_{h_T,h_S}^{\textnormal{AWSM}}(\theta)
  = \frac{1}{2}\,\mathbb E\left[
    \sum_{n=1}^{N(\mathcal X)}
      \left\{
        (\psi_{T,n}-\psi_{T,n,\theta})^2 h_T
        + \|\bm\psi_{S,n}-\bm\psi_{S,n,\theta}\|^2 h_S
      \right\}\right]\ge 0,
\end{equation*}
where $h_T=h_T(t_{n-1},t_n)$ and $h_S=h_S(\bm s_n)$.
For $\theta=\theta^*$ the scores coincide with the truth, so 
$\mathcal L_{h_T,h_S}^{\textnormal{AWSM}}(\theta^*)=0$, hence $\theta^*$ is a minimizer.
Let $\theta_1$ be any minimizer. Then 
$\mathcal L_{h_T,h_S}^{\textnormal{AWSM}}(\theta_1)=0$, so the nonnegative random variable
$Z(\theta_1):=\sum_{n=1}^{N(\mathcal X)}\{(\psi_{T,n}-\psi_{T,n,\theta_1})^2 h_T
+\|\bm\psi_{S,n}-\bm\psi_{S,n,\theta_1}\|^2 h_S\}$ satisfies $\mathbb E[Z(\theta_1)]=0$ and
hence $Z(\theta_1)=0$ a.s.
Since $h_T>0$ and $h_S>0$ a.e., each squared term must vanish a.s., so
$\psi_{T,n}=\psi_{T,n,\theta_1}$ and $\bm\psi_{S,n}=\bm\psi_{S,n,\theta_1}$ a.s.\ for all $n$.
By the identifiability assumption this implies $\theta_1=\theta^*$, so $\theta^*$ is the unique minimizer.
\end{proof}


\subsection{Proof of \texorpdfstring{\cref{thm:tractable EWSM}}{Theorem 14}}

\begin{lemma}\label{Lemma:exchange summation}
    For a regular and finite spatio-temporal point process on $V=(0,T)\times S$ with the underlying probability measure $\mathbb P$. Recall the notation in \cref{prop:existence_of_conditional_densityd}, denote $p_n(\bm x_1,\ldots, \bm x_n)=p_n(t_n,\bm s_n|\mathcal H_{n-1})\times\ldots \times p_1(t_1,\bm s_1)$. Suppose $g_n: D_n^T\times S^n\rightarrow \mathbb R$ and $g_n(\bm X):=g_n(t_1,\bm s_1,\ldots, t_n,\bm s_n)$,
    If $\sum_{n=1}^{N(\bm X)}g_n(\bm X) \in L^1(\mathbb P)$,
    then
    \begin{equation*}
        \mathbb E\Big[\sum_{n=1}^{N(\mathcal X)}g_n(\mathcal X)\Big]=\sum_{n=1}^\infty \int_{D_n^T\times S^{(n)}}p_n(\bm X_n) g_n(\bm X_n) d\bm X_n
    \end{equation*}
\end{lemma}
\begin{proof}
    Notice that,
    \begin{align*}
        \mathbb E\Big[\sum_{n=1}^{N(\mathcal X)}g_n(\mathcal X)\Big]&=\sum_{N=1}^\infty\frac{1}{N!}\int_{V^{(N)}} j_N(\bm X_N)\sum_{n=1}^Ng_n(\bm X_N)d \bm X_N\\
        &=\sum_{N=1}^\infty\int_{D_N^T\times S^{(N)}}j_N(\bm X_N)\sum_{n=1}^Ng_n(\bm X_n)d\bm X_N\\
        =\sum_{N=1}^\infty\sum_{n=1}^N \int_{D_n^{T}\times S^{(n)}}&p_n(\bm X_n)g_n(\bm X_n)\int_{D_{N-n}^{t_n,T}\times S^{(N-n)}}\big[p(\bm X_{n+1:N}|\mathcal H_{n})G_{N+1}(T|\bm X_N)d\bm X_{n+1:N}\big]d\bm X_n\\
        &=\sum_{N=1}^\infty\sum_{n=1}^N \int _{D_n^T\times S^{(n)}}p_n(\bm X_n)g_n(\bm X_n)\mathbb P(N(\mathcal X)=N|\mathcal H_n) d\bm X_n\\
        &=\sum_{n=1}^\infty\int_{D_n^T\times S^{(n)}}p_n(\bm X_n)g_n(\bm X_n)d\bm X_n\sum_{N=n}^\infty \mathbb P(N(\mathcal X)=N|\mathcal H_n)\\
        &=\sum_{n=1}^\infty\int_{D_n^T\times S^{(n)}}p_n(\bm X_n)g_n(\bm X_n)d\bm X_n,
    \end{align*}
where the second equation is because a $\sum_{n=1}^N g_n(\bm X_N)$ function of chronologically ordered points and (when perceived as a function on $V^\cup$) is symmetric w.r.t. all $N$ points. The third and the fifth equation is by Fubini theorem. $D_{N-n}^{t_n,T}$ represents the following area $t_n< t_{n+1}\ldots<t_N<T$. The sixth equation is by observing that $\sum_{N=n}^\infty\mathbb P(N(\mathcal X)=N|\mathcal H_n)=1$.

\end{proof}

\begin{proof}\textbf{of \cref{thm:tractable EWSM}}

First we show that under \cref{assum:AWSM_score_regularity} and \cref{assum:awsm_regularity_of_weight}, we have,
\begin{align*}
    &\sum_{n=1}^{N(\bm X)}\psi^2_{T,n,\theta}(t_n|\mathcal H_{n-1})h_T(t_{n-1}, t_n), \sum_{n=1}^{N(\bm X)}\psi_{T,n}(t_n|\mathcal H_{n-1}) \psi_{T,n,\theta}(t_n|\mathcal H_{n-1})h_T(t_{n-1}, t_n)\in L^1(\mathbb P)\\
    &\sum_{n=1}^{N(\bm X)}\bm \psi^2_{S,n,\theta}(\bm s_n|\mathcal H_{t_n})h_S(\bm s_n), \sum_{n=1}^{N(\bm X)}\bm \psi_{S,n}(\bm s_n|\mathcal H_{t_n})\cdot \bm \psi_{S,n,\theta}(\bm s_n|\mathcal H_{t_n})h_S(\bm s_n)\in L^1(\mathbb P)
\end{align*}

Again we prove the first term is L1 integrable, the other one follows the same derivation. This is because by Cauchy-Schwarz inequality, we have,
\begin{align*}
    \mathbb E [\sum_{n=1}^{N(\mathcal X)}\psi^2_{T,n,\theta}(t_n|\mathcal H_{n-1})h_T(t_{n-1}, t_n)]\leq \Big(\mathbb E [\sum_{n=1}^{N(\mathcal X)} \psi^4_{T,n,\theta}(t_n|\mathcal H_{n-1})]\Big)^{\frac{1}{2}}\Big(\mathbb E [\sum_{n=1}^{N(\mathcal X)} h^2_T(t_{n-1}, t_n)]\Big)^{\frac{1}{2}},
\end{align*}
then by (A1) in \cref{assum:AWSM_score_regularity} and (A3) in \cref{assum:awsm_regularity_of_weight}, we know RHS is bounded. Thus we verify $\sum_{n=1}^{N(\bm X)}\psi^2_{T,n,\theta}(t_n|\mathcal H_{n-1})h_T(t_{n-1}, t_n) \in L^1(\mathbb P)$.

\begin{align*}
        &\mathbb E\Big\{ \sum_{n=1}^{N(\mathcal X)}\psi_{T,n}(t_n|\mathcal H_{n-1})\psi_{T,n,\theta}(t_n|\mathcal H_{n-1})h_T(t_{n-1}, t_n)\Big\}
        \\&=\sum_{n=1}^\infty \int_{D_n^T\times S^{(n)}} p_n(\bm X_n)\psi_{T,n}(t_n|\mathcal H_{n-1})\psi_{T,n,\theta}(t_n|\mathcal H_{n-1})h_T(t_{n-1}, t_n)d\bm X_n\\
        &=\sum_{n=1}^\infty \int_{D_n^T\times S^{(n-1)}} p_{n-1}(\bm X_{n-1})p_n(t_n|\mathcal H_{n-1})
        \psi_{T,n}(t_n|\mathcal H_{n-1})\psi_{T,n,\theta}(t_n|\mathcal H_{n-1})h_T(t_{n-1}, t_n)d(\bm X_{n-1}, t_n)\\
        &=\sum_{n=1}^\infty \int_{D_{n-1}^T\times S^{(n-1)}} p_{n-1}(\bm X_{n-1})d\bm X_{n-1}\int_{t_{n-1}}^T\partial_{t_n}p_n(t_n|\mathcal H_{n-1})\psi_{T,n,\theta}(t_n|\mathcal H_{n-1})h_T(t_{n-1}, t_n)dt_n 
    \end{align*}

By \cref{assum:AWSM_score_regularity} and \cref{lemma:divergence_thm}, we have,
\begin{align*}
    &\int_{t_{n-1}}^T\partial_{t_n}p_n(t_n|\mathcal H_{n-1})\psi_{T,n,\theta}(t_n|\mathcal H_{n-1})h_T(t_{n-1}, t_n)dt_n\\&=p_n(t_n|\mathcal H_{n-1})\psi_{T,n,\theta}(t_n|\mathcal H_{n-1})h_T(t_{n-1}, t_n)\Big|_{t_{n-1}}^T- \int_{t_{n-1}}^Tp_n(t_n|\mathcal H_{n-1})\partial_{t_n}\big[\psi_{T,n,\theta}(t_n|\mathcal H_{n-1})h_T(t_{n-1}, t_n)\big]dt_n\\
    &=- \int_{t_{n-1}}^Tp_n(t_n|\mathcal H_{n-1})\partial_{t_n}\big[\psi_{T,n,\theta}(t_n|\mathcal H_{n-1})h_T(t_{n-1}, t_n)\big]dt_n,
\end{align*}
where the second equation is due to the assumption on weight function. Plug this term back,
\begin{align*}
    &-\mathbb E\Big\{ \sum_{n=1}^{N(\mathcal X)}\psi_{T,n}(t_n|\mathcal H_{n-1})\psi_{T,n,\theta}(t_n|\mathcal H_{n-1})h_T(t_{n-1}, t_n)\Big\}
    \\
    &=\sum_{n=1}^\infty \int_{D_n^T\times S^{n-1}}p_{n-1}(\bm X_{n-1})p_n(t_n|\mathcal H_{n-1})\partial t_n\big[\psi_{T,n,\theta}(t_n|\mathcal H_{n-1})h_T(t_{n-1}, t_n)]d(\bm X_{n-1}, t_n)\\
    &=\sum_{n=1}^\infty \int_{D_n^T\times S^{n}}p_{n}(\bm X_{n})\partial t_n\big[\psi_{T,n,\theta}(t_n|\mathcal H_{n-1})h_T(t_{n-1}, t_n)]d\bm X_n\\
    &=\mathbb E\Bigg\{\sum_{n=1}^{N(\mathcal X)}\Big[\partial_{t_n}\psi_{T,n,\theta}(t_n|\mathcal H_{n-1})h_T(t_{n-1}, t_n)+\psi_{T,n,\theta}(t_n|\mathcal H_{n-1})\partial_{t_n}h_T(t_{n-1}, t_n)\Big]\Bigg\}
\end{align*}

Combined with $\sum_{n=1}^{N(\bm X)}\psi^2_{T,n,\theta}(t_n|\mathcal H_{n-1})h_T(t_{n-1}, t_n) \in L^1(\mathbb P)$, we have proved,
\begin{align*}
    \frac{1}{2}\mathbb E\Bigg\{\sum_{n=1}^{N(\mathcal X)}&\Big (\psi_{T,n}(t_n|\mathcal H_{{n-1}})-\psi_{T,n,\theta}(t_n|\mathcal H_{n-1})\Big)^2h_{T}(t_{n-1},t_n)\Bigg\}
    =\mathcal J^{\textnormal{AWSM, T}}_{h_T}(\theta) + \textnormal{const}
\end{align*}

By the same proof, we also have,
\begin{align*}
    \frac{1}{2}\mathbb E\Bigg\{\sum_{n=1}^{N(\mathcal X)}\|\bm \psi_{S,n}(\bm s_n|\mathcal F_{t_n})-\bm \psi_{S,n, \theta}(\bm s_n|\mathcal F_{t_n})\|^2h_{S}(\bm s_n)\Bigg\} = \mathcal J^{\textnormal{AWSM, S}}_{h_S}(\theta) + \textnormal{const}
\end{align*}

\end{proof}

\section{Proof of Results in \texorpdfstring{\cref{sec:Theory}}{Section 5}}\label{appendix:proof_of_asymptotic}

In this section, we will state some classical lemmas to establish the statistical property of our estimator. When stating those lemmas, we always assume there is an underlying data generating probability space $(\mathcal Y, \mathbb P)$. Our sample is generated i.i.d from this space denoted as $(Y_1,\ldots, Y_m)$ with $m$ being the sample size. When we apply those lemmas to our problem, $\mathbb P$ will be the measure that generate the ground-truth point process as introduced in \cref{def:data_generating_measure}.  We will use a \textit{compact} set $\Theta\subset \mathbb R^p$ to represent the parameter space and the ground-truth parameter $\theta^*\in\text{Interior}(\Theta)$.

\begin{lemma}\label{lemma:uniform_lln}
    \textnormal{(Uniform LLN \citep{newey1994large})} Suppose $\bm a_{\theta}(y): \Theta\times \mathcal Y \times \rightarrow \mathbb R^{p\times p}$ is continuous at each $\theta\in\Theta$ with probability one and $\Theta$ being compact. If the following holds,
    \begin{equation*}
        \mathbb E\big[\sup_{\theta\in\Theta}\|\bm a_{\theta}(y)\|\big]<\infty.
    \end{equation*}
    Then it holds that,
    \begin{equation*}
        \sup_{\theta\in \Theta}\|\frac{1}{m}\sum_{i=1}^m \bm a_{\theta}(Y_i)-\mathbb E[\bm a_{\theta}(Y_i)]\|\xrightarrow[]{\mathbb P}0,
    \end{equation*}
    where all $\|\cdot \|$ above is the Frobenius norm.
\end{lemma}

For the following two lemmas, we will denote the criterion function as $q_{\theta}(y):\Theta\times \mathcal Y \rightarrow \mathbb R$. The empirical loss is $Q_m(\theta):=\frac{1}{m}\sum_{i=1}^mq_{\theta}(Y_i)$ and population loss $Q(\theta):=\mathbb E[q_{\theta}(Y)]$.

\begin{lemma}\textnormal{(Theorem 5.7 in \citet{van2000asymptotic})}\label{lemma:wald's_consistency}
If for a $\theta^*\in \Theta$ and every $\varepsilon>0$,
    \begin{align*}
        \sup_{\theta\in \Theta}|Q_m(\theta)-Q(\theta)|&\xrightarrow{\mathbb p} 0,\\
        \sup_{\theta:\|\theta-\theta^*\|\geq \varepsilon}Q(\theta)&<Q(\theta^*),
    \end{align*}
    then any sequence of estimators $\hat {\theta}_m$ with $Q_m(\hat \theta_m)>Q_m(\theta^*)-o_p(1)$ converges to $\theta^*$ in probability.
\end{lemma}

\begin{lemma}\textnormal{(Theorem 3.1 in \citet{newey1994large})}\label{lemma:normality}
    Suppose a sequence $\hat \theta_m \xrightarrow{\mathbb p} \theta^*$ and satisfies (i) $\theta^*\in \textnormal{interior}(\Theta)$; (ii) $Q_m(\theta)$ is twice twice continuously differentiable in a neighborhood $B(\theta^*,r)$ of $\theta^*$ (w.p.1); (iii) $\sqrt m\nabla _{\theta}Q_m(\theta^*)\xrightarrow{w}\mathcal N(0, \Sigma)$; (iv) there is $V(\theta)\in \mathbb R^{p\times p}$ that is continuous at $\theta^*$ and $\sup_{\|\theta-\theta^*\|<\varepsilon}\|\nabla^2_{\theta}Q_m(\theta)-V(\theta)\|\xrightarrow[]{\mathbb p} 0$; (v) $V=V(\theta^*)$ is nonsingular. Then we have,
    \begin{equation*}
        \sqrt n (\hat \theta_m - \theta^*)\xrightarrow[]{w}\mathcal N(0, V^{-1}\Sigma V^{-1}).
    \end{equation*}
\end{lemma}

\subsection{Proof of \texorpdfstring{\cref{thm:WSM_consistency}}{Theorem 15}}

Now we will be using those lemmas to our problem with $\mathcal Y = V^{\cup}$. We will give the general assumptions for the asymptotic property to hold. For a certain parametric point process model, one can choose a common weight function and verify those assumptions to establish the asymptotic property.

\begin{assumption} \label{assum:wsm_asymp}
\hspace*{\fill}
    \begin{enumerate}
        \item $q^{\textnormal{WSM}}_{h,\theta}(\bm X)$ is twice continuously differentiable at each $\theta\in \Theta$ a.s.
        \item It holds that for some $\varepsilon >0$,
        \begin{align*}
            &\mathbb E\Big[\sup_{\theta\in \Theta}|q_{h,\theta}^{\textnormal{WSM}}(\mathcal X)|\Big]<\infty,
            \mathbb E\big[\sup_{\|\theta-\theta^*\|<\varepsilon}\|\nabla_{\theta}q^{\textnormal{WSM}}_{h,\theta}(\mathcal X)\|\Big]<\infty,\\
            &\mathbb E\big[\sup_{\|\theta-\theta^*\|<\varepsilon}\|\nabla^2_{\theta}q^{\textnormal{WSM}}_{h,\theta}(\mathcal X)\|\Big]<\infty.
        \end{align*}
        \item $\mathbb E\Big\|\nabla q_{h,\theta^*}^{\textnormal{WSM}}(\mathcal X)\nabla q_{h,\theta^*}^{\textnormal{WSM}}(\mathcal X)^\top\Big\|<\infty$.
        \item $\nabla_{\theta}^2\mathcal J^{\textnormal{WSM}}_{h}(\theta^*)$ is nonsingular. 
    \end{enumerate}
\end{assumption}

\begin{proof}
By the condition 1 and the first term in the condition 2, we can apply the ULLN (to $1\times 1$ matrix) and get that 
\begin{equation*}
    \sup_{\theta\in \Theta} |\frac{1}{m}\sum_{i=1}^m q_{h,\theta}^{\textnormal{WSM}}(\bm X^{(i)})-\mathcal J_{h}^{\textnormal{WSM}}(\theta)|\xrightarrow[]{\mathbb p}0.
\end{equation*}

Also by the first condition and the first term in the second condition, we can use DCT and show that $\mathcal J_{h}^{\textnormal{WSM}}(\theta)$ is continuous. Thus on the closed set, let 
    \begin{equation*}
        \theta_{\varepsilon} = \mathop{\textnormal{argmin}}_{\theta: \|\theta-\theta^*\|\geq \varepsilon} \mathcal J_{h}^{\textnormal{WSM}}(\theta).
    \end{equation*}
    By \cref{thm:unique_wsm}, $\theta^*$ uniquely minimzes  $\mathcal L_{h}^{\textnormal{WSM}}(\theta)$ thus also uniquely minimzes $\mathcal J_{h}^{\textnormal{WSM}}(\theta)$ due to \cref{thm:equivalence between EWSM and IWSM}. Therefore, $\mathcal J_{h}^{\textnormal{WSM}}(\theta_{\varepsilon})>\mathcal J_{h}^{\textnormal{WSM}}(\theta^*)$. Since $\hat \theta_m$ minimizes empirical loss, by \cref{lemma:wald's_consistency}, $\hat {\theta}_m \rightarrow \theta^*$.

    By condition 1 and the second term in condition 2, we can exchange the integral and derivative and get,
    \begin{equation*}
        \mathbb E[\nabla_{\theta} q_{h,\theta^*}^{\textnormal{WSM}}(\mathcal X)]=\nabla_{\theta} \mathcal J_{h}^{\textnormal{WSM}}(\theta^*) = 0,
    \end{equation*}
    where the second equation is guaranteed by \cref{thm:unique_wsm}. Then with condition 3, we are able to apply the central limit theorem and verify (iii) in \cref{lemma:normality}. Then by condition 1 and the third term in condition 2, we can define in $B(\theta^*, \varepsilon)$,
    \begin{equation*}
        V(\theta)=\mathbb E[\nabla^2_{\theta}q_{h,\theta}^{\textnormal{WSM}}(\mathcal X)]=\nabla^2_{\theta}\mathcal J_{h}^{\textnormal{WSM}}(\theta),
    \end{equation*}
    where the second equality is guaranteed by DCT. For the same reason, $V(\theta)$ defined above is continuous at $\theta^*$. This along with the thrid term in condition 2 guarantees the ULLN. Then by ULLN, we have
    verify (iv) in \cref{lemma:normality}. Finally condition 4 verifies (v) in \cref{lemma:normality}. Thus we have,
    \begin{equation*}
        \sqrt m (\hat \theta_m-\theta^*)\xrightarrow[]{w} \mathcal N(0, \nabla_{\theta}^2\mathcal J_{h}^{\textnormal{WSM}}(\theta^*)^{-1} \mathbb E[\nabla q_{h,\theta^*}^{\textnormal{WSM}}(\mathcal X)\nabla q_{h,\theta^*}^{\textnormal{WSM}}(\mathcal X)^\top]\nabla_{\theta}^2\mathcal J_{h}^{\textnormal{WSM}}(\theta^*)^{-1})
    \end{equation*}
    
\end{proof}

\begin{example}\textnormal{(Truncated Normal Point Process Continued')}
    In this case, the criterion function is,
    \begin{align*}
        q_{h,\theta}^{\textnormal{WSM}}(\bm X) = \sum_{n=1}^{N(\bm X)}\Big\{[\frac{1}{2}(\bm x_n-\theta)^\top\bm \Sigma^{-2}(\bm x_n-\theta)-\textnormal{Tr}(\bm \Sigma^{-1})]h(\bm x_n) - (\bm x_n-\theta)^\top\bm \Sigma^{-1}\nabla_{\bm x_n} h(\bm x_n)\Big\}
    \end{align*}
    It's still a polynomial of parameter $\theta$, and thus follows the same steps as in \cref{ex:truncated_gaussian}, one can verify condition 1 to condition 3 in \cref{assum:wsm_asymp}. Then one can change differentiation and integration and write,
    \begin{equation*}
        \nabla_{\theta}^2 \mathcal J_{h}^{\textnormal{WSM}}(\theta^*)= {\bm \Sigma^*}^{-2}\mathbb E[\sum_{n=1}^{N(\mathcal X)}h(\bm x_n)]\succ 0.
    \end{equation*}
    Therefore, all conditions in \cref{assum:wsm_asymp} can be verified and the asymptotic property holds.
\end{example}

\subsection{Proof of \texorpdfstring{\cref{thm:AWSM_consistency}}{Theorem 16}}

\begin{assumption}\label{assu:asymp_awsm}
    \hspace*{\fill}
    \begin{enumerate}
        \item $q^{\textnormal{AWSM}}_{h_T,h_S,\theta}(\bm X)$ is twice continuously differentiable at each $\theta\in \Theta$ $\mathbb P$ almost everywhere.
        \item It holds that for some $\varepsilon >0$,
        \begin{align*}
            \mathbb E\Big[\sup_{\theta\in \Theta}|q_{h_T,h_S,\theta}^{\textnormal{AWSM}}(\mathcal X)|\Big]&<\infty,
            \mathbb E\big[\sup_{\|\theta-\theta^*\|<\varepsilon}\|\nabla_{\theta}q^{\textnormal{AWSM}}_{h_T,h_S,\theta}(\mathcal X)\|\Big]<\infty,\\
            \mathbb E\big[\sup_{\|\theta-\theta^*\|<\varepsilon}\|\nabla^2_{\theta}q^{\textnormal{AWSM}}_{h_T,h_S,\theta}(\mathcal X)\|\Big]&<\infty.
        \end{align*}
        \item $\mathbb E\Big\|\nabla q_{h_T,h_S,\theta^*}^{\textnormal{AWSM}}(\mathcal X)\nabla q_{h_T,h_S,\theta^*}^{\textnormal{AWSM}}(\mathcal X)^\top\Big\|<\infty$.
        \item $\nabla_{\theta}^2\mathcal J^{\textnormal{WSM}}_{h}(\theta^*)$ is nonsingular. 
    \end{enumerate}
\end{assumption}

\begin{proof}
Since the modified assumptions ensure that the argument mirrors that of the WSM loss asymptotics, the proof can be replicated verbatim. We therefore omit the repetition.
\end{proof}

\subsection{Proof of Theorem \texorpdfstring{\cref{thm:WSM_bound}}{19} and \texorpdfstring{\cref{thm:AWSM_bound}}{20}}

Now we will again state some technical lemmas for a general probability space $(\mathcal Y, \mathbb P)$ with sample size being $m$. And our \cref{thm:WSM_bound} and \cref{thm:AWSM_bound} are direct corollary of these lemmas when the probability space is $(V^\cup, \mathbb P)$.

\begin{lemma}\textnormal{(Theorem 5.52 in van der Vaart (2000))}Suppose $0<\beta<\alpha, \alpha>1$.
Suppose criterion function q is $(\underline{C},\alpha)$-strongly identifiable and,
\begin{equation}\label{eq:ep_inequality}
    \mathbb E[\sup_{\theta:\|\theta-\theta^*\|\leq \delta}|\mathbb G_m(q_{\theta}-q_{\theta^*})|]\leq \bar C\delta^\beta,
\end{equation}
where $\mathbb G_mf=\frac{1}{\sqrt m}\sum_{i=1}^m(f(Y_i)-\mathbb E[f(Y_i)])$. Suppose the $\hat \theta_m$ minimizes empirical loss $Q_m(\theta)$ and is consistent, then for any positive $K$ we have,
\begin{equation}\label{lemma:vanila_convergence_rate} 
    \textnormal{Pr}\left\{\left\|\hat{{\theta}}_m-{\theta}^{*}\right\|>\frac{2^{K}}{m^{1 /(2(\alpha-\beta))}}\right\} \leq 2^{K(\beta-\alpha)} \frac{2^{2 \alpha}}{2^{\alpha-1}-1} \frac{\bar C}{\underline{C}} .
\end{equation}
\end{lemma}

This is the lemma we will base on to establish the rate of convergence in the parametric model. Usually, one can verify the population loss to be $(\underline{C},1)$-strongly identifiable by Taylar expansion. To evaluate \cref{eq:ep_inequality}, we need to further analyze the function class of the loss function.

Let us define $N_{\square}\left(\varepsilon, \mathcal{G}, L^{2}(\mathbb P)\right)$ be the bracketing number of $\mathcal{G}$
with the radius $\varepsilon$ under the norm of $L^{2}(\mathbb P)$ and $J_{[]}\left(\eta, \mathcal{G}, L^{2}(\mathbb P)\right)$ be
$$
J_{[]}\left(\eta, \mathcal{G}, L^{2}(\mathbb P)\right)=\int_{0}^{\eta} \sqrt{\log N_{[]}\left(\varepsilon, \mathcal{G}, L^{2}(\mathbb P)\right)} \mathrm{d} \varepsilon.
$$

\begin{lemma}\label{lemma:bounding_ep}
\textnormal{(Corollary 19.35 of van der Vaart (2000))} Let $G$ be an envelope for the function class $\mathcal{G} \subset L^{2}(\mathbb P)$, i.e., $\sup _{g \in \mathcal{G}}\|g(y)\|_{\infty} \leq G(y)$ for every $y \in \mathcal Y$ and suppose $\mathbb{E}\left[|G(Y)|^{2}\right]<\infty$. Then, we have
$$
\mathbb{E}\left[\sup _{g \in \mathcal{G}}\left|\mathbb{G}_{m}(g)\right|\right] \leq C J_{[]}\left(\|G\|_{L^{2}(\mathbb P)}, \mathcal{G}, L^{2}(\mathbb P)\right)
$$
where $C$ is a universal constant.
\end{lemma}

\begin{lemma}\label{lemma:parametric_bracketing_num}
\textnormal{(Example 19.6 in van der Vaart (2000))} The bracketing number of the parametrized loss function $q_{{\theta}}(y)$ is given as follows. Let $\Theta \subset \mathbb{R}^{p}$ be contained in a ball of radius $R$. Let $\mathcal{F}=\left\{q_{\boldsymbol{\theta}}(y) | \boldsymbol{\theta} \in \Theta\right\}$ be a function class indexed by $\Theta$. Suppose there exists a function $\dot{q}(y)$ with $\|\dot{q}\|_{L^{2}(\mathbb P)}<\infty$ such that
$$
\left|q_{{\theta}_{1}}(y)-q_{{\theta}_{2}}(y)\right| \leq \dot{q}(y)\left\|{\theta}_{1}-{\theta}_{2}\right\|_{2}
$$
for all $y \in \mathcal Y$ and ${\theta}_{1}, {\theta}_{2} \in \Theta$. Then, for every $\varepsilon>0$,
$$
N_{[]}\left(\varepsilon\|\dot{q}\|_{L^{2}(\mathbb P)}, \mathcal{F}, L^{2}(\mathbb P)\right) \leq\left(1+\frac{4 R}{\varepsilon}\right)^{p}
$$
\end{lemma}

\begin{proposition}\label{lemma:main_result}
    $q_{\theta}(y): \Theta\times \mathcal Y\rightarrow \mathbb R$ and $q_{\theta}$ has envelope function $\dot q(y)$. Assume $q_{\theta}(y)$ is $(\underline{C},\alpha)$-strongly identifiable and $\hat {\theta}_m=\mathop{\textnormal{argmin}}_{\theta\in \Theta}\frac{1}{m}\sum_{i=1}^m q_{\theta}(Y_i)$, then for any $0<\delta<C\frac{K_{\alpha}\|\dot q\|_{L^2(\mathbb P)}\sqrt p}{\underline{C}}$,
    \begin{equation*}
            \text{Pr}\left[||\hat {\theta}_m - \theta^*|| \leq CK_{\alpha}\frac{\|\dot q\|_{L^2(\mathbb P)}}{\delta \underline{C}}\sqrt{\frac{p}{m}}^{1/(\alpha-1)}\right]\geq 1-\delta,
    \end{equation*}
    with $C$ being a absolute constant and $K_{\alpha}=\frac{2^{2\alpha}}{2^{\alpha-1}-1}$.
\end{proposition}

\begin{proof}
    Define,
\begin{equation*}
    \mathcal G_{\eta} :=\{q_{\theta}(y) - q_{\theta^*}(y)\Big|\theta \in \Theta, ||\theta - \theta^*||\leq \eta\}.
\end{equation*}

We want to upper bound this expected empirical process $\mathbb E\sup_{g\in \mathcal G_{\eta}}|\mathbb G_m(g)|$. Notice that in $\mathcal G_{\eta}$, we have,
\begin{equation*}
    |(q_{\theta_1}(y) - q_{\theta^*}(y))-(q_{\theta_2}(y) - q_{\theta^*}(y))|\leq \dot q(y)\|\theta_1-\theta_2\|.
\end{equation*}

Therefore by \cref{lemma:parametric_bracketing_num}, we have,
\begin{equation*}
    N_{[]}\left(\varepsilon\|\dot{q}\|_{L^{2}(\mathbb P)}, \mathcal{G}_{\eta}, L^{2}(\mathbb P)\right) \leq\left(1+\frac{4 \eta}{\varepsilon}\right)^{p}
\end{equation*}

Also, notice that $\mathcal G_{\eta}$ has envelope function $\dot q(\bm x)\eta$, since,
\begin{equation*}
        |(q_{\theta}(y) - q_{\theta^*}(y))|\leq \dot q(y)\|\theta-\theta_*\|\leq \dot q(y)\eta, \forall \theta\in \Theta
\end{equation*}

Then applying \cref{lemma:bounding_ep} and combine it with \cref{lemma:parametric_bracketing_num}, we have,
\begin{align*}
    \mathbb{E}\left[\sup _{g \in \mathcal{G}_{\eta}}\left|\mathbb{G}_{m}(g)\right|\right] \leq C^\prime J_{[]}\left(\|\dot q\eta\|_{L^{2}(\mathbb P)}, \mathcal{G}_{\eta}, L^{2}(\mathbb P)\right)&=\int_0^{\|\dot q\eta\|_{L^2(\mathbb P)}}\sqrt {\log N_{[]}(\varepsilon,\mathcal G_{\eta}, L^2(\mathbb P))}d\varepsilon\\
    =&C^\prime\|\dot q\|_{L^2(\mathbb P)}\eta\int_0^{1}\sqrt{\log N_{[]}(\|\dot q\|_{L^2(\mathbb P)}\varepsilon\delta, \mathcal G_{\eta}, L^2(\mathbb P))}d\varepsilon\\
    \leq& C^\prime\|\dot q\|_{L^2(\mathbb P)}\eta\int_0^1 \sqrt{p\log(1+\frac{4}{\varepsilon})}d\varepsilon\\
    \leq &C\|\dot q\|_{L^2(\mathbb P)} \sqrt{p}\eta
\end{align*}
Therefore we have verified \cref{eq:ep_inequality} hold with $\beta = 1$ and $\bar C = C\|\dot q\|_{L^2(\mathbb P)}\sqrt p$. Now, for a given probability threshold $\delta$ that is small enough, we can choose a positive $K$ such that the RHS of \cref{lemma:vanila_convergence_rate} is equal to $\delta$, we have w.p. at least $1-\delta$
\begin{equation*}
    \|\hat {\theta}_m - \theta^*\| < \frac{1}{m^{1/(2(\alpha-1))}}\Big(\frac{\bar CK_{\alpha}}{\underline{C}\delta}\Big)^{1/(\alpha-1)}
\end{equation*}

Then plug in $\overline{C}=C\|\dot q\|_{L^2(\mathbb P)}\sqrt p$, this will satisfies the proof.

\end{proof}

\paragraph{Proof for \cref{thm:WSM_bound}}\ 
\begin{proof}
    Notice that by Lipschitz modulus assumption in the theorem, we have 
    \begin{align*}
        &|q^{\textnormal{WSM}}_{h,\theta_1}(\bm X) - q^{\textnormal{WSM}}_{h,\theta_2}(\bm X)|\\\leq& \|\theta_1-\theta_2\|\sum_{n=1}^{N(\bm X)}\Big[\dot A_n(\bm X)h(\bm x_n)+\dot B_n(\bm X)\|\nabla_{\bm x_n} h(\bm x_n)\|\Big].
    \end{align*}
    Therefore $q_{h,\theta}^{\text{WSM}}$ has Lipschitz modulus $\sum_{n=1}^{N(\bm X)}\Big[\dot A_n(\bm X)h(\bm x_n)+\dot B_n(\bm X)\|\nabla_{\bm x_n} h(\bm x_n)\|\Big]$. Then applying \cref{lemma:main_result} with $\underline C=C_h$ and $\dot q=\dot q_{h}^{\textnormal{WSM}}$ 
    will satisfy the proof.
\end{proof}

\paragraph{Proof for \cref{thm:AWSM_bound}}\ 
\begin{proof}
    by Lipschitz modulus assumption in the theorem, we have,
    \begin{align*}
        &|q^{\textnormal{AWSM}}_{h_S,h_T,\theta_1}(\bm X) - q^{\textnormal{AWSM}}_{h_S,h_T,\theta_2}(\bm X)|\\\leq& \|\theta_1-\theta_2\|\cdot\\&\sum_{n=1}^{N(\bm X)}\Big[\dot A_n(\bm X)h_T(t_{n-1}, t_n)+\dot B_n(\bm X)\partial_{t_n} h_T(t_{n-1}, t_n)+\dot C_n(\bm X)h_S(\bm x_n)+\dot D_n(\bm X) \|\nabla_{\bm s_n}h_S(\bm s_n)\|\Big].
    \end{align*}
    Therefore we also identifies the Lipschitz modulus of criterion function $q_{h_T,h_S,\theta}^{\text{AWSM}}$, then applying \cref{lemma:main_result} will satisfy the proof.
\end{proof}

\begin{example}\label{exp:non_asymp_example}\textnormal{(Exponential Family Continued)}
    Recall the notations in \cref{ex:exponential}, notice that when \cref{assum:wsm_score_regularity,assum:wsm_regularity_of_weight} are met, we have
    \begin{gather*}
        \mathcal J^{\textnormal{WSM}}_h(\theta) 
        =\frac{1}{2}(\theta-\theta^*)^\top\bm K(\theta-\theta^*)-\frac{1}{2}\theta^{*\top}\bm K\theta^*+\textnormal{const},\\
         \bm K = \mathbb E\Big[\sum_{n=1}^{N(\mathcal X)}h(\bm x_n)\bm S_n \bm S_n^\top\Big], \quad \bm S_n=\nabla_{\bm x_n} T(\mathcal X) \in \mathbb R^{p\times d}. 
    \end{gather*}
    Then we have for any $\theta\in\Theta$, 
    \begin{gather*}
        \mathcal J^{\textnormal{WSM}}_h(\theta)-\mathcal J^{\textnormal{WSM}}_h(\theta^*) = \frac{1}{2}(\theta-\theta^*)^\top\bm K(\theta-\theta^*). 
    \end{gather*}
    Since $T(\bm X)$ is symmetric w.r.t. its input points, if we assume $h(\bm x_1)\bm S_1 \bm S_1^\top$ has a non-zero minimum eigenvalue $\lambda_{\textnormal{min}}$ a.s., then $a^\top\bm K a = \mathbb E[a^\top N(\mathcal X) h(\bm x_1)\bm S_1 \bm S_1^\top a]\geq \lambda_{\textnormal{min}}\|a\|^2\mathbb E[N(\mathcal X)]$, then we have
    \begin{equation*}
        \mathcal J^{\textnormal{WSM}}_h(\theta)-\mathcal J^{\textnormal{WSM}}_h(\theta^*) \geq \frac{1}{2}\mathbb E[N(\mathcal X)]\lambda_{\textnormal{min}}\|\theta-\theta^*\|^2.
    \end{equation*}
    Therefore, for the exponential family, it usually has a criterion function that is $(\frac{1}{2}\mathbb E[N(\mathcal X)]\lambda_{\textnormal{min}}, 2)$ strongly identifiable. 

    By Taylor expanding $\frac{1}{2}\|\bm \psi_{n,\theta}(\bm X)\|^2+\textnormal{Tr}(\nabla_{\bm x_n}\bm \psi_{n,\theta}(\bm X))$ and $\bm \psi_{n,\theta}(\bm X)$ then taking supreme over $\theta\in \Theta$, one can find Lipschitz modulus $\dot A_n$ and $\dot B_n$ such that $\Gamma(h, \dot A,\dot B)\simeq \mathbb E[N(\mathcal X)]$. Finally, this results in a bound with a probability greater than $1-\delta$, 
    \begin{equation*}
        \|\hat \theta_m-\theta^*\|\lesssim \frac{1}{\delta \lambda_{\textnormal{min}}}\sqrt{\frac{p}{m}}. 
    \end{equation*}
\end{example}

\subsection{Proof of \texorpdfstring{\cref{prop:distance_minimize_local}}{Proposition 21}}

\begin{proof}
We will only prove the first claim, and the other follow the same derivation.

\begin{align*}
    \mathcal L_{h}^{\textnormal{WSM}}(\theta)
    &= \frac{1}{2}\,\mathbb E\Bigg[
        \sum_{n=1}^{N(\mathcal X)}
        \big\|\bm \psi_{n}(\mathcal X)-\bm \psi_{n,\theta}(\mathcal X)\big\|^2
        \,h(\bm x_n)
      \Bigg]\\
    &= \frac{1}{2}\,\mathbb E\Bigg[
        \sum_{n=1}^{N(\mathcal X)}
        \big\|\bm \psi_{n}(\mathcal X)-\bm \psi_{n,\theta}(\mathcal X)\big\|^2
        \,\big(h(\bm x_n)-h(\bm z_n)\big)
      \Bigg],\qquad \forall\,\bm z_n\in\partial V\\
    &\le \frac{1}{2}\,\mathbb E\Bigg[
        \sum_{n=1}^{N(\mathcal X)}
        \big\|\bm \psi_{n}(\mathcal X)-\bm \psi_{n,\theta}(\mathcal X)\big\|^2
        \,L\,\|\bm x_n-\bm z_n\|
      \Bigg],\qquad \forall\,h\in\mathcal H(V,L),\ \forall\,\bm z_n\in\partial V\\
    &\le \frac{1}{2}\,\mathbb E\Bigg[
        \sum_{n=1}^{N(\mathcal X)}
        \big\|\bm \psi_{n}(\mathcal X)-\bm \psi_{n,\theta}(\mathcal X)\big\|^2
        \,L\,\mathrm{dist}(\bm x_n,\partial V)
      \Bigg]\\
    &= \mathcal L^{\textnormal{WSM}}_{L\,\mathrm{dist}(\bm x,\partial V)}(\theta),
    \qquad \forall\,h\in\mathcal H(V,L).
\end{align*}
Above, the line equation use that $h(z_n)=0$ at the boundary. The third line use the Lipschitz property. The fourth line use the definition of distance function. Then taking $L=1$ satisfies our claim.

Now we show that $\textnormal{dist}(\bm x, \partial V)$ satisfies \cref{assum:wsm_regularity_of_weight}. Clearly it is bounded on $V$, since we assume finite first moment, therefore,
\begin{equation*}
    \mathbb E [\sum_{n=1}^{N(\mathcal X)} (\text{dist}(\bm x_n, \partial V))^2]\leq \frac{(\textnormal{diam}(V))^2}{4} \mathbb E[N(\mathcal X)]<\infty.
\end{equation*}
And in lemma 3 of \citet{liu2022estimating}, it is proved that $\textnormal{dist}(\bm x,\partial V) \in H^1(V)$. Therefore $L\textnormal{dist}(\bm x, \partial V) \in \mathcal H(V, L)$. Then choosing $L=1$ gives,
\begin{align*}
    h_0^{\text{WSM}} &= \mathop{\text{argmax}}_{h\in \mathcal H(V,1)}  \mathcal L_{h}^{\text{WSM}}(\theta), \forall \theta\\
    &=\mathop{\text{argmax}}_{h\in \mathcal H(V,1)} \inf_{\theta:\|\theta-\bm\theta^*\|\geq \delta} \mathcal L_{h}^{\textnormal{WSM}}(\theta)\\
    &=\mathop{\text{argmax}}_{h\in \mathcal H(V,1)} \inf_{\theta:\|\theta-\bm\theta^*\|\geq \delta} [\mathcal L_{h}^{\textnormal{WSM}}(\theta) - \mathcal L_{h}^{\textnormal{WSM}}(\theta^*)]\\
    &=\mathop{\text{argmax}}_{h\in \mathcal H(V,1)} \inf_{\theta:\|\theta-\bm\theta^*\|\geq \delta} [\mathcal J_{h}^{\textnormal{WSM}}(\theta) - \mathcal J_{h}^{\textnormal{WSM}}(\theta^*)],
\end{align*}
where the thrid equation is because $\mathcal L_h^{\text{WSM}}(\theta^*)=0$.

\end{proof}

\subsection{Proof of \texorpdfstring{\cref{thm:change_of_variable}}{Theorem 22}}

\begin{proof}
   Notice that the score function of Y is,
   \begin{align*}
        \psi_Y(y)&= \frac{d}{dy}\log p_X(x(y))+\frac{d}{dy}\log {\phi}(y)\\
        &=\frac{d}{dx} \log p_X(x(y))  {\phi}(y)+\frac{\frac{d}{dy}\phi(y)}{\phi(y)}
        \\
       &=\psi_X(x)\phi(y) + \frac{\frac{d}{dy}\phi(y)}{\phi(y)}\\
       \frac{d}{dy}\psi_Y(y)&=\frac{d}{dx}\psi_X(x)\phi(y)^2+\frac{d}{dy}\phi(y)\psi_X(x)+c_1(y)
   \end{align*}
where $c_1(y)$ is a known function that only depends on the transformation. Plug those expressions into $\mathcal J_{\text{SM},Y}(\theta)$, we have,
    \begin{align*}
        \mathcal J_{\textnormal{SM},Y}(\theta)&=\mathbb E\Bigg[\frac{1}{2}[\psi_X(X) \phi(Y)+\frac{\frac{d}{dy}\phi(Y)}{\phi(Y)}]^2+\frac{d}{dx}\psi_X(X)\phi(Y)^2+\frac{d}{dy}\phi(Y)\psi_X(X)+c_1(Y)\Bigg]\\
        &= \mathbb E\Bigg[\frac{1}{2}\psi^2_X(X) \phi^2(Y)+\frac{d}{dx}\psi_X(X)\phi(Y)^2+\psi_X(X)\frac{d}{dy}\phi(Y)+\frac{d}{dy}\phi(Y)\psi_X(X)+c_1(Y)\Bigg]\\
        &=\mathbb E\Bigg[\frac{1}{2}\psi^2_X(X) \phi^2(Y)+2\psi_X(X)\frac{d}{dy}\phi(Y)+\frac{d}{dx}\psi_X(X)\phi(Y)^2+c_1(Y)\Bigg]
    \end{align*}
Here we does not write explicitly, but only $\psi_X$ depends on $\theta$ and $c_1(Y)$ is a constant that does not depend on $\theta$, only depend on $Y$.

    Notice that,
    \begin{align}\label{eq:change_derivative}
        \frac{d}{dx} \phi(y(x))=\frac{d}{dx}(\frac{d}{dy}\phi(y(x)))=\frac{d}{dy}\phi(y(x))\frac{d}{dx}g(x)=\frac{d}{dy}\phi(y(x))\frac{1}{\phi(y(x))}.
    \end{align}
    For the last equation we use the result that $\phi(y)\frac{d}{dx}g(x)=1$.
    
    Then we substitute $\frac{d}{dy}\phi(y)$ by $\frac{d}{dx} \phi(y(x))\phi(y(x))$ (based on \cref{eq:change_derivative}) and have,
    \begin{align*}
         \mathcal J_{\textnormal{SM},Y}(\theta)& = \mathbb E\Bigg[\frac{1}{2}\psi^2_X(x) \phi^2(y(x))+\psi_X(x)2\frac{d}{dx} \phi(y(x))\phi(y(x))+\frac{d}{dx}\psi_X(x)\phi(y(x))^2\Bigg]+C_1\\
         &= \mathcal J^{\text{WSM}}_{\phi^2(g(x))}(\theta)
    \end{align*}

    Finally, notice that the weight function indeed takes zero value at the boundary, since
    \begin{equation*}
\text{As } x\to \partial I,\ \ y=g(x)\to\pm\infty,\ \text{ hence }
h(x)=\phi(g(x))^2 \to \lim_{|y|\to\infty}\phi(y)^2 = 0.
    \end{equation*}

\end{proof}

\section{Proof of Results in \texorpdfstring{\cref{sec:integration-free-training}}{Section 6}}

\subsection{Proof of \texorpdfstring{\cref{prop:effectiveness_of_loss}}{Proposition 24}}\label{sec:derivation_of_surviving_intensity}

\begin{proof}
First, this proof hold regardless of a zero probability set under the true data-generating measure. So we overlook measure zero set here. Denote 
\begin{equation*}
    q_{T,n}(t|\mathcal H_{n-1}):=\tilde \lambda_{T,n}(t|\mathcal H_{n-1})\exp(-\int_{t_{n-1}}^{t} \tilde \lambda_{T,n}(\tau|\mathcal H_{n-1})d\tau
\end{equation*}
Notice that when $\mathcal L^{\textnormal{AWSM},T}(\tilde \lambda_T)$ is minimized, this means the conditional score of $q_n$ is equivalent to the conditional score of $p_{T,n}$, this means that,
\begin{equation*}
    q_{T,n}(t|\mathcal H_{n-1})=c(\mathcal H_{n-1}) p_{T,n}(t|\mathcal H_{n-1}),
\end{equation*}
where $c(\mathcal H_{n-1})$ is a constant that does not depend on $t$. Notice that if we integrate both sides at $[t_{n-1},T]$, we have,
\begin{equation*}
    c(\mathcal H_{n-1})=\frac{\int_{t_{n-1}}^T q_{T,n}(\tau|\mathcal H_{n-1})d\tau }{\int_{t_{n-1}}^T p_{T,n}(\tau|\mathcal H_{n-1})d\tau}=\frac{\int_{t_{n-1}}^T q_{T,n}(\tau|\mathcal H_{n-1})d\tau }{p(N(\mathcal X)\geq n|\mathcal H_{n-1})}.
\end{equation*}
When $\mathcal L^{\textnormal{Survival}}(\hat F)$ is minimized, this means $\hat F_n(\mathcal H_{n-1})=p(N(\mathcal X)\geq n|\mathcal H_{n-1})$. Therefore we have,
\begin{equation*}
    \hat p_{T,n}(t|\mathcal H_{n-1})=\hat F_n(\mathcal H_{n-1})\frac{q_{T,n}(t|\mathcal H_{n-1})}{\int_{t_{n-1}}^T q_{T,n}(\tau|\mathcal H_{n-1})d\tau }=p_{T,n}(t|\mathcal H_{n-1}).
\end{equation*}

Then the hazard function of $\hat p_{T,n}(t|\mathcal H_{n-1})$ will also equal to the hazard function of $p_{T,n}(t|\mathcal H_{n-1})$, which is $\lambda_{T,n}(\mathcal H_{n-1})$, this means,
\begin{equation*}
    \hat \lambda_{T,n}(t|\mathcal H_{n-1}):=\frac{\hat p_{T,n}(t|\mathcal H_{n-1})}{1-\int_{t_{n-1}}^t \hat p_{T,n}(\tau|\mathcal H_{n-1})d\tau}=\lambda_{T,n}(t|\mathcal H_{n-1}).
\end{equation*}
Finally, notice that,
\begin{equation}\label{eq:close-form of q integration}
    \begin{aligned}
         \int_{t_{n-1}}^t q_{T,n}(\tau|\mathcal H_{n-1})dt&=\left. -[\exp(-\tilde \Lambda_n(\tau|\mathcal H_{n-1})]\right|_{\tau=t_{n-1}}^{\tau=t}=1-\exp(-\tilde \Lambda_n(t|\mathcal H_{n-1}))\\&=1-\tilde G_n(t|\mathcal H_{n-1})
    \end{aligned}
\end{equation}
Now we plug in \cref{eq:close-form of q integration} into the expression of $\hat \lambda_{T,n}$ results in,
\begin{equation*}
       \hat \lambda_{T,n}(t_n\mid \mathcal H_{n-1})= \frac{\tilde  G_n(t_n\mid \mathcal H_{n-1})\,\tilde \lambda_{T,n}(t_n\mid \mathcal H_{n-1})}{\tilde G(t_n\mid \mathcal H_{n-1})+\dfrac{1-\tilde G_n(T\mid \mathcal H_{n-1})}{\hat F_n(\mathcal H_{n-1})}-1}=\lambda_{T,n}(t|\mathcal H_{n-1}).
\end{equation*}

\end{proof}

\subsection{Derivation of Log Likelihood}\label{sec:derivation_of_ll}
Notice that we have,
\begin{align*}
    ll = \log \hat j_N(\bm x_1,\ldots, \bm x_N) &= \log p_N(\bm x_1,\ldots, \bm x_N)+\log \hat G_{N+1}(T|\bm x_1,\ldots,\bm x_N)\\
    &=\sum_{n=1}^N \log \hat p_n(\bm x_n|\mathcal H_{n-1}) + \log (1-\hat F_{N+1}(T|\mathcal H_N))\\
    &=\sum_{n=1}^N [\log \hat f_n(\bm s_n|\mathcal H_{t_n})+\log \hat p_n(t_n|\mathcal H_{n-1})] + \log (1-\hat F_{N+1}(T|\mathcal H_N))
\end{align*}

Then notice that the estimated conditional temporal density should be the unnormlalized conditional density divided by the estimated normalizing constant, which is
\begin{align*}
    \log \hat p_n(t_n|\mathcal H_{n-1})&=\log \hat F_n(\mathcal H_{n-1})+\log q_{T,n}(t_n|\mathcal H_{n-1})-\log \int_{t_{n-1}}^T q_{T,n}(\tau|\mathcal H_{n-1})d\tau\\
    &=\log \hat F_n(\mathcal H_{n-1})+\log \tilde \lambda_{T,n}(t_n|\mathcal H_{n-1})-\tilde\Lambda_n(t_n|\mathcal H_{n-1})-\log(1-\tilde G_n(T|\mathcal H_{n-1})).
\end{align*}
where we use some equations we derive from \cref{sec:derivation_of_surviving_intensity}. And the spatial likelihood is the log of unnormalized $\tilde f$ divided by the integral of it on $\mathcal S$.

\subsection{A Heuristic Discussion on the Necessity of Survival Classification} 
\label{sec:heuristic_discussion_on_survival}

Let $\{\lambda_n(t\mid\mathcal H_{n-1})\}_{n\ge1}$ denote the (nonnegative) conditional intensity of a temporal point process on $(t_{n-1},T)$ given the history $\mathcal H_{n-1}=(t_1,\ldots,t_{n-1})$, and let $\tilde \lambda_n(t\mid\mathcal H_{n-1})\ge 0$ be a candidate model. On any subinterval where both intensities are strictly positive and differentiable, conditional score matching reduces to the first–order ODE
\begin{equation}\label{eq:score-ode}
\frac{d}{dt}\log \tilde \lambda_n(t\mid\mathcal H_{n-1})-\tilde \lambda_n(t\mid\mathcal H_{n-1})
\;=\;
\frac{d}{dt}\log \lambda_n(t\mid\mathcal H_{n-1})-\lambda_n(t\mid\mathcal H_{n-1}),
\quad t\in (t_{n-1},T).
\end{equation}
Define the cumulative hazard
\begin{equation}\label{eq:cumhaz}
\Lambda_n(t\mid\mathcal H_{n-1}) \;=\; \int_{t_{n-1}}^{t} \lambda_n(d\tau\mid\mathcal H_{n-1})\,d\tau.
\end{equation}
Solving \eqref{eq:score-ode} yields the family of solutions
\begin{equation}\label{eq:family}
\tilde \lambda_n(t\mid\mathcal H_{n-1})
\;=\;
\frac{\lambda_n(t\mid\mathcal H_{n-1})}{\,1-\alpha\exp\!\big(\Lambda_n(t\mid\mathcal H_{n-1})\big)\,}\,,
\end{equation}
where $\alpha$ is a constant with respect to $t$ determined by the initial value of this ODE, and \eqref{eq:family} is also a valid intensity whenever the denominator remains positive. In particular, if there exists $t^\star\in (t_{n-1},T]$ such that $\tilde \lambda_n(t^\star\mid\mathcal H_{n-1})=\lambda_n(t^\star\mid\mathcal H_{n-1})$, then $\alpha=0$ and \eqref{eq:family} collapses to $f_n\equiv \lambda_n$ on that interval. 

\paragraph{Implicit Bias.}
When $\tilde\lambda_n$ is parameterized nonparametrically (e.g., splines, kernels, or neural networks), there is no explicit regularization that enforces $\alpha = 0$ in \cref{eq:family}. In particular, the score condition only pins down $\lambda_n$ up to the family of transformed solutions $\tilde \lambda_n$ indexed by $\alpha$, and there is no built-in preference for the “correct” solution within this family. In overparameterized neural networks trained with stochastic gradient methods, it is known \citep{soudry2018implicit}  that optimization often exhibits implicit bias toward solutions that are in some sense simpler (e.g., smaller norm or smoother), but it is a priori unclear whether the ground-truth intensity $\lambda_n$ or one of the alternative solutions $\tilde \lambda_n$ with $\alpha\neq 0$ is simpler under this bias.

To make this ambiguity explicit, note that \cref{eq:score-ode} can also be solved by expressing the true intensity $\lambda_n$ as a transformation of an \emph{unnormalized} (this does not mean it should normalize to 1, but meaning that they have the same score) intensity $\tilde\lambda_n$:
\begin{equation*}
    \lambda_n(t \mid \mathcal H_{n-1})
    = \frac{\tilde\lambda_n(t \mid \mathcal H_{n-1})}
           {1+\kappa \exp\big(\tilde \Lambda_n(t \mid \mathcal H_{n-1})\big)},
    \qquad
    \tilde \Lambda_n(t \mid \mathcal H_{n-1})
    := \int_{t_{n-1}}^{t} \tilde\lambda_n(\tau \mid \mathcal H_{n-1})\,d\tau,
\end{equation*}
for some constant $\kappa$ chosen so that $\lambda_n$ is a valid (nonnegative) intensity. This representation exposes a whole equivalence class: by varying $\kappa$ and $\tilde\lambda_n$ we obtain different pairs $(\lambda_n,\tilde\lambda_n)$ that satisfy the same ODE. Crucially, we can choose $\tilde\lambda_n$ to be much simpler than $\lambda_n$. For example, in our logistic construction we take $\tilde\lambda_n$ to be a constant function, whereas the corresponding $\lambda_n$ becomes a more complex logistic curve. In this situation, a neural network trained via score matching is naturally attracted to the simpler “fake intensity” $\tilde\lambda_n$ rather than the more complex but true intensity $\lambda_n$. This mechanism is exactly what motivates the ablation study: the survival classification term is introduced to break this degeneracy and disfavor such fake solutions.

\paragraph{Practical Takeaways.}
In practice, whether the ground-truth intensity $\lambda_n$ or an ODE-solution $\tilde\lambda_n$ is “simpler” is highly data-dependent. From a modeling perspective, augmenting the score objective with a survival classification term is always theoretically well-justified. However, empirically it introduces a practical issue: the associated binary classification problem is typically highly imbalanced, since negative examples (events indicating that the sequence terminates at the current time) are much rarer than positive ones. This imbalance can make the survival classifier hard to train and may cause the learned intensity to behave poorly near the end of a sequence.

In our neural network experiments, we indeed observe that omitting the survival term sometimes yields comparable or even slightly better performance than including it. As discussed above, this happens in regimes where the true intensity is itself the simplest solution in the ODE-equivalence class, so the implicit bias of the optimizer already favors the correct solution, or when the imbalanced classification problem arises. Thus, there is a trade-off in whether to include the survival component. In practice, we recommend treating the survival classification term as a hyperparameter that can be switched on or off depending on the dataset and task.

\section{Additional Experimental Details}\label{sec:additional experiment}

In this section, we present some experimental details\footnote{Our implementation is available at \url{https://github.com/KenCao2007/WSM_STPP.git}.}
. First, we briefly introduce the deep point process models that we use. Then we introduce how we implement baseline training objectives. We also specify some training hyperparameters and present some additional results. \ifarxiv
Finally, we discuss some empirical findings about the need of a fitting survival function.
\fi


\subsection{Detailed Implementation of Deep Intensity Models}\label{sec:network_details}

\paragraph{THP and SAHP}
We follow the open-source Easy-TPP implementation~\citep{xue2024easytpp}. To ensure derivable, we replace the ReLU activation with softplus function.
Our score modulus is implemented using their model architecture and intensity function interface. For the survival modulus, we add a MLP that takes as input the transformer-encoded history embedding.


\paragraph{SMASH.}\label{app:st_transformer_encoder}
For this model, we largely follow the open-source implementation of \citep{li2024beyond}, with the following modifications. Firstly, to ensure derivable, we replace some ReLU activation into other derivable actions. Second, in the original implementation, the spatial score is modeled directly by a neural network. We instead replace the last layer with a softplus activation and reuse the same network (up to the last layer) to parameterize the spatial density $f_{S,n}$.
Third, for multivariate processes we add an additional MLP classifier that takes as input the history embedding, the current time, and the current location to predict the event type, and we parameterize the survival function by another MLP that takes only the history embedding as input.

In our experiments, we observed that the public implementation of SMASH shares part of the network between the temporal and spatial intensity heads, which can lead to a more complex optimization landscape. We therefore also provide a variant of SMASH in which the two heads have disjoint parameters (i.e., no weight sharing). We use this decoupled version to produce \cref{fig:synthetic_spatial_temporal_intensity_visualize} and \cref{fig:earthquake_map_visual} for clearer intensity visualizations, since our contribution lies in the training objective rather than in the specific baseline architecture, and the objective can be applied to any intensity model.

\subsection{Detailed Implementation of Baseline Training Methods}

\paragraph{MLE.}
We compute the MLE objective using \cref{eq:true_ll}. To approximate $\tilde{\Lambda}_n(t_n)$ and $\tilde{\Lambda}_n(T)$, we perform Riemann integration over a uniform grid with a fixed number of nodes on the temporal axis. For the spatial log-likelihood, the normalizing constant is computed analogously via a Riemann sum over a uniform grid with a fixed number of nodes along both dimensions of $\mathcal{S}$.

\paragraph{DSM}: For DSM, at the score part, we use the training objective as in \citep{li2024beyond}.  For the spatial noise, we sample from Gaussian distribution. For the temporal noise, we sampled from log-normal distribution. We also add a survival classifier to recover the intensity besides matching the score. 


\subsection{Hyperparameters}
For results in \cref{table:results_on_temporal_dataset_split} and \cref{table:results_on_ST_dataset}, we train them on three seeds. We train each experiment with the same number of epochs and validate every 10 epochs to report the best result. The detailed hyperparameters for \cref{table:results_on_temporal_dataset_split} is summarized in \cref{tab:t_hyperparams} and \cref{table:results_on_ST_dataset} are summarized in \cref{tab:st_hyperparams}. \textbf{Epochs} column represents the number of epochs for the experiments of a model  on a dataset. \textbf{bs} is the batch size. $N_T$ and $N_s$ is the bool variable for whether normalizing the time or space.


\begin{table*}[t]
\centering
\caption{The MAE of 2-variate Hawkes processes trained by MLE, (A)SM, and (A)WSM on the synthetic dataset.}
\label{table:additional_experiments}
\begin{sc}
\scalebox{0.74}{
\begin{tabular}{c|ccc}
    \toprule
    \multirow{2}{*}{Estimator} & \multicolumn{3}{c}{Exp-Hawkes }   \\
    \cmidrule{2-4}
        & $\alpha_{21}$ & $\alpha_{22}$ & $\mu_{2}$\\   
    \midrule
    (A)WSM & $\bm{0.052_{\pm 0.054}}$ & $\bm{0.022_{\pm 0.005}}$& $\bm{0.011_{\pm 0.012}}$  
    \\
     \midrule
    (A)SM & $0.769_{\pm 0.001}$ & $0.769_{\pm 0.001}$& $0.680_{\pm 0.270}$
    \\
    \midrule
    MLE  & {$0.065_{\pm 0.032}$} & {$0.034_{\pm 0.015}$} & ${0.014_{\pm 0.002}}$\\
    \bottomrule
\end{tabular}}
\end{sc}
\end{table*}

\begin{table}[t]
    \centering
    \small
    \setlength{\tabcolsep}{4pt}
    \caption{Training hyperparameters for \cref{table:results_on_temporal_dataset_split}.}
    \label{tab:t_hyperparams}
    \begin{tabular}{lccccccccc}
        \toprule
        Dataset 
        & $\alpha_{K}$ 
        & $\alpha_{\text{surv}}^{\text{THP}}$ 
        & $\alpha_{\text{surv}}^{\text{SAHP}}$ 
        & grid 
        & num-noise 
        & $\sigma_t$ 
        & $N_T$ 
        & epochs 
        & bs \\
        \midrule
        retweet       & 10 & 50 & 10 & 10 & 50 & 0.5 & 1 & 100 & 64 \\
        stackoverflow & 10 & 50 & 10 & 10 & 50 & 0.5 & 1 & 200 & 64 \\
        taobao        & 10 & 50 & 10 & 10 & 50 & 0.5 & 1 & 200 & 64 \\
        taxi          & 10 & 50 &  1 & 10 & 50 & 0.5 & 1 & 250 & 64 \\
        \bottomrule
    \end{tabular}
\end{table}


\begin{table}[t]
    \centering
    \small
    \setlength{\tabcolsep}{4pt} 
    \caption{Training hyperparameters for \cref{table:results_on_ST_dataset}}
    \label{tab:st_hyperparams}
    \begin{tabular}{lcccccccccccc}
        \toprule
        Dataset & $\textnormal{grid}_T$ & $\textnormal{grid}_S$ & $\alpha_{K}$ & $\alpha_{\text{surv}}$ & $\alpha_{S}$ & num-noise& $\sigma_t$ & $\sigma_s$ & bs & epochs & $N_T$ & $N_S$ \\
        \midrule
        Earthquake & 4 & 4 & NA & 1 & 1  &10  & 0.5 & 0.05 & 160 & 200 & 0 & 1 \\
        Citibike   & 5 & 5 & NA & 1 & 10 &10  & 0.5 & 0.05 & 60  & 150 & 0 & 1 \\
        COVID19    & 5 & 5 & NA & 1 & 0.5 &10 & 0.5 & 0.01 & 90  & 200 & 0 & 1 \\
        football   & 5 & 5 & 1  & 1 & 1  &10 & 0.1 & 0.1  & 12  & 200 & 1 & 1 \\
        \bottomrule
    \end{tabular}
\end{table}

\vskip 0.2in
\bibliography{sample}

\end{document}